\documentclass[10pt, journal]{IEEEtran}
\pdfoutput=1
\usepackage[noadjust]{cite}
\usepackage[bookmarks=false]{hyperref}
\hypersetup{
  colorlinks,
  linkcolor={blue!50!black},
  citecolor={blue!50!black},
  urlcolor={blue!80!black}
}
\usepackage[font=footnotesize]{subcaption}
\usepackage[font=footnotesize]{caption}
\usepackage[dvipsnames]{xcolor}
\usepackage[utf8]{inputenc} % allow utf-8 input
\usepackage{graphicx,url}
\usepackage{booktabs}       % professional-quality tables
\graphicspath{ {./images/} }
\usepackage{multirow}
\usepackage{hhline}
\usepackage{soul}
\usepackage{xspace}
\usepackage{graphicx}
\usepackage{amssymb,amsmath,mathtools} %
\usepackage{amsthm}
\usepackage[outline]{contour}% http://ctan.org/pkg/contour
\usepackage{listings}
\usepackage{tcolorbox}
\definecolor{codegreen}{rgb}{0,0.6,0}
\definecolor{codegray}{rgb}{0.5,0.5,0.5}
\definecolor{codepurple}{rgb}{0.58,0,0.82}
\definecolor{backcolour}{rgb}{0.95,0.95,0.92}
\definecolor{RoyalAzure}{rgb}{0.0, 0.22, 0.66}
\newtheorem{theorem}{Theorem}   %
\newtheorem{proposition}{Proposition}
\newtheorem{problem}{Problem}
\newtheorem{lemma}{Lemma}
\newtheorem{remark}{Remark}
\newtheorem{scenario}{Scenario}

\usepackage{bm}
\newtheorem{definition}{Definition}
\newtheorem{assumption}{Assumption}
\usepackage{color,soul}
\usepackage{bbm}
\usepackage{stmaryrd}
\usepackage{textcomp}
\usepackage{lipsum,lineno}
\usepackage{float}
\makeatletter
\newif\if@restonecol
\makeatother
\usepackage{amsmath}
\usepackage{kbordermatrix}
\usepackage{mathrsfs}
\usepackage[shortlabels]{enumitem}
\newcommand\numeq[1]%
{\stackrel{\scriptscriptstyle(\mkern-1.5mu#1\mkern-1.5mu)}{=}}
\newcommand\numeqq[1]%
{\stackrel{\scriptscriptstyle(\mkern-1.5mu#1\mkern-1.5mu)}{\triangleq}}
\newcommand\numleq[1]%
{\stackrel{\scriptscriptstyle(\mkern-1.5mu#1\mkern-1.5mu)}{\leq}}
\newcommand\numgeq[1]%
{\stackrel{\scriptscriptstyle(\mkern-1.5mu#1\mkern-1.5mu)}{\geq}}
\newcommand\numimp[1]%
{\stackrel{\scriptscriptstyle(\mkern-1.5mu#1\mkern-1.5mu)}{\implies}}
\newcommand\numin[1]%
{\stackrel{\scriptscriptstyle(\mkern-1.5mu#1\mkern-1.5mu)}{\in}}
\newcommand\numiff[1]%
{\stackrel{\scriptscriptstyle(\mkern-1.5mu#1\mkern-1.5mu)}{\iff}}
\usepackage[titlenumbered,ruled,linesnumbered]{algorithm2e}
\SetAlCapNameFnt{\footnotesize}
\SetAlCapFnt{\footnotesize}

\usepackage{fancyhdr} % This should be set AFTER setting up the page geometry
\usepackage{tabularx}
\usepackage{multirow}
\usepackage{dsfont}
\usepackage{listings}
\usepackage{tcolorbox}

\newcommand{\acro}[1]{\textsc{#1}\xspace}

\newcommand{\apt}{\acro{apt}}

\newcommand{\pomdp}{\acro{pomdp}}

\newcommand{\ssh}{\acro{ssh}}

\newcommand{\http}{\acro{http}}

\newcommand{\cve}{\acro{cve}}

\newcommand{\cwe}{\acro{cwe}}

\newcommand{\mab}{\acro{mab}}

\newcommand{\pomis}{\acro{pomis}}

\newcommand{\indep}{\perp \!\!\! \perp}

\newcommand{\scm}{\acro{scm}}

\newcommand{\mat}[1]{\mathbf{#1}}
\newcommand{\X}{\mathbf{X}}
\newcommand{\N}{\mathbf{N}}

\newcommand{\Y}{\mathbf{Y}}

\newcommand{\V}{\mathbf{V}}
\newcommand{\U}{\mathbf{U}}

\newcommand{\x}{\mathbf{x}}

\newcommand{\graph}{\mathcal{G}}

\newcommand{\myP}[1]{P \! \left ( #1    \right)}

\newcommand{\an}[2]{\mathrm{an}(#1)_{#2}}
\newcommand{\de}[2]{\mathrm{de}(#1)_{#2}}
\newcommand{\ch}[2]{\mathrm{ch}(#1)_{#2}}

\newcommand{\pa}[2]{\mathrm{pa}(#1)_{#2}}

\lstdefinestyle{mystyle}{
  backgroundcolor=\color{backcolour},
  commentstyle=\color{codegreen},
  keywordstyle=\color{magenta},
  numberstyle=\tiny\color{codegray},
  stringstyle=\color{codepurple},
  basicstyle=\ttfamily\footnotesize,
  breakatwhitespace=false,
  breaklines=true,
  captionpos=b,
  keepspaces=true,
  showspaces=false,
  showstringspaces=false,
  showtabs=false,
  tabsize=2,
  xleftmargin=50pt,
  xrightmargin=50pt
}
\usepackage{tikz,pgfplots}
\usepackage{pgfplots}
\pgfplotsset{compat=1.15}
\usepackage{pgfplotstable}
\definecolor{gray2}{HTML}{ededed}
\definecolor{gray3}{HTML}{F5F5F5}
\usetikzlibrary{shapes.geometric,backgrounds,patterns, trees}
\usetikzlibrary{3d,decorations.text,shapes.arrows,positioning,fit,backgrounds}
\usetikzlibrary{positioning, decorations.pathmorphing, shapes}
\usetikzlibrary{decorations.pathreplacing}
\usetikzlibrary{shapes.geometric,backgrounds,patterns, trees}
\pgfmathdeclarefunction{gauss}{3}{%
  \pgfmathparse{1/(#3*sqrt(2*pi))*exp(-((#1-#2)^2)/(2*#3^2))}%
}
\pgfmathdeclarefunction{cdf}{3}{%
  \pgfmathparse{1/(1+exp(-0.07056*((#1-#2)/#3)^3 - 1.5976*(#1-#2)/#3))}%
}
\pgfmathdeclarefunction{fq}{3}{%
  \pgfmathparse{1/(sqrt(2*pi*#1))*exp(-(sqrt(#1)-#2/#3)^2/2)}%
}
\pgfmathdeclarefunction{fq0}{1}{%
  \pgfmathparse{1/(sqrt(2*pi*#1))*exp(-#1/2)}%
}
\usetikzlibrary{spy}
\usetikzlibrary{arrows.meta,
  bending,
  intersections,
  quotes,
  shapes.geometric}
\usetikzlibrary{automata, positioning}
\usepgfplotslibrary{fillbetween}
\usetikzlibrary{shapes,arrows}
\usetikzlibrary{arrows.meta}
\usetikzlibrary{positioning}
\tikzset{set/.style={draw,circle,inner sep=0pt,align=center}}
\usetikzlibrary{automata, positioning}
\usetikzlibrary{shapes,shadows}
\tikzstyle{abstractbox} = [draw=black, fill=white, rectangle,
inner sep=10pt, style=rounded corners, drop shadow={fill=black,
  opacity=1}]
\tikzstyle{abstracttitle} =[fill=white]
\usetikzlibrary{calc,positioning,shapes.geometric}
\usetikzlibrary{arrows.meta,arrows}
\DeclareMathOperator*{\argmax}{arg\,max}

\DeclareMathOperator*{\maximize}{maximize}

\usetikzlibrary{matrix}
\tikzstyle{cblue}=[circle, draw, thin,fill=cyan!20, scale=0.8]
\tikzstyle{qgre}=[rectangle, draw, thin,fill=green!20, scale=0.8]
\tikzstyle{rpath}=[ultra thick, red, opacity=0.4]
\tikzstyle{legend_isps}=[rectangle, rounded corners, thin,
fill=gray!20, text=blue, draw]

\tikzstyle{legend_overlay}=[rectangle, rounded corners, thin,
top color= white,bottom color=green!25,
minimum width=2.5cm, minimum height=0.8cm,
pinegreen]
\tikzstyle{legend_phytop}=[rectangle, rounded corners, thin,
top color= white,bottom color=cyan!25,
minimum width=2.5cm, minimum height=0.8cm,
royalblue]
\tikzstyle{legend_general}=[rectangle, rounded corners, thin,
top color= white,bottom color=lavander!25,
minimum width=2.5cm, minimum height=0.8cm,
violet]
\usetikzlibrary{matrix}

\colorlet{myRed}{red!20}
\tikzset{
  rows/.style 2 args={/utils/temp/.style={row ##1/.append style={nodes={#2}}},
    /utils/temp/.list={#1}},
  columns/.style 2 args={/utils/temp/.style={column ##1/.append style={nodes={#2}}},
    /utils/temp/.list={#1}}}
\usetikzlibrary{backgrounds,calc,shadings,shapes.arrows,shapes.symbols,shadows}
\definecolor{switch}{HTML}{006996}

\makeatletter
\pgfkeys{/pgf/.cd,
  parallelepiped offset x/.initial=2mm,
  parallelepiped offset y/.initial=2mm
}
\pgfdeclareshape{parallelepiped}
{
  \inheritsavedanchors[from=rectangle] % this is nearly a rectangle
  \inheritanchorborder[from=rectangle]
  \inheritanchor[from=rectangle]{north}
  \inheritanchor[from=rectangle]{north west}
  \inheritanchor[from=rectangle]{north east}
  \inheritanchor[from=rectangle]{center}
  \inheritanchor[from=rectangle]{west}
  \inheritanchor[from=rectangle]{east}
  \inheritanchor[from=rectangle]{mid}
  \inheritanchor[from=rectangle]{mid west}
  \inheritanchor[from=rectangle]{mid east}
  \inheritanchor[from=rectangle]{base}
  \inheritanchor[from=rectangle]{base west}
  \inheritanchor[from=rectangle]{base east}
  \inheritanchor[from=rectangle]{south}
  \inheritanchor[from=rectangle]{south west}
  \inheritanchor[from=rectangle]{south east}
  \backgroundpath{
    \southwest \pgf@xa=\pgf@x \pgf@ya=\pgf@y
    \northeast \pgf@xb=\pgf@x \pgf@yb=\pgf@y
    \pgfmathsetlength\pgfutil@tempdima{\pgfkeysvalueof{/pgf/parallelepiped
        offset x}}
    \pgfmathsetlength\pgfutil@tempdimb{\pgfkeysvalueof{/pgf/parallelepiped
        offset y}}
    \def\ppd@offset{\pgfpoint{\pgfutil@tempdima}{\pgfutil@tempdimb}}
    \pgfpathmoveto{\pgfqpoint{\pgf@xa}{\pgf@ya}}
    \pgfpathlineto{\pgfqpoint{\pgf@xb}{\pgf@ya}}
    \pgfpathlineto{\pgfqpoint{\pgf@xb}{\pgf@yb}}
    \pgfpathlineto{\pgfqpoint{\pgf@xa}{\pgf@yb}}
    \pgfpathclose
    \pgfpathmoveto{\pgfqpoint{\pgf@xb}{\pgf@ya}}
    \pgfpathlineto{\pgfpointadd{\pgfpoint{\pgf@xb}{\pgf@ya}}{\ppd@offset}}
    \pgfpathlineto{\pgfpointadd{\pgfpoint{\pgf@xb}{\pgf@yb}}{\ppd@offset}}
    \pgfpathlineto{\pgfpointadd{\pgfpoint{\pgf@xa}{\pgf@yb}}{\ppd@offset}}
    \pgfpathlineto{\pgfqpoint{\pgf@xa}{\pgf@yb}}
    \pgfpathmoveto{\pgfqpoint{\pgf@xb}{\pgf@yb}}
    \pgfpathlineto{\pgfpointadd{\pgfpoint{\pgf@xb}{\pgf@yb}}{\ppd@offset}}
  }
}

\makeatletter
\tikzset{anchor/.append code=\let\tikz@auto@anchor\relax,
  add font/.code=%
  \expandafter\def\expandafter\tikz@textfont\expandafter{\tikz@textfont#1},
  left delimiter/.style 2 args={append after command={\tikz@delimiter{south east}
      {south west}{every delimiter,every left delimiter,#2}{south}{north}{#1}{.}{\pgf@y}}}}
\tikzstyle{sms} = [rectangle callout, draw,very thick, rounded corners, minimum height=20pt]
\makeatletter
\tikzset{anchor/.append code=\let\tikz@auto@anchor\relax,
  add font/.code=%
  \expandafter\def\expandafter\tikz@textfont\expandafter{\tikz@textfont#1},
  left delimiter/.style 2 args={append after command={\tikz@delimiter{south east}
      {south west}{every delimiter,every left delimiter,#2}{south}{north}{#1}{.}{\pgf@y}}}}
\tikzstyle{sms} = [rectangle callout, draw,very thick, rounded corners, minimum height=20pt]
\usetikzlibrary{positioning,calc}
\tikzstyle{block} = [rectangle, draw,
text width=10.5em, text centered, rounded corners, minimum height=4em]
\tikzstyle{line} = [draw, -latex]
\tikzset{
  mybackground9/.style={execute at end picture={
      \begin{scope}[on background layer]
        \draw[black,fill=black!5,rounded corners=6ex] (current bounding box.south west)
        rectangle (current bounding box.north east);
        \node[draw,fill=white,ellipse,anchor=west,inner sep=1pt,minimum width=4ex] at (current bounding box.north
        west){#1};
      \end{scope}
    }},
}
\tikzset{
  mybackground13/.style={execute at end picture={
      \begin{scope}[on background layer]
        \draw[black, fill=gray2, rounded corners=4ex] (current bounding box.south west)
        rectangle (current bounding box.north east);
        \node[draw,fill=white,ellipse,anchor=west,inner sep=1pt,minimum width=4ex] at (current bounding box.north
        west){#1};
      \end{scope}
    }},
}
\tikzset{
  mybackground14/.style={execute at end picture={
      \begin{scope}[on background layer]
        \draw[black, rounded corners=2ex] (current bounding box.south west)
        rectangle (current bounding box.north east);
        \node[draw,fill=white,ellipse,anchor=west,inner sep=1pt,minimum width=4ex] at (current bounding box.north
        west){#1};
      \end{scope}
    }},
}

\tikzset{
  mybackground6/.style={execute at end picture={
      \begin{scope}[on background layer]
        \draw[black,rounded corners=1ex, line width=0.15mm] (current bounding box.south west)
        rectangle (current bounding box.north east);
        \node[draw,fill=white,ellipse,anchor=west,inner sep=1pt,minimum width=4ex] at (current bounding box.north
        west){#1};
      \end{scope}
    }},
}

\tikzset{
  mybackground11/.style={execute at end picture={
      \begin{scope}[on background layer]
        \draw[black, fill=Black!80!Sepia!9, rounded corners=6ex] (current bounding box.south west)
        rectangle (current bounding box.north east);
        \node[draw,fill=white,ellipse,anchor=west,inner sep=1pt,minimum width=4ex] at (current bounding box.north
        west){#1};
      \end{scope}
    }},
}

\tikzset{
  mybackground15/.style={execute at end picture={
      \begin{scope}[on background layer]
        \draw[black, fill=Black!80!Sepia!9, rounded corners=3ex] (current bounding box.south west)
        rectangle (current bounding box.north east);
        \node[draw,fill=white,ellipse,anchor=west,inner sep=1pt,minimum width=4ex] at (current bounding box.north
        west){#1};
      \end{scope}
    }},
}

\tikzset{
  mybackground12/.style={execute at end picture={
      \begin{scope}[on background layer]
        \draw[black, fill=Black!40!Emerald!30, rounded corners=3ex, line width=0.3mm] (current bounding box.south west)
        rectangle (current bounding box.north east);
      \end{scope}
    }},
}
\tikzset{
  mybackground18/.style={execute at end picture={
        \begin{scope}[on background layer]
          \draw[black, fill=gray3, rounded corners=4ex] (current bounding box.south west)
                    rectangle (current bounding box.north east);
          \node[draw,fill=white,ellipse,anchor=west,inner sep=1pt,minimum width=4ex] at (current bounding box.north
                   west){#1};
        \end{scope}
    }},
}
\tikzset{
  mybackground58/.style={execute at end picture={
      \begin{scope}[on background layer]
        \draw[black, fill=blue!40!black!5, rounded corners=1ex] (current bounding box.south west)
        rectangle (current bounding box.north east);
        \node[draw,fill=white,ellipse,anchor=west,inner sep=1pt,minimum width=4ex, rounded corners=1ex] at (current bounding box.north
        west){#1};
      \end{scope}
    }},
}
\tikzset{l3 switch/.style={
    parallelepiped,fill=switch, draw=white,
    minimum width=0.75cm,
    minimum height=0.75cm,
    parallelepiped offset x=1.75mm,
    parallelepiped offset y=1.25mm,
    path picture={
      \node[fill=white,
      circle,
      minimum size=6pt,
      inner sep=0pt,
      append after command={
        \pgfextra{
          \foreach \angle in {0,45,...,360}
          \draw[-latex,fill=white] (\tikzlastnode.\angle)--++(\angle:2.25mm);
        }
      }
      ]
      at ([xshift=-0.75mm,yshift=-0.5mm]path picture bounding box.center){};
    }
  },
  ports/.style={
    line width=0.3pt,
    top color=gray!20,
    bottom color=gray!80
  },
  rack switch/.style={
    parallelepiped,fill=white, draw,
    minimum width=1.25cm,
    minimum height=0.25cm,
    parallelepiped offset x=2mm,
    parallelepiped offset y=1.25mm,
    xscale=-1,
    path picture={
      \draw[top color=gray!5,bottom color=gray!40]
      (path picture bounding box.south west) rectangle
      (path picture bounding box.north east);
      \coordinate (A-west) at ([xshift=-0.2cm]path picture bounding box.west);
      \coordinate (A-center) at ($(path picture bounding box.center)!0!(path
      picture bounding box.south)$);
      \foreach \x in {0.275,0.525,0.775}{
        \draw[ports]([yshift=-0.05cm]$(A-west)!\x!(A-center)$)
        rectangle +(0.1,0.05);
        \draw[ports]([yshift=-0.125cm]$(A-west)!\x!(A-center)$)
        rectangle +(0.1,0.05);
      }
      \coordinate (A-east) at (path picture bounding box.east);
      \foreach \x in {0.085,0.21,0.335,0.455,0.635,0.755,0.875,1}{
        \draw[ports]([yshift=-0.1125cm]$(A-east)!\x!(A-center)$)
        rectangle +(0.05,0.1);
      }
    }
  },
  server/.style={
    parallelepiped,
    fill=white, draw,
    minimum width=0.35cm,
    minimum height=0.75cm,
    parallelepiped offset x=3mm,
    parallelepiped offset y=2mm,
    xscale=-1,
    path picture={
      \draw[top color=gray!5,bottom color=gray!40]
      (path picture bounding box.south west) rectangle
      (path picture bounding box.north east);
      \coordinate (A-center) at ($(path picture bounding box.center)!0!(path
      picture bounding box.south)$);
      \coordinate (A-west) at ([xshift=-0.575cm]path picture bounding box.west);
      \draw[ports]([yshift=0.1cm]$(A-west)!0!(A-center)$)
      rectangle +(0.2,0.065);
      \draw[ports]([yshift=0.01cm]$(A-west)!0.085!(A-center)$)
      rectangle +(0.15,0.05);
      \fill[black]([yshift=-0.35cm]$(A-west)!-0.1!(A-center)$)
      rectangle +(0.235,0.0175);
      \fill[black]([yshift=-0.385cm]$(A-west)!-0.1!(A-center)$)
      rectangle +(0.235,0.0175);
      \fill[black]([yshift=-0.42cm]$(A-west)!-0.1!(A-center)$)
      rectangle +(0.235,0.0175);
    }
  },
}

\usetikzlibrary{calc, shadings, shadows, shapes.arrows}
\tikzset{cross/.style={cross out, draw=black, minimum size=2*(#1-\pgflinewidth), inner sep=0pt, outer sep=0pt},
  cross/.default={1pt}}
\tikzset{%
  interface/.style={draw, rectangle, rounded corners, font=\LARGE\sffamily},
  ethernet/.style={interface, fill=yellow!50},% ethernet interface
  serial/.style={interface, fill=green!70},% serial interface
  speed/.style={sloped, anchor=south, font=\large\sffamily},% line speed at edge
  route/.style={draw, shape=single arrow, single arrow head extend=4mm,
    minimum height=1.7cm, minimum width=3mm, white, fill=switch!20,
    drop shadow={opacity=.8, fill=switch}, font=\tiny}% inroute/outroute arrows
}
\tikzset{
    state/.style ={ellipse, draw, minimum width = 0.7 cm},
    point/.style = {circle, draw, inner sep=0.04cm,fill,node contents={}},
    bidirected/.style={Latex-Latex,dashed}
}
\newcommand*{\shift}{1.3cm}% For placing the arrows later
\newcommand*{\router}[1]{
  \begin{tikzpicture}
    \coordinate (ll) at (-3,0.5);
    \coordinate (lr) at (3,0.5);
    \coordinate (ul) at (-3,2);
    \coordinate (ur) at (3,2);
    \shade [shading angle=90, left color=switch, right color=white] (ll)
    arc (-180:-60:3cm and .75cm) -- +(0,1.5) arc (-60:-180:3cm and .75cm)
    -- cycle;
    \shade [shading angle=270, right color=switch, left color=white!50] (lr)
    arc (0:-60:3cm and .75cm) -- +(0,1.5) arc (-60:0:3cm and .75cm) -- cycle;
    \draw [thick] (ll) arc (-180:0:3cm and .75cm)
    -- (ur) arc (0:-180:3cm and .75cm) -- cycle;
    \draw [thick, shade, upper left=switch, lower left=switch,
    upper right=switch, lower right=white] (ul)
    arc (-180:180:3cm and .75cm);
    \node at (0,0.5){\color{blue!60!black}\Huge #1};% The name of the router
    \begin{scope}[yshift=2cm, yscale=0.28, transform shape]
      \node[route, rotate=45, xshift=\shift] {\strut};
      \node[route, rotate=-45, xshift=-\shift] {\strut};
      \node[route, rotate=-135, xshift=\shift] {\strut};
      \node[route, rotate=135, xshift=-\shift] {\strut};
    \end{scope}
  \end{tikzpicture}}

\makeatletter
\pgfdeclareradialshading[tikz@ball]{cloud}{\pgfpoint{-0.275cm}{0.4cm}}{%
  color(0cm)=(tikz@ball!75!white);
  color(0.1cm)=(tikz@ball!85!white);
  color(0.2cm)=(tikz@ball!95!white);
  color(0.7cm)=(tikz@ball!89!black);
  color(1cm)=(tikz@ball!75!black)
}
\tikzoption{cloud color}{\pgfutil@colorlet{tikz@ball}{#1}%
  \def\tikz@shading{cloud}\tikz@addmode{\tikz@mode@shadetrue}}
\makeatother

\tikzset{my cloud/.style={
    cloud, draw, aspect=2,
    cloud color={gray!5!white}
  }
}

\renewcommand\thetable{\arabic{table}}
\usepackage{etoolbox}
\makeatletter
\patchcmd{\@makecaption}
{\scshape}
{}
{}
{}
\makeatother

\floatstyle{ruled}
\newfloat{model}{thp}{lop}
\floatname{model}{}

\makeatletter
\renewcommand*{\fnum@model}{\fname@model}
\makeatother

\makeatletter
\newcommand{\customlabel}[2]{%
   \protected@write \@auxout {}{\string \newlabel {#1}{{#2}{\thepage}{#2}{#1}{}} }%
   \hypertarget{#1}{#2}
}
\makeatother

\makeatletter
\newcommand\xlabel[2][]{\phantomsection\def\@currentlabelname{#1}\label{#2}}
\makeatother

\allowdisplaybreaks
\begin{document}
\bstctlcite{MyBSTcontrol}
\title{Optimal Defender Strategies for CAGE-2 \\using Causal Modeling and Tree Search}
\author{\IEEEauthorblockN{Kim Hammar \IEEEauthorrefmark{2}, Neil Dhir \IEEEauthorrefmark{3}, and Rolf Stadler\IEEEauthorrefmark{2}}\\
 \IEEEauthorblockA{\IEEEauthorrefmark{2}
   KTH Royal Institute of Technology, Sweden\\
 \IEEEauthorblockA{\IEEEauthorrefmark{3}
   Siemens Technology, USA\\
 }
 Email: \{kimham, stadler\}@kth.se}, neil.dhir@siemens.com\\\
\today
}
\maketitle
\begin{abstract}
The \textsc{cage-2} challenge is considered a standard benchmark to compare methods for autonomous cyber defense. Current state-of-the-art methods evaluated against this benchmark are based on model-free (offline) reinforcement learning, which does not provide provably optimal defender strategies. We address this limitation and present a formal (causal) model of \textsc{cage}-2 together with a method that produces a provably optimal defender strategy, which we call Causal Partially Observable Monte-Carlo Planning (\textsc{c-pomcp}). It has two key properties. First, it incorporates the causal structure of the target system, i.e., the causal relationships among the system variables. This structure allows for a significant reduction of the search space of defender strategies. Second, it is an online method that updates the defender strategy at each time step via tree search. Evaluations against the \textsc{cage-2} benchmark show that \textsc{c-pomcp} achieves state-of-the-art performance with respect to effectiveness and is two orders of magnitude more efficient in computing time than the closest competitor method.
%This finding advocates for the integration of causal structure into methods for autonomous cyber defense.
%and shows that, tree search, which has been successfully in other domains, can be effective in the context of autonomous cyber defense.p
\end{abstract}
\section{Introduction}
\let\thefootnote\relax\footnotetext{\textsc{Distribution Statement A} (Approved for Public Release, Distribution Unlimited). This research is supported by the Defense Advanced Research Project Agency (\textsc{darpa}) through the \textsc{castle} program under Contract No. \texttt{W912CG23C0029}. The views, opinions, and/or findings expressed are those of the authors and should not be interpreted as representing the official views or policies of the Department of Defense or the U.S. Government.}
Domain experts have traditionally defined and updated an organization's security strategy. Though this approach can offer basic security for an organization’s IT infrastructure, a growing concern is that infrastructure update cycles become shorter and attacks increase in sophistication. To address this challenge, significant research efforts to automate the process of obtaining effective security strategies have started \cite{nework_security_alpcan,Miehling_control_theoretic_approaches_summary, control_rl_reviews}. A driving factor behind this research is the development of evaluation benchmarks, which allow researchers to compare the performance of different methods. Presently, the most popular benchmark is the Cyber Autonomy Gym for Experimentation 2 (\textsc{cage-2}) \cite{cage_challenge_2_announcement}, which involves defending a networked system against an Advanced Persistent Threat (\apt), see Fig. \ref{fig:use_case}.

At the time of writing, more than $30$ methods have been evaluated against \textsc{cage-2} \cite{cage_challenge_2_announcement}. Detailed descriptions of some methods can be found in \cite{vyas2023automated,wolk2022cage,alan_turing_1, foley_cage_1, foley2023inroads,sussex_1, TANG2024103871,kiely2023autonomous,Richer2023,cage_cognitive,doi:10.1177/1071181322661504,tabular_Q_andy,wiebe2023learning,10476122,yan2024depending,cheng2024rice,llm_cage_2_5}. While good results have been obtained, key aspects remain unexplored. For example, current methods are narrowly focused on \textit{offline} reinforcement learning and require a lengthy training phase to obtain effective strategies. Further, these methods are \textit{model-free} and do not provide provably optimal strategies. In addition, current methods provide limited ways to include domain expertise in the learning process, though attempts have been made with reward shaping \cite{alan_turing_1}.

\begin{figure}
  \centering
  \scalebox{1.2}{
    \input{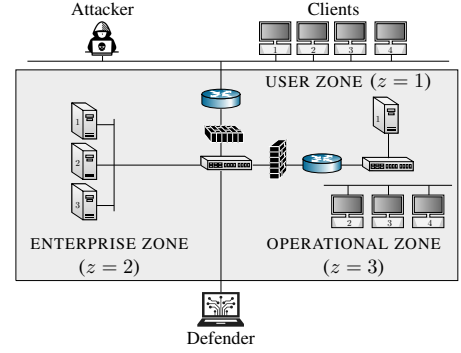}
  }
  \caption{The \textsc{cage}-2 scenario \cite{cage_challenge_2_announcement}: a defender aims to protect a networked system against an Advanced Persistent Threat (\apt{}) caused by an attacker while maintaining services for clients; the system configuration is listed in Appendix \ref{appendix:infrastructure_configuration}.}
  \label{fig:use_case}
\end{figure}
% Further, despite the growing interest in \textsc{cage-2}, no prior work has provided a formal model of the \textsc{cage-2} scenario.
%Using this model, we formulate the \textsc{cage-2} benchmark as the problem of finding an optimal \textit{intervention strategy} for the defender \cite[Def. 3.2.1]{pearl2000causality}.
In this paper, we address the above limitations and use the \textsc{cage-2} scenario to illustrate and evaluate our solution method. First, we develop a formal (causal) model of \textsc{cage-2}, which allows us to define and prove the existence of an optimal defender strategy. This model is based on the source code of \textsc{cage-2} and is formalized as a Structural Causal Model (\scm) \cite[Def 7.1.1]{pearl2000causality}. We prove that this \scm is equivalent to a specific Partially Observed Markov Decision Process (\pomdp) \cite[P.1]{ASTROM1965174}. Compared to the \pomdp{}, our \scm offers a more expressive representation of the underlying causal structure, allowing us to understand the causal effects of defender strategies \cite[Def. 3.2.1]{pearl2000causality}.

Second, we design an online method that produces a \textit{provably optimal} defender strategy, which we call \textit{Causal Partially Observable Monte-Carlo Planning (\textsc{c-pomcp})}. The method has two key properties: (1) it incorporates causal information of the target system in the form of a \textit{causal graph} \cite[Def. 2.2.1]{pearl2000causality}, which allows us to prune the search space of defender strategies; and (2) it is an \textit{online} method that updates the defender strategy at each time step via \textit{tree search}.

Our causal model represents one of many ways of formally modeling \textsc{cage-2}. A key question is the level of abstraction at which \textsc{cage-2} is modeled. The more detailed we construct the model, the closer it can capture the \textsc{cage-2} implementation. However, this comes at the expense of higher computational complexity and lower generalization ability. When balancing this trade-off, we follow the principle that a model should be detailed enough so that a theoretically optimal defender strategy exhibits state-of-the-art performance in a practical implementation \cite{cage_challenge_2_announcement}.

We evaluate \textsc{c-pomcp} against the \textsc{cage-2} benchmark and show that it achieves state-of-the-art effectiveness while being two orders of magnitude more computationally efficient than the closest competitor method: \textsc{cardiff-ppo} \cite{vyas2023automated}. The evaluation results also show that \textsc{c-pomcp} performs significantly better than its non-causal version: \textsc{pomcp} \cite[Alg. 1]{pomcp}. While prior work has focused on offline methods that require hours of training, \textsc{c-pomcp} produces equally effective defender strategies through less than $15$ seconds of online search.

Our contributions can be summarized as follows:
\begin{itemize}
\item We present a causal model of the \textsc{cage-2} scenario (\hyperref[pos]{M1}). This model allows us to define and prove the existence of optimal defender strategies (Thm. \ref{thm:well_defined_prob}).
\item We design \textsc{c-pomcp}, an online method that leverages the causal structure of the target system to efficiently find an optimal defender strategy (Alg. \ref{alg:c_pomcp}). \textsc{c-pomcp} includes a novel approach to leverage causal information for tree search, which may be of independent interest. The code is available at \cite{csle_docs}.
\item We prove that \textsc{c-pomcp} converges to an optimal strategy with increasing search time (Thm. \ref{thm:pomcp_convergence}).
\item We evaluate \textsc{c-pomcp} against the \textsc{cage-2} benchmark. The results show that \textsc{c-pomcp} outperforms the current state-of-the-art methods in effectiveness and performs significantly better in computational efficiency \cite{vyas2023automated}.
\end{itemize}
\section{Related Work}\label{sec:related_work}
Autonomous cyber defense is an active area of research that uses concepts and methods from various fields (see Fig. \ref{fig:related_work}): reinforcement learning \cite{control_rl_reviews,4725362,janisch2023nasimemu,hammar_stadler_cnsm_20,hammar_stadler_cnsm_21,hammar_stadler_cnsm_22,hammar_stadler_tnsm,hammar_stadler_tnsm_23,kim_gamesec23,9833086,kunz2023multiagent,foley2023inroads,tifs_hlsz_extended,Malialis2013MultiagentRT,nyberg2023training,kiely2023autonomous,Richer2023,tabular_Q_andy, wiebe2023learning,10476122,cheng2024rice,TANG2024103871}, control theory \cite{feedback_control_computing_systems,Miehling_control_theoretic_approaches_summary,7011201,dsn24_hammar_stadler,li2024conjectural,Kreidl2004FeedbackCA,scada_control_example,8941015,control_ddos_10}, causal inference \cite{causal_neil_agent,highnam2023adaptive,7946131,8999155}, game theory \cite{nework_security_alpcan,tambe,levente_book,9087864,5270307,dynamic_game_linan_zhu,apt_rl_simulation,DBLP:journals/compsec/HorakBTKK19,posg_cyber_deception_network_epidemic,stocahstic_games_security_indep_nodes_nguyen_alpcan_basar,9328143,ddos_game_smart_grid}, rule-based systems \cite{playbook_response, snort, trellix,adepts_irs,1006572,pbft}, large language models \cite{rigaki2023cage,moskal2023llms,yan2024depending,llm_cage_2_5}, evolutionary computation \cite{hemberg_oreily_evo,cyberevo_hemberg,armsrace_malware,evo_apt,evo_ddos,4024067}, and general optimization \cite{7127023,6514999,wang2024intrusion,quan_opt_5,7054460}. Several of these works use the \textsc{cage-2} benchmark \cite{cage_challenge_2_announcement} for evaluating their methods, see for example \cite{vyas2023automated,wolk2022cage,alan_turing_1, foley_cage_1, foley2023inroads,sussex_1,TANG2024103871,kiely2023autonomous,Richer2023,tabular_Q_andy,wiebe2023learning,10476122,cheng2024rice,yan2024depending,llm_cage_2_5}. The best benchmark performance is achieved by those methods that are based on deep reinforcement learning, where the current state-of-the-art methods use Proximal Policy Optimization (\textsc{ppo}) \cite[Alg. 1]{ppo}\cite{vyas2023automated}.

To our knowledge, no prior work has provided a formal model of \textsc{cage-2}, nor considered tree search for finding effective defender strategies. Moreover, the only prior works that use causal inference are \cite{causal_neil_agent,highnam2023adaptive,7946131,8999155}. This paper differs from them in two ways. First, the studies presented in \cite{highnam2023adaptive,7946131,8999155} use causality for analyzing the effects of attacks and countermeasures but do not present methods for finding defender strategies. Second, the method for finding defender strategies in \cite{causal_neil_agent} uses Bayesian optimization and is \textit{myopic}, i.e., it does not consider the future when selecting strategies. While this approach simplifies computations, the method is sub-optimal for most practical scenarios. By contrast, our method is non-myopic and produces optimal strategies (Thm. \ref{thm:pomcp_convergence}).
\begin{figure}
  \centering
  \scalebox{0.65}{
    \input{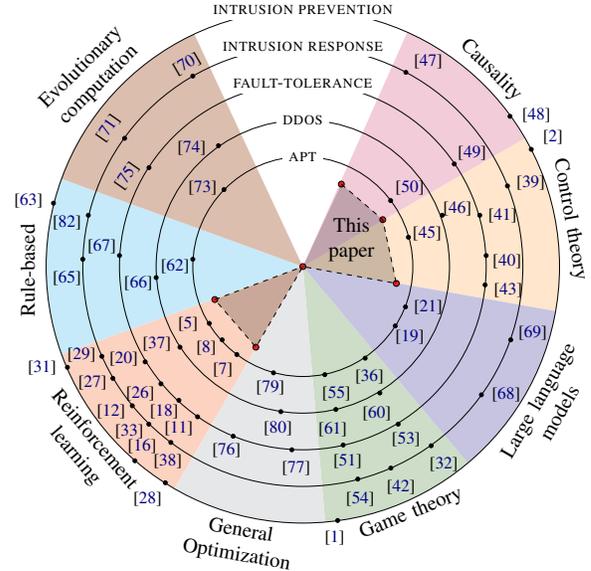}
  }
\caption{Related work on autonomous cyber defense; this paper addresses \apt defense using causality, control theory, and reinforcement learning.}
  \label{fig:related_work}
\end{figure}
\section{The cage-2 Scenario}\label{sec:use_case}
The \textsc{cage-2} scenario involves defending a networked system against \apt{}s \cite{cage_challenge_2_announcement}. The operator of the system, which we call the defender, takes measures to protect it against an attacker while providing services to a client population (see Fig. \ref{fig:use_case}). The system is segmented into \textit{zones} with \textit{nodes} (servers and workstations) that run network services. Services are realized by \textit{workflows} that are accessed by clients through a gateway, which is also open to the attacker. The detailed system configuration can be found in Appendix \ref{appendix:infrastructure_configuration}.

The attacker aims to intrude on the system and disrupt service for clients. To achieve this goal, it can take five actions: scan the network to discover nodes; exploit a vulnerability to compromise a node; perform a brute force attack to obtain login credentials to a node; escalate privileges on a compromised node to gain root access; and disrupt the service on a compromised node.
%These five action types constitute a so-called \textit{cyber kill chain} \cite{kill_chain}.

The defender monitors the system through log files and network statistics. It can make four types of \textit{interventions} on a node to prevent, detect, and respond to attacks: analyze the node for a possible intrusion; start a decoy service on the node; remove malware from the node; and restore the node to a secure state, which temporarily disrupts its service. When deciding between these interventions, the defender balances two conflicting objectives: maximize service utility towards its clients and minimize the cost of attacks.
\section{Causal Inference Preliminaries}
This section covers notation and provides an overview of causal inference, which lays the foundation for the subsequent section, where we deduce a causal model of the \textsc{cage-2} scenario.

\vspace{2mm}

\noindent\textbf{Notation.} Random variables are denoted by upper-case letters (e.g., $X$), their values by lower-case letters (e.g., $x$), and their domains by $\mathrm{dom}(\cdot)$. $X \indep Y$ means that $X$ and $Y$ are independent. $P$ is a probability measure. (Since we focus on countable sample spaces, the construction of the underlying probability space is standard.) The expectation of $f$ with respect to $X$ is written as $\mathbb{E}_X[f]$, and $x \sim f$ means that $x$ is sampled from $f$. (As the sample spaces are countable, no question of the existence of $\mathbb{E}_X[f]$ will arise.) If $f$ includes many random variables that depend on $\pi$, we simply write $\mathbb{E}_{\pi}[f]$. We use $P(x)$ as a shorthand for $P(X=x)$. Random vectors (or sets) and their values are written in bold upper-case and lower-case letters, respectively (e.g., $\X$ and $\x$). A (column) vector is written as $(x_1, x_2,\hdots)$ and a $1$-dimensional vector as $(x_1,)$. $\mathbbm{1}$ is the indicator function. $\pa{X}{\graph}, \ch{X}{\graph}, \an{X}{\graph}$ and $\de{X}{\graph}$ denote parents, children, ancestors, and descendants of node $X$ in a graph $\graph$. $\mathcal{G}[\mathbf{V}]$ denotes the subgraph obtained by restricting $\mathcal{G}$ to the nodes in $\mathbf{V}$. Further notation is listed in Table \ref{tab:notation}.

\begin{table}
  \centering
  \scalebox{0.87}{
  \begin{tabular}{ll} \toprule
    {\textit{Notation(s)}} & {\textit{Description}} \\ \midrule
    $\mathcal{G}_{\mathrm{S}}$, $\mathcal{G}_{\mathrm{W}}$ & System graph and workflow graph (\S \ref{sec:infrastructucture})\\
    $\mathcal{V}$, $\mathcal{E}$, $\mathcal{Z}$ & Set of nodes and edges in $\mathcal{G}_{\mathrm{S}}$, set of zones (\S \ref{sec:infrastructucture})\\
    $\mathbf{D}_t, \mathbf{I}_t, \mathbf{S}_t$ & Decoy states, intrusion states, and service states (\S \ref{sec:infrastructucture})\\
    $\mathsf{U}, \mathsf{K}, \mathsf{S}$ & The unknown, known, and scanned intrusion states (Fig. \ref{fig:attacker_actions})\\
    $\mathsf{C}, \mathsf{R}$ & The compromised and root intrusion states (Fig. \ref{fig:attacker_actions})\\
    $\mathbf{A}_t, \alpha_{t}$ & Attacker action and attacker action type (\S \ref{sec:attacker})\\
    $V_t, P_t, T_t$ & Vulnerability, privileges, and target of attacker action (\S \ref{sec:attacker})\\
    $\mathtt{S}, \mathtt{E}, \mathtt{P}$ & Scan, exploit, and privilege escalation attacker actions (\S \ref{sec:attacker})\\
    $\mathtt{I}, \mathtt{D}$ & Impact and discover attacker actions (\S \ref{sec:attacker})\\
    $f_{\mathrm{I}}$, $f_{\mathrm{S}}, f_{\mathrm{C}}$ & Causal functions for $I_{i,t}$ (\ref{eq:intrusion_state_fun}), $S_{i,t}$ (\ref{eq:service_state_fun}), and $C_{t}$ (\ref{eq:client_function})\\
    $Z_{i,t}$, $\mathbf{Z}_t$ & Observation of node $i$ and observations for all nodes (\S \ref{sec:observations})\\
    $f_{\mathrm{Z},i}$, $W_{i,t}$, $\mathbf{W}_t$ & Causal function for $Z_{i,t}$ (\ref{eq:obs_function}), noise variable, noise variables \\
    $C_t$, $\mathscr{A}_t$, $\mathscr{D}_t$ & Number of clients, arrivals, and departures (\ref{eq:client_function}) \\
    $\widehat{\mathbf{X}}_t, \widehat{\mathbf{x}}_t,\mathscr{T}$ & Intervention variables and values (\S \ref{sec:defender_interventions}), search operator (\ref{eq:planning_operator})\\
    $\mathrm{do}(\widehat{\mathbf{X}}_t=\widehat{\mathbf{x}}_t)$ & Intervention at time $t$ (\ref{eq:defender_interventions})\\
    $\mathrm{do}(\mathbf{X}^{\star}_t=\mathbf{x}^{\star}_t)$ & Optimal intervention at time $t$ \\
    $R_t$, $J$ & Defender reward at time $t$ and defender objective (\ref{eq:defender_objective}) \\
    $f_{\mathrm{R}}$, $f_{\mathrm{J}}, f_{\mathrm{D}}$ & Causal functions of $R_t$, $J$ (\ref{eq:defender_objective}), and $\mathbf{D}_{i,t}$ \\
    $\U, \V$ & Exogeneous and endogenous variables (\hyperref[pos]{M1})\\
    $\X_{t}, \N_t$ & Manipulative, non-manipulative variables (\S \ref{sec:scm})\\
   $\mathbf{O}_{t}, \mathbf{L}_t, f_{\mathrm{A}}$ & Observed and latent variables (\S \ref{sec:scm}), causal function of $\mathbf{A}_{t}$ (\ref{eq:attacker_function})\\
    $\mathbf{Y}_t$, $\mathcal{D}$ & Target variables (\S \ref{sec:scm}) at time $t$, set of decoys (Appendix \ref{appendix:infrastructure_configuration})\\
    $\mathscr{M},\mat{F}$ & \scm and causal functions (\hyperref[pos]{M1})\\
    $\mathbf{o}_t, \mathcal{T}$ & Observation (\ref{eq:feedback}) and time horizon (\ref{eq:defender_objective})\\
    $\mathbf{H}_{t}, \mathbf{h}_t, \gamma$ & History and its realization (\ref{eq:history_def}), discount factor \\
    $\pi_{\mathrm{A}}, \pi_{\mathrm{D}} \in \Pi$ & Attacker and defender strategies (\S \ref{sec:attacker} and \S \ref{sec:defender_interventions}) \\
    $\pi^{\star}_{\mathrm{D}},\mathfrak{R}$ & Optimal defender strategy (Thm. \ref{thm:well_defined_prob}), cumulative regret (\ref{eq:regret})\\
    $q_t, \psi_{z_i}, \beta_{z_i}$ & Parameters in the defender objective (\ref{eq:defender_objective}) \\
    $\widehat{\mathbf{b}}_t, \mathbf{P}_{\mathcal{G}}^{\star}$ & Belief state (\ref{eq:belief_def}), the set of \pomis{}s (Def. \ref{def:pomis})\\
    $\mathbf{\Sigma}_t$, $\bm{\sigma}, \mathcal{G}$ & Markov state and its realization (Thm. \ref{thm:well_defined_prob}), causal graph (Fig. \ref{fig:causal_diagram})\\
    $M, s_{\mathrm{T}}, c$ & Number of particles, search time, exploration parameter (\ref{eq:ucb}) \\
    \bottomrule\\
  \end{tabular}}
  \caption{Notation.}\label{tab:notation}
\end{table}
\subsection{Structural Causal Models}\label{sec:scms_intro}
A Structural Causal Model (\scm{}) is defined as
\begin{align}
\mathscr{M} \triangleq \left\langle \U,\V,\mat{F}, \myP{\U} \right\rangle,\label{eq:scm_def}
\end{align}
where $\U$ is a set of \textit{exogenous} random variables and $\mat{V}$ is a set of \textit{endogenous} random variables \cite[Def 7.1.1]{pearl2000causality}. Within $\V \cup \U$ we distinguish between five subsets that may overlap: the set of \textit{manipulative} variables $\X$; \textit{non-manipulative} $\N$; \textit{observed} $\mathbf{O}$; \textit{latent} $\mathbf{L}$; and \textit{targets} $\Y$. (\scm{}s with latent variables are called ``partially observable'' \cite[Def. 1]{10.5555/3495724.3496752}.)

An \scm induces a \textit{causal graph} $\graph$ \cite[Def. 2.2.1]{pearl2000causality}, where nodes correspond to $\V \cup \U$ and edges represent (causal) functions $\mat{F} \triangleq \{f_i\}_{V_i \in \mat{V}}$. A function $f_i$ is a mapping from the domains of a subset $\mathbf{K} \subseteq (\mathbf{U} \cup \pa{V_i}{\graph})$ to the domain of $V_i$, which is represented graphically by directed edges from the nodes in $\mathbf{K}$ to $V_i$ (see Fig. \ref{fig:causal_graphs_intro}). If each function is independent of time, the \scm{} is said to be \textit{stationary}.
%and if both $\mat{U}$ and $\mat{V}$ are finite, the \scm{} is said to be finite.

Causal graphs with latent variables can be drawn in two ways (cf. Fig. \ref{fig:causal_graphs_intro}.a and Fig. \ref{fig:causal_graphs_intro}.b). One option is to include the latent variables in the graph (Fig. \ref{fig:causal_graphs_intro}.a). Another option is to represent the latent variables with bidirected edges, where a bidirected edge between two observed variables means that they share an \textit{unobserved confounder} \cite[Def. 6.2.1]{pearl2000causality}, i.e., a latent variable that influences both of them (Fig. \ref{fig:causal_graphs_intro}.b).

\begin{figure}
  \centering
  \scalebox{1.1}{
    \begin{tikzpicture}[fill=white, >=stealth,
    node distance=3cm,
    database/.style={
      cylinder,
      cylinder uses custom fill,
      shape border rotate=90,
      aspect=0.25,
      draw}]

    \tikzset{
node distance = 9em and 4em,
sloped,
   box/.style = {%
    shape=rectangle,
    rounded corners,
    draw=blue!40,
    fill=blue!15,
    align=center,
    font=\fontsize{12}{12}\selectfont},
 arrow/.style = {%
    line width=0.1mm,
    -{Triangle[length=5mm,width=2mm]},
    shorten >=1mm, shorten <=1mm,
    font=\fontsize{8}{8}\selectfont},
}

\node[scale=0.7] (gi) at (3,0)
{
\begin{tikzpicture}
\node[draw,circle, minimum width=0.8cm, scale=0.7](x) at (0,0) {};
\node[draw,circle, minimum width=0.8cm, scale=0.7](z) at (1.25,0) {};
\node[draw,circle, minimum width=0.8cm, scale=0.7](y) at (2.5,0) {};
\draw[{Latex[length=2mm]}-{Latex[length=2mm]}, line width=0.22mm, color=black, dashed, bend left=50] (x) to (y);
\draw[-{Latex[length=2mm]}, line width=0.22mm, color=black] (x) to (z);
\draw[-{Latex[length=2mm]}, line width=0.22mm, color=black] (z) to (y);
\node[inner sep=0pt,align=center, scale=0.8] (dots4) at (0,0) {$X$};
\node[inner sep=0pt,align=center, scale=0.8] (dots4) at (1.25,0) {$Z$};
\node[inner sep=0pt,align=center, scale=0.8] (dots4) at (2.5,0) {$Y$};
\end{tikzpicture}
};

\node[scale=0.7] (gi) at (0,0)
{
  \begin{tikzpicture}
\node[draw,circle, minimum width=0.8cm, scale=0.7](x) at (0,0) {};
\node[draw,circle, minimum width=0.8cm, scale=0.7](z) at (1.25,0) {};
\node[draw,circle, minimum width=0.8cm, scale=0.7](y) at (2.5,0) {};
\node[draw,circle, minimum width=0.8cm, scale=0.7, fill=black!12](u) at (1.25,0.75) {};
\draw[-{Latex[length=2mm]}, line width=0.22mm, color=black] (x) to (z);
\draw[-{Latex[length=2mm]}, line width=0.22mm, color=black] (z) to (y);
\draw[-{Latex[length=2mm]}, line width=0.22mm, color=black,dashed] (u) to (y);
\draw[-{Latex[length=2mm]}, line width=0.22mm, color=black,dashed] (u) to (x);
\node[inner sep=0pt,align=center, scale=0.8] (dots4) at (1.25,0.75) {$U$};
\node[inner sep=0pt,align=center, scale=0.8] (dots4) at (0,0) {$X$};
\node[inner sep=0pt,align=center, scale=0.8] (dots4) at (1.25,0) {$Z$};
\node[inner sep=0pt,align=center, scale=0.8] (dots4) at (2.5,0) {$Y$};
\end{tikzpicture}
};

\node[inner sep=0pt,align=center, scale=0.65] (dots4) at (0,-0.75) {a)};
\node[inner sep=0pt,align=center, scale=0.65] (dots4) at (3,-0.75) {b)};

\end{tikzpicture}
  }
  \caption{Causal graphs \cite[Def. 2.2.1]{pearl2000causality}; circles represent variables in an \scm (\ref{eq:scm_def}); solid arrows represent causal relations, and dashed edges represent effects caused by latent variables; latent variables can either be represented with shaded circles or with bidirected dashed edges, i.e., the graphs in a) and b) represent the same causal structure.}
  \label{fig:causal_graphs_intro}
\end{figure}
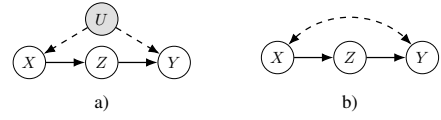

We say that $P(\mathbf{V})$ is Markov relative to $\mathcal{G}$ if it admits the following factorization \cite[Def. 1.2.2, Thm. 1.2.7]{pearl2000causality}
\begin{align}
P(\V) = \prod_{i=1}^{|\V|}P(\V_{i} \mid \mathrm{pa}(\V_{i})_{\mathcal{G}}). \label{eq:parental_markov}
\end{align}
Similarly, we say that an \scm is Markov if it induces a distribution over the observables $\mathbf{O}$ that satisfies (\ref{eq:parental_markov}) \cite[Thm. 1.4.1]{pearl2000causality}. If the \scm is not Markov and $\mathcal{G}$ is acyclic, we say that it is \textit{semi}-Markov \cite[Ch. 3]{pearl2000causality}.

\subsection{Interventions and Causal Effect Identifiability}
The operator $\mathrm{do}(\mathbf{X}=\mathbf{x})$ represents an \textit{atomic intervention} that fixes a set of endogenous variable(s) $\mathbf{X}$ to constant value(s) $\mathbf{x}$ irrespective of the functions $\mathbf{F}$ \cite[Def. 3.2.1]{pearl2000causality}. Similarly, $\mathrm{do}(\mathbf{X}=\pi(\mathbf{O}))$ represents a \textit{conditional intervention}, whereby the function(s) $\{f_{i}\}_{i\in \mathbf{X}}$ are replaced with a deterministic function $\pi$ of the observables. We call such a function an \textit{intervention strategy}.

The standard way to estimate causal effects of interventions \cite[Def. 3.2.1]{pearl2000causality} is through controlled experiments \cite{fisher:1935}. In practice, however, experimentation can be costly and is often not feasible in operational systems. This leads to the fundamental question of whether causal effects can be estimated only from observations. Such estimation can be performed using Pearl's \textit{do-calculus}, which is an axiomatic system for replacing expressions that contain the do-operator with conditional probabilities \cite[Thm. 3.4.1]{pearl2000causality}.

In case an \scm includes latent variables \cite[Def. 2.3.2]{pearl2000causality}, the question of \textit{identifiability} arises:
\begin{definition}[Causal effect identifiability \protect{\cite[Def. 3.2.4]{pearl2000causality}}]\label{def:ID}
The causal effect \cite[Def. 3.2.1]{pearl2000causality} of $\mathrm{do}(\mathbf{X}=\pi(\mathbf{O}))$ on $Y$ is identifiable from $\graph$ if $P(Y \mid \mathrm{do}(\mathbf{X}=\pi(\mathbf{O})), \mathbf{O})$ is \emph{uniquely} computable from $P(\mathbf{O}) > 0$ in every \scm conforming to $\graph$.
\end{definition}
Do-calculus is \textit{complete} in that it allows to derive all identifiable causal effects \cite[Cor. 3.4.2]{pearl2000causality}\cite[Thm. 23]{shpitser2008complete}. Consequently, one can prove identifiability by providing a do-calculus derivation that reduces the causal effect to an expression involving only $P(\mathbf{O})$.
\subsection{Automatic Intervention Control}
The problem of finding a sequence of conditional interventions $\mathrm{do}(\mathbf{X}_1=\pi(\mathbf{O})), \hdots, \mathrm{do}(\mathbf{X}_{\mathcal{T}}=\pi(\mathbf{O}))$ that maximizes a target variable $J$ can be formulated as a \textit{feedback control problem} \cite{bert05}, also known as a \textit{dynamic treatment regime problem} \cite{murphy_dtr}. We say that such a problem is identifiable if the effect on $J$ caused by every intervention strategy $\pi$ is identifiable:
\begin{definition} [Control problem identifiability \protect{\cite[Ch. 4.4]{pearl2000causality}}]\label{def:control_id}
$\text{ }$ A control problem with target $J$ and time horizon $\mathcal{T}$ is identifiable from $\graph$ if $P(J \mid \mathrm{do}(\mathbf{X}_1=\pi(\mathbf{O})), \hdots, \mathrm{do}(\mathbf{X}_{\mathcal{T}}=\pi(\mathbf{O})))$ is identifiable for each intervention strategy $\pi$ (Def. \ref{def:ID}).
\end{definition}

Given a control problem and a causal graph, we can derive \emph{possibly optimal minimal intervention sets} (\pomis{}s):
\begin{definition}[\pomis, adapted from \protect{\cite[Def. 3]{lee2019structural}}]\label{def:pomis}
$\quad\quad\quad$ Given a control problem with target $J$ and a causal graph $\mathcal{G}$, $\tilde{\mathbf{X}} \subseteq \mathbf{X}$ is a \pomis if, for each \scm conforming to $\mathcal{G}$, there is no $\mathbf{X}^{\prime} \subset \tilde{\mathbf{X}}$ such that $\mathbb{E}_{\pi}[J \mid \mathrm{do}(\tilde{\mathbf{X}} = \tilde{\mathbf{x}})] = \mathbb{E}_{\pi}[J \mid \mathrm{do}(\mathbf{X}^{\prime} = \mathbf{x}^{\prime})]$ and there exists an \scm such that
\begin{align}
\mathbb{E}_{\pi^{\star}}[J \mid \mathrm{do}(\tilde{\mathbf{X}}=\tilde{\mathbf{x}})] \geq \mathbb{E}_{\pi^{\star}}[J \mid \mathrm{do}(\mathbf{X}^{\prime}=\mathbf{x}^{\prime})],\label{eq:pomis_condition}
\end{align}
for all $\pi$, $\mathbf{X}^{\prime}$, and $\mathbf{x}^{\prime}$, where $\pi^{\star}$ is an optimal intervention strategy satisfying $\mathbb{E}_{\pi^{\star}}[J] \geq \mathbb{E}_{\pi}[J] \text{ }\forall \pi$.
\end{definition}
Let $\mathbf{P}_{\mathcal{G}}^{\star}$ denote the set of \pomis{}s for a causal graph $\mathcal{G}$. $\mathbf{P}_{\mathcal{G}}^{\star}$ for two example graphs are shown in Fig. \ref{fig:pomis}. As can be seen in Fig. \ref{fig:pomis}.a, when $\mathcal{G}$ is Markovian, and all variables except the target are manipulative, the only \pomis is the set of parents of the target \cite[Prop. 2]{lee2019structural}. When there are unobserved confounders \cite[Def. 6.2.1]{pearl2000causality}, however, $\mathbf{P}_{\mathcal{G}}^{\star}$ generally includes many more sets, as shown in Fig. \ref{fig:pomis}.b. An algorithm for computing $\mathbf{P}_{\mathcal{G}}^{\star}$ can be found in \cite[Alg. 1]{lee2019structural}.
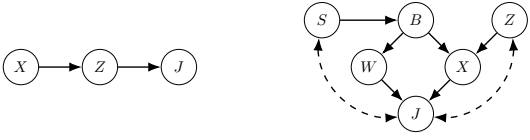
\begin{figure}
  \centering
  \scalebox{1.20}{
    \begin{tikzpicture}[fill=white, >=stealth,
    node distance=3cm,
    database/.style={
      cylinder,
      cylinder uses custom fill,
      shape border rotate=90,
      aspect=0.25,
      draw}]

    \tikzset{
node distance = 9em and 4em,
sloped,
   box/.style = {%
    shape=rectangle,
    rounded corners,
    draw=blue!40,
    fill=blue!15,
    align=center,
    font=\fontsize{12}{12}\selectfont},
 arrow/.style = {%
    line width=0.1mm,
    -{Triangle[length=5mm,width=2mm]},
    shorten >=1mm, shorten <=1mm,
    font=\fontsize{8}{8}\selectfont},
}

\node[scale=0.7] (gi) at (0,0)
{
\begin{tikzpicture}
\node[draw,circle, minimum width=0.8cm, scale=0.7](x) at (0,0) {$X$};
\node[draw,circle, minimum width=0.8cm, scale=0.7](z) at (1.25,0) {$Z$};
\node[draw,circle, minimum width=0.8cm, scale=0.7](y) at (2.5,0) {$J$};
\draw[-{Latex[length=2mm]}, line width=0.22mm, color=black] (x) to (z);
\draw[-{Latex[length=2mm]}, line width=0.22mm, color=black] (z) to (y);
\end{tikzpicture}
};

\node[scale=0.7] (gi) at (3.5,0)
{
\begin{tikzpicture}[node distance =1.5cm]
\node[draw,circle, minimum width=0.8cm, scale=0.7](s) at (0,0) {$S$};
\node[draw,circle, minimum width=0.8cm, scale=0.7, below right of = s](w) {$W$};
\node[draw,circle, minimum width=0.8cm, scale=0.7, below right of = w](y) {$J$};
\node[draw,circle, minimum width=0.8cm, scale=0.7, above right of = y](x) {$X$};
\node[draw,circle, minimum width=0.8cm, scale=0.7, above right of = x](z) {$Z$};
\node[draw,circle, minimum width=0.8cm, scale=0.7, above right of = w](b) {$B$};

\draw[-{Latex[length=2mm]}, line width=0.22mm, color=black] (s) to (b);
\draw[-{Latex[length=2mm]}, line width=0.22mm, color=black] (b) to (w);
\draw[-{Latex[length=2mm]}, line width=0.22mm, color=black] (b) to (x);
\draw[-{Latex[length=2mm]}, line width=0.22mm, color=black] (w) to (y);
\draw[-{Latex[length=2mm]}, line width=0.22mm, color=black] (x) to (y);
\draw[-{Latex[length=2mm]}, line width=0.22mm, color=black] (z) to (x);
\draw[{Latex[length=2mm]}-{Latex[length=2mm]}, line width=0.22mm, color=black, dashed, bend right=55] (s) to (y);
\draw[{Latex[length=2mm]}-{Latex[length=2mm]}, line width=0.22mm, color=black, dashed, bend left=55] (z) to (y);
\end{tikzpicture}
};

\node[inner sep=0pt,align=center, scale=0.65] (dots4) at (3.2,-0.95) {b) $\mathbf{P}_{\mathcal{G}}^{\star}=\{\emptyset, \{X\}, \{W\}, \{Z\}, \{B,W\}, \{X,W\}, \{Z,W\}\}$};
\node[inner sep=0pt,align=center, scale=0.65] (dots4) at (-0.5,-0.95) {a) $\mathbf{P}_{\mathcal{G}}^{\star}=\{\{Z\}\}$};
\end{tikzpicture}
  }
  \caption{Two causal graphs and the corresponding sets of \pomis{}s (Def. \ref{def:pomis}); $J$ is the target variable, and all other variables are manipulative.}
  \label{fig:pomis}
\end{figure}
\section{Causal Model of the cage-2 Scenario}
We model the \textsc{cage-2} scenario by constructing an \scm (\ref{eq:scm_def}) and formulate the benchmark problem as the problem of finding an optimal intervention strategy for the defender.

\subsection{Target System (Fig. \ref{fig:use_case})}\label{sec:infrastructucture}
We represent the physical topology of the target system as a directed graph $\mathcal{G}_{\mathrm{S}}\triangleq \langle \mathcal{V}, \mathcal{E}\rangle$, where nodes represent servers and workstations; edges represent network connectivity. Each node $i\in \mathcal{V}$ is (permanently) located in a zone $z_i \in \mathcal{Z}$ and has three state variables: an intrusion state $I_{i,t}$, a service state $S_{i,t}$, and a decoy state $\mathbf{D}_{i,t}$.

$I_{i,t}$ takes five values: $\mathsf{U}$ if the node is unknown to the attacker, $\mathsf{K}$ if it is known, $\mathsf{S}$ if it has been scanned, $\mathsf{C}$ if it is compromised, and $\mathsf{R}$ if the attacker has root access (see Fig. \ref{fig:attacker_actions}). Similarly, $S_{i,t}$ takes two values: $1$ if the service provided by node $i$ is accessible for clients, $0$ otherwise. Lastly, the decoy state $\mathbf{D}_{i,t}$ is a vector $(d_{i,t,1}, \hdots, d_{i,t,|\mathcal{D}|})$, where $d_{i,t,j}=1$ if decoy $j$ is active on node $i$, $0$ otherwise. The set of decoys in \textsc{cage-2} is available in Appendix \ref{appendix:infrastructure_configuration} and is denoted by $\mathcal{D}$. The initial state of node $i$ is $(I_{i,1}=\mathsf{U},S_{i,1}=1,\mathbf{D}_{i,1}=\mathbf{0})$.

A \textit{workflow graph} $\mathcal{G}_{\mathrm{W}}$ captures service dependencies among nodes (see Appendix \ref{appendix:infrastructure_configuration}). A directed edge $i\rightarrow j$ in $\mathcal{G}_{\mathrm{W}}$ means that the service provided by node $i$ is used by node $j$.
\subsection{Attacker}\label{sec:attacker}
During each time step, the attacker performs an action $\mathbf{A}_t$, which targets a single node or all nodes in a zone (in case of a scan action). The action is determined by an attacker strategy $\pi_{\mathrm{A}}$. It consists of four components $\mathbf{A}_t\triangleq (\alpha_t, V_t, P_t, T_t)$: $\alpha_t$ is the action type, $V_t$ is the vulnerability, $P_t \in \{(\mathtt{U})\text{ser}, (\mathtt{R})\text{oot}\}$ is the privileges obtained by exploiting the vulnerability, and $T_t$ is the target, which can be either a single node $i \in \mathcal{V}$ or a zone $z \in \mathcal{Z}$.

There are five attack actions: $(\mathtt{S})\text{can}$, which scans the vulnerabilities of a node; $(\mathtt{E})\text{xploit}$, which attempts to exploit a vulnerability of a node; $(\mathtt{P})\text{rivilege escalation}$, which escalates privileges of a compromised node; $(\mathtt{I})\text{mpact}$, which stops the service on a compromised node; and $(\mathtt{D})\text{iscover}$, which discovers the nodes in a zone. These actions have the following causal effects on the intrusion state $I_{i,t}$ and the service state $S_{i,t}$ \cite[Def. 3.2.1]{pearl2000causality}.
\begin{align}
&\mathbf{A}_{t-1}=f_{\mathrm{A}}(\pi_{\mathrm{A}},\{I_{i,t-1}\}_{i \in \mathcal{V}})\label{eq:attacker_function}\\
  &I_{i,t} = f_{\mathrm{I}}(I_{i,t-1}, \mathbf{A}_{t-1}, \mathbf{D}_{i,t}, E_{t}) \triangleq \label{eq:intrusion_state_fun}\\
  &\begin{dcases}
    \mathsf{K} \text{ if }T_{t-1} = z_i,\alpha_{t-1}=\mathtt{D}\\
    \mathsf{K}  \text{ if }T_{t-1} \in \mathrm{pa}(i)_{\mathcal{G}_{\mathrm{W}}},\alpha_{t-1}=\mathtt{P}\\
    \mathsf{S}  \text{ if }T_{t-1}=i,\alpha_{t-1}=\mathtt{S}\\
    \mathsf{C}  \text{ if }T_{t-1}=i,\alpha_{t-1}=\mathtt{E},P_{t-1}=\mathtt{U},\mathbf{D}_{i,V_t,t}=0,E_{t}=1\\
    \mathsf{R}  \text{ if }T_{t-1}=i,\alpha_{t-1}=\mathtt{E},P_{t-1}=\mathtt{R},\mathbf{D}_{i,V_t,t}=0,E_{t}=1\\
    \mathsf{R}  \text{ if }T_{t-1}=i,\alpha_{t-1}=\mathtt{P}\\
    I_{i,t-1}  \text{ otherwise}
  \end{dcases}\nonumber\\
&S_{i, t}= f_{\mathrm{S}}(\mathbf{A}_{t-1}, S_{i,t-1}) \triangleq
                          \begin{dcases}
                            0 & \text{if } T_{t-1}=i,\alpha_{t-1}=\mathtt{I}\\
                            S_{i,t-1} & \text{otherwise},
                          \end{dcases}\label{eq:service_state_fun}
\end{align}
where $E_{t}$ is a binary random variable and $P(E_{t}=1)$ is the probability that an exploit at time $t$ is successful (see Fig. \ref{fig:attacker_actions}).
\begin{figure}
  \centering
  \scalebox{0.85}{
    \begin{tikzpicture}[fill=white, >=stealth,
    node distance=3cm,
    database/.style={
      cylinder,
      cylinder uses custom fill,
      shape border rotate=90,
      aspect=0.25,
      draw}]

\node[scale=0.8] (kth_cr) at (0,2.15)
{
  \begin{tikzpicture}

\node[scale=1] (level1) at (-1.7,-5.6)
{
  \begin{tikzpicture}
\node[draw,circle, minimum width=15mm, scale=0.6](s0) at (0,0) {};
\node[draw,circle, minimum width=15mm, scale=0.6](s1) at (2.5,0) {};
\node[draw,circle, minimum width=15mm, scale=0.6](s2) at (5,0) {};
\node[draw,circle, minimum width=15mm, scale=0.6](s3) at (7.5,0) {};
\node[draw,circle, minimum width=15mm, scale=0.6](s4) at (10,0) {};

\node[inner sep=0pt,align=center, scale=0.85] (time) at (0.07,0)
{
$\mathsf{U}$
};

\node[inner sep=0pt,align=center, scale=0.85] (time) at (2.57,0)
{
$\mathsf{K}$
};
\node[inner sep=0pt,align=center, scale=0.85] (time) at (5.07,0)
{
$\mathsf{S}$
};
\node[inner sep=0pt,align=center, scale=0.85] (time) at (7.57,0)
{
$\mathsf{C}$
};
\node[inner sep=0pt,align=center, scale=0.85] (time) at (10.07,0)
{
$\mathsf{R}$
};

\draw[thick,-{Latex[length=2mm]}] (s0) to (s1);
\draw[thick,-{Latex[length=2mm]}] (s1) to (s2);
\draw[thick,-{Latex[length=2mm]}] (s2) to (s3);
\draw[thick,-{Latex[length=2mm]}, bend left=50] (s2) to (s4);
\draw[thick,-{Latex[length=2mm]}] (s3) to (s4);
%\draw[thick,-{Latex[length=2mm]}, bend left=25] (s3) to (s2);
%\draw[thick,-{Latex[length=2mm]}, bend left=50] (s4) to (s2);

\node[inner sep=0pt,align=center, scale=1] (time) at (1.23,0.23)
{
  $(\mathtt{D})$iscover
};

\node[inner sep=0pt,align=center, scale=1] (time) at (3.75,0.23)
{
  $(\mathtt{S})$can
};
\node[inner sep=0pt,align=center, scale=1] (time) at (7.5,1.02)
{
 $(\mathtt{E})$xploit
};

\node[inner sep=0pt,align=center, scale=0.9] (time) at (8.65,0.42)
{
  $(\mathtt{P})$rivilege\\
  escalation
};

\node[inner sep=0pt,align=center, scale=1] (time) at (6.3,0.2)
{
 $(\mathtt{E})$xploit
};
%\node[inner sep=0pt,align=center, scale=1] (time) at (6.84,-0.65)
%{
%\textit{remove}
%};
%\node[inner sep=0pt,align=center, scale=1] (time) at (7.55,-1.08)
%{
%\textit{restore}
%};
    \end{tikzpicture}
  };
    \end{tikzpicture}
  };

\end{tikzpicture}
  }
  \caption{Transition diagram of the intrusion state $I_{i,t}$ (\ref{eq:intrusion_state_fun}); self-transitions are not shown; disks represent states; arrows represent state transitions; labels indicate conditions for state transition; the initial state is $I_{i,1}=\mathsf{U}$.}
  \label{fig:attacker_actions}
\end{figure}
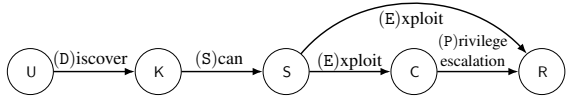

The first two cases in (\ref{eq:intrusion_state_fun}) capture the transition $\mathtt{U} \rightarrow \mathtt{K}$, which occurs when the attacker discovers the zone of node $i$. The third case defines the transition $\mathtt{K} \rightarrow \mathtt{S}$, which happens when the attacker scans the node. The fourth case captures the transition $\mathtt{S} \rightarrow \mathtt{C}$, which occurs when the attacker compromises the node. The fifth and sixth cases define the transitions $\mathtt{S} \rightarrow \mathtt{R}$ and $\mathtt{C} \rightarrow \mathtt{R}$, which occur when the attacker obtains root privileges on the node. The final case captures the recurrent transition $I_{i,t} = I_{i,t-1}$. Lastly, (\ref{eq:service_state_fun}) states that the ($\mathtt{I}$)mpact action disrupts the service.
\begin{remark}
\normalfont
The dependence between (\ref{eq:intrusion_state_fun}) and $\mathcal{G}_{\mathrm{W}}$ is unintuitive but is warranted based on the source code of \textsc{cage-2} \cite{cage_challenge_2_announcement}.
\end{remark}
\subsection{Observations and Clients}\label{sec:observations}
The defender knows the decoy state $\mathbf{D}_{i,t}$ and the service state $S_{i,t}$, but can not observe the intrusion state $I_{i,t}$ nor the attacker action $\mathbf{A}_t$. Instead of $I_{i,t}$ and $\mathbf{A}_t$, the defender observes $Z_{i,t}$, which represents network activity at node $i$.

Like the intrusion state $I_{i,t}$ (\ref{eq:intrusion_state_fun}), the activity $Z_{i,t}$ takes five values: ($\mathsf{U}$)nknown, ($\mathsf{K}$)nown, ($\mathsf{S}$)canned, ($\mathsf{C}$)ompromised, and ($\mathsf{R}$)oot. The value of $Z_{i,t}$ is influenced both by attacker actions and by clients requesting service, which we express as
\begin{align}
Z_{i,t} = f_{\mathrm{Z},i}(C_{t}, \mathbf{A}_{t-1}, W_{i,t}),\label{eq:obs_function}
\end{align}
where $W_{i,t} \in \mathbb{N}$ is a noise variable and $C_t$ represents the number of clients requesting service at time $t$, which is determined as
\begin{align}
C_{t} = f_{\mathrm{C}}(C_{t-1}, \mathscr{A}_{t}, \mathscr{D}_{t}) \triangleq \max\left[0, C_{t-1} + \mathscr{A}_{t} - \mathscr{D}_{t}\right],\label{eq:client_function}
\end{align}
where $\mathscr{A}_{t}$ and $\mathscr{D}_t$ are the number of clients that arrive and depart in the time interval $[t-1,t]$, respectively.

\subsection{Defender Objective}\label{sec:defender_objective}
The defender balances two objectives: maintain services to its clients and minimize the cost of intrusion. In \textsc{cage-2}, this bi-objective corresponds to maximizing
\begin{align}
J &= f_{\mathrm{J}}(\{R_t\mid 1 \leq t \leq \mathcal{T}\}) \triangleq \sum_{t=1}^{\mathcal{T}}\gamma^{t-1} R_t\label{eq:defender_objective}\\
R_t &= f_{\mathrm{R}}(\{I_{i,t}, S_{i,t}\}_{i\in\mathcal{V}}) \triangleq \overbrace{-q_t + \sum_{i \in \mathcal{V}}\psi_{z_i}(S_{i,t}-1)}^{\text{intervention \& downtime cost}} \overbrace{-\beta_{z_i,I_{i,t}},}^{\text{intrusion cost}}\nonumber
\end{align}
where $R_t$ is the reward at time $t$; $q_t \geq 0$ is the intervention cost; $\psi_{z_i} \geq 0$ is the cost of service disruption in zone $z_i$; $\beta_{I_{i,t},z_i} \geq 0$ is the cost of intrusion in zone $z_i$; $\gamma \in [0,1]$ is a discount factor; and $\mathcal{T}$ is the time horizon. The configuration of $q_t$, $\psi_{z_i}$, and $\beta_{I_{i,t},z_i}$ for the target infrastructure in \textsc{cage-2} (Fig. \ref{fig:use_case}) can be found in Appendix \ref{appendix:hyperparameters}.

\subsection{Defender Interventions}\label{sec:defender_interventions}
During each time step, the defender performs an intervention that targets a single node. The defender can make four types of interventions: analyze the node for a possible intrusion, start a decoy service, remove malware, and restore the node to a secure state, which temporarily disrupts its service. We model these interventions as follows.
\begin{subequations}\label{eq:defender_interventions}
\begin{align}
&\mathrm{do}(Z_{i,t}=I_{i,t})       &&\text{analyze;} \label{analyze_intervention}\\
&\mathrm{do}(\mathbf{D}_{i,j,t}=1)  &&\text{decoy;}\\
&\mathrm{do}(I_{i,t}=\mathsf{S})    \text{ if }I_{i,t-1}=\mathsf{C} &&\text{remove;}\\
&\mathrm{do}(\mathbf{D}_{i,t}=\mathbf{0}, I_{i,t}=\mathsf{S}) \text{ if } I_{i,t-1}\in \{\mathsf{C},\mathsf{R}\} &&\text{restore;}\label{eq:restore_intervention}\\
&\mathrm{do}(\emptyset)             &&\text{none.}
\end{align}
\end{subequations}
Note that $\mathbf{D}_{i,t}$ remains constant if no interventions occur, i.e., $\mathbf{D}_{i,t} = f_{\mathrm{D}}(\mathbf{D}_{i,t-1})\triangleq \mathbf{D}_{i,t-1}$.

When selecting interventions, the defender considers the \textit{history} $\mathbf{H}_t$, which we define as
\begin{align}
\mathbf{H}_{t} &\triangleq (\mathbf{V}_1,\mathrm{do}(\widehat{\mathbf{X}}_1), \mathbf{O}_2, \mathrm{do}(\widehat{\mathbf{X}}_2), \hdots, \mathrm{do}(\widehat{\mathbf{X}}_{t-1}), \mathbf{O}_{t})\label{eq:history_def}\\
\mathbf{O}_{t} &\triangleq \{\mathbf{D}_{i,t}, S_{i,t}, Z_{i,t}, C_t \mid i \in \mathcal{V}\}, \label{eq:feedback}
\end{align}
where $\mathrm{do}(\widehat{\mathbf{X}}_t)$ is a shorthand for $\mathrm{do}(\widehat{\mathbf{X}}_t=\widehat{\mathbf{x}}_t)$ and $\mathbf{V}_1$ is the set of endogeneous variables at time $t$ (defined below). The intervention at time $t$ can thus be expressed as $\mathrm{do}(\widehat{\mathbf{X}}_t=\pi_{\mathrm{D}}(\mathbf{h}_t)$, where $\pi_{\mathrm{D}}$ is a \textit{defender strategy}.
\begin{remark}
\normalfont The fact that the defender remembers the history $\mathbf{H}_{t}$ (\ref{eq:history_def}) means that it has \textit{perfect recall} \cite[Def. 7]{kuhn1953}.
\end{remark}
\subsection{A Structural Causal Model of \textsc{cage-2}}\label{sec:scm}
\begin{model}[t]
  \caption{\footnotesize Causal model of the \textsc{cage-2} scenario\hfill (\textcolor{blue!50!black}{M1})\xlabel[positive]{pos}}\label{scm_cage}
\footnotesize
\textsc{\scm}: $\mathscr{M} \triangleq \langle \U,\V,\mat{F}, \myP{\U} \rangle.$ $\quad$ \textsc{causal graph}: $\mathcal{G}$ (Fig. \ref{fig:causal_diagram}).

\vspace{1mm}

\noindent \textsc{target system}:

\vspace{1mm}

{\centering
  $ \displaystyle
    \begin{aligned}
&\mathcal{V},\mathcal{T}  && \text{set of nodes and time horizon; see Appendix \ref{appendix:infrastructure_configuration}}\\
\mathcal{G}_{\mathrm{S}}&=\langle \mathcal{V}, \mathcal{E}\rangle  && \text{physical topology graph (Fig. \ref{fig:use_case})}\\
  \mathcal{G}_{\mathrm{W}}&=\langle \mathcal{V}_{\mathrm{W}} \subseteq \mathcal{V}, \mathcal{E}_{\mathrm{W}}\rangle &&\text{workflow graph; see Appendix \ref{appendix:infrastructure_configuration}}.
    \end{aligned}
  $
\par}

\vspace{1mm}

\textsc{random variables}:

\vspace{1mm}

{\centering
  $ \displaystyle
    \begin{aligned}
\U &\triangleq \{E_t, \pi_{\mathrm{A}}, \mathscr{A}_t, \mathscr{D}_t, W_{i,t} \mid i \in \mathcal{V}, 2\leq t \leq \mathcal{T}\} \\%\quad\quad\quad\quad\quad \text{exogeneous} \\
\V &\triangleq \{I_{i,t}, Z_{i,t}, S_{i,t},\mathbf{D}_{i,t}, \mathbf{A}_{t>1}, C_{t}, R_{t}, J \mid i\in \mathcal{V}, 1\leq t \leq \mathcal{T}\}\\ %\text{ endogeneous}\\
\mathbf{X} &\triangleq \{\mathbf{D}_{i,t}, Z_{i,t}, I_{i,t} \mid i \in \mathcal{V}, 1\leq t \leq \mathcal{T}\} \\%\quad\quad\quad\quad\quad\quad\quad\text{ }\text{manipulative}\nonumber \\
\mathbf{N} &\triangleq (\V \cup \U) \setminus \mathbf{X}\\
\mathbf{O} &\triangleq \{\mathbf{D}_{i,t}, S_{i,t}, Z_{i,t}, C_t, \mid i \in \mathcal{V}, 1\leq t \leq \mathcal{T}\} \\%\quad\quad\quad\quad\quad\text{ }\text{ observable}\nonumber \\
\mathbf{L} &\triangleq (\V\cup \U) \setminus \mathbf{O}\\
%%= \{E_t, \pi_{\mathrm{A}}, \mathscr{A}_t, \mathscr{D}_t, W_{i,t},I_{i,t},\mathbf{A}_{t>1},R_t, \mid  i \in \mathcal{V}, 1\leq t \leq \mathcal{T}\}\\
\mathbf{Y} &\triangleq \{R_t, J \mid 1\leq t \leq \mathcal{T}\}. \\%\quad\quad\quad\quad\quad\quad\quad\quad\quad\quad\quad\quad\quad \text{ }\text{targets}\nonumber
%\mathbf{L} &\triangleq \V\cup \U \setminus \mathbf{O} \quad\quad \mathbf{N} \triangleq \V \cup \U \setminus \mathbf{X}.
    \end{aligned}
  $
\par}

\vspace{1mm}

\textsc{initial condition}: $I_{i,1}=Z_{i,1}=\mathsf{U}, S_{i,1}=1, \mathbf{D}_{i,1}=\mathbf{0}, C_1=R_1=0$.

\vspace{1mm}

\noindent \textsc{causal functions}: $\mathbf{F} \triangleq \{f_{\mathrm{I}}, f_{\mathrm{S}}, (f_{Z,i})_{i \in \mathcal{V}}, f_{\mathrm{C}}, f_{\mathrm{R}}, f_{\mathrm{J}}, f_{\mathrm{A}},f_{\mathrm{D}}\}.$

\vspace{1mm}

\noindent \textsc{observational distributions}:

\vspace{1mm}

{\centering
  $ \displaystyle
    \begin{aligned}
P(\U) &= P(\pi_{\mathrm{A}})\prod_{t=2}^{\mathcal{T}}P(E_t)P(\mathscr{A}_t)P(\mathscr{D}_t)P((W_{i,t})_{i \in \mathcal{V}}) \\
P(\V_t) &= \prod_{i=1}^{|\V_t|}P(\V_{i,t} \mid \mathrm{pa}(\V_{i,t})_{\graph}). %\quad\quad\quad\quad\text{causal Markov condition.}\nonumber
    \end{aligned}
  $
  \par}

\vspace{1mm}

\textsc{interventions}: $\mathrm{do}(\widehat{\mathbf{X}}_t=\pi_{\mathrm{D}}(\mathbf{H}_t))$ (\ref{eq:defender_interventions}).
\end{model}
The variables and the causal functions (\ref{eq:attacker_function})--(\ref{eq:defender_objective}) described above determine the \scm (\ref{eq:scm_def}) defined in (\hyperref[pos]{M1}). Notable properties of (\hyperref[pos]{M1}) are a) the causal graph is acyclic (see Fig. \ref{fig:causal_diagram}); b) the model is \textit{stationary}; c) the model is \textit{semi-Markov} \cite[Thm. 1.4.1]{pearl2000causality}; d) the exogeneous variables are jointly independent; and e) $P(\V_t)$ is Markov relative to $\mathcal{G}$ (\ref{eq:parental_markov}) \cite[Thm. 1.2.7]{pearl2000causality}.
%however, $P(\mathbf{O}_t)$  is not. As a consequence, (\hyperref[pos]{M1}) is \textit{semi}-Markov \cite[Thm. 1.4.1]{pearl2000causality}

The size of the causal graph $\mathcal{G}$ associated with (\hyperref[pos]{M1}) grows linearly with the time horizon $\mathcal{T}$ and with the number of nodes in the target system $|\mathcal{V}|$. A summary of $\mathcal{G}$ is shown in Fig. \ref{fig:causal_diagram}.

\begin{figure}
  \centering
  \scalebox{1}{
    \begin{tikzpicture}[fill=white, >=stealth,
    node distance=3cm,
    database/.style={
      cylinder,
      cylinder uses custom fill,
      shape border rotate=90,
      aspect=0.25,
      draw}]

    \tikzset{
node distance = 9em and 4em,
sloped,
   box/.style = {%
    shape=rectangle,
    rounded corners,
    draw=blue!40,
    fill=blue!15,
    align=center,
    font=\fontsize{12}{12}\selectfont},
 arrow/.style = {%
    line width=0.1mm,
    -{Triangle[length=5mm,width=2mm]},
    shorten >=1mm, shorten <=1mm,
    font=\fontsize{8}{8}\selectfont},
}

\node[scale=0.7] (gi) at (0,0)
{
\begin{tikzpicture}
\node[draw,circle, minimum width=0.8cm, scale=1, fill=black!12](dt) at (0,0) {};
\node[draw,circle, minimum width=0.8cm, scale=1, fill=black!12](at) at (1.5,1.5) {};
\node[draw,circle, minimum width=0.8cm, scale=1](ct) at (0,1.5) {};
\node[draw,circle, minimum width=0.8cm, scale=1](ct1) at (-1.5,1.5) {};
\node[draw,circle, minimum width=0.8cm, scale=1](ot) at (0,3) {};
\node[draw,circle, minimum width=0.8cm, scale=1, fill=black!12](wt) at (1.5,3) {};
\node[draw,circle, minimum width=0.8cm, scale=1](dtt) at (0,4.5) {};
\node[draw,circle, minimum width=0.8cm, scale=1](dt1) at (-1.5,4.5) {};
\node[draw,circle, minimum width=0.8cm, scale=1](st) at (1.5,4.5) {};
\node[draw,circle, minimum width=0.8cm, scale=1](st1) at (1.5,6) {};
\node[draw,circle, minimum width=0.8cm, scale=1, fill=black!12](rt) at (3,4.5) {};
\node[draw,circle, minimum width=0.8cm, scale=1](sj) at (4.5,4.5) {};
\node[draw,circle, minimum width=0.8cm, scale=1,fill=black!12](ij) at (3,6) {};
\node[draw,circle, minimum width=0.8cm, scale=1,fill=black!12](it) at (0,6) {};
\node[draw,circle, minimum width=0.8cm, scale=1,fill=black!12](et) at (0,7.5) {};
\node[draw,circle, minimum width=0.8cm, scale=1,fill=black!12](atk) at (-1.5,6) {};
\node[draw,circle, minimum width=0.8cm, scale=1,fill=black!12](pia) at (-3,6) {};
\node[draw,circle, minimum width=0.8cm, scale=1,fill=black!12](ijj) at (3,7.5) {};
\node[draw,circle, minimum width=0.8cm, scale=1,fill=black!12](it1) at (-3,7.5) {};
\node[draw,circle, minimum width=0.8cm, scale=1,fill=black!12](it11) at (-1.5,7.5) {};

\node[draw,circle, minimum width=0.8cm, scale=1](zjj) at (-3,4) {};
\node[draw,circle, minimum width=0.8cm, scale=1, fill=black!12](ijjj) at (-3,2.5) {};
\node[draw,circle, minimum width=0.8cm, scale=1](sjjj) at (-3,1) {};

\draw[-{Latex[length=2mm]}, line width=0.22mm, color=black, dashed] (atk) to (zjj);
\draw[-{Latex[length=2mm]}, line width=0.22mm, color=black, dashed] (atk) to (ijjj);
\draw[-{Latex[length=2mm]}, line width=0.22mm, color=black, dashed] (atk) to (sjjj);

\draw[-{Latex[length=2mm]}, line width=0.22mm, color=black] (dt) to (ct);
\draw[-{Latex[length=2mm]}, line width=0.22mm, color=black] (at) to (ct);
\draw[-{Latex[length=2mm]}, line width=0.22mm, color=black] (ct1) to (ct);
\draw[-{Latex[length=2mm]}, line width=0.22mm, color=black] (ct) to (ot);
\draw[-{Latex[length=2mm]}, line width=0.22mm, color=black, dashed] (wt) to (ot);
\draw[-{Latex[length=2mm]}, line width=0.22mm, color=black, dashed] (atk) to (ot);
\draw[-{Latex[length=2mm]}, line width=0.22mm, color=black, dashed] (atk) to (it);
\draw[-{Latex[length=2mm]}, line width=0.22mm, color=black] (dt1) to (dtt);
\draw[-{Latex[length=2mm]}, line width=0.22mm, color=black] (dtt) to (it);
\draw[-{Latex[length=2mm]}, line width=0.22mm, color=black, dashed] (it) to (rt);
\draw[-{Latex[length=2mm]}, line width=0.22mm, color=black] (st) to (rt);
\draw[-{Latex[length=2mm]}, line width=0.22mm, color=black, dashed] (atk) to (st);
\draw[-{Latex[length=2mm]}, line width=0.22mm, color=black, dashed] (pia) to (atk);
\draw[-{Latex[length=2mm]}, line width=0.22mm, color=black, dashed] (et) to (it);
\draw[-{Latex[length=2mm]}, line width=0.22mm, color=black, dashed] (et) to (ij);
\draw[-{Latex[length=2mm]}, line width=0.22mm, color=black, dashed] (ij) to (rt);
\draw[-{Latex[length=2mm]}, line width=0.22mm, color=black] (sj) to (rt);
\draw[-{Latex[length=2mm]}, line width=0.22mm, color=black] (st1) to (st);
\draw[-{Latex[length=2mm]}, line width=0.22mm, color=black, dashed] (it) to (ijj);
\draw[-{Latex[length=2mm]}, line width=0.22mm, color=black, dashed] (it1) to (atk);
\draw[-{Latex[length=2mm]}, line width=0.22mm, color=black, dashed] (it11) to (it);

\node[inner sep=0pt,align=center, scale=1] (dots4) at (0,0) {$\mathscr{D}_t$};
\node[inner sep=0pt,align=center, scale=1] (dots4) at (1.5,1.5) {$\mathscr{A}_t$};
\node[inner sep=0pt,align=center, scale=1] (dots4) at (0,1.5) {$C_t$};
\node[inner sep=0pt,align=center, scale=1] (dots4) at (-1.5,1.5) {$C_{t-1}$};
\node[inner sep=0pt,align=center, scale=1] (dots4) at (0,3) {$Z_{i,t}$};
\node[inner sep=0pt,align=center, scale=1] (dots4) at (1.5,3) {$W_{i,t}$};
\node[inner sep=0pt,align=center, scale=1] (dots4) at (0,4.5) {$\mathbf{D}_{i,t}$};
\node[inner sep=0pt,align=center, scale=0.75] (dots4) at (-1.5,4.5) {$\mathbf{D}_{i,t-1}$};
\node[inner sep=0pt,align=center, scale=1] (dots4) at (-1.5,6) {$\mathbf{A}_{t-1}$};
\node[inner sep=0pt,align=center, scale=1] (dots4) at (-3,6) {$\pi_{\mathrm{A}}$};
\node[inner sep=0pt,align=center, scale=1] (dots4) at (0,6) {$I_{i,t}$};
\node[inner sep=0pt,align=center, scale=1] (dots4) at (0,7.5) {$E_{t}$};
\node[inner sep=0pt,align=center, scale=1] (dots4) at (1.5,4.5) {$S_{i,t}$};
\node[inner sep=0pt,align=center, scale=0.8] (dots4) at (1.5,6) {$S_{i,t-1}$};
\node[inner sep=0pt,align=center, scale=1] (dots4) at (3,4.5) {$R_{t}$};
\node[inner sep=0pt,align=center, scale=1] (dots4) at (4.5,4.5) {$S_{j,t}$};
\node[inner sep=0pt,align=center, scale=1] (dots4) at (3,6) {$I_{j,t}$};
\node[inner sep=0pt,align=center, scale=1] (dots4) at (3,7.5) {$I_{j,t}$};
\node[inner sep=0pt,align=center, scale=0.9] (dots4) at (-1.5,7.5) {$I_{j,t-1}$};
\node[inner sep=0pt,align=center, scale=0.9] (dots4) at (-3,7.5) {$I_{j,t-1}$};
\node[inner sep=0pt,align=center, scale=1] (dots4) at (-3,2.5) {$I_{j,t}$};
\node[inner sep=0pt,align=center, scale=1] (dots4) at (-3,4) {$Z_{j,t}$};
\node[inner sep=0pt,align=center, scale=1] (dots4) at (-3,1) {$S_{j,t}$};

\draw (2.35, 3.9) rectangle (5.4, 6.65);
\node[inner sep=0pt,align=center, scale=1] (dots4) at (4.4,6) {$j \in \mathcal{V} \setminus \{i\}$};

\draw (2.35, 7) rectangle (5.4, 8.1);
\node[inner sep=0pt,align=center, scale=1] (dots4) at (4.45,7.5) {$j \in \mathrm{ch}(i)_{\mathcal{G}_{\mathrm{w}}}$};

\draw (-3.9, 6.9) rectangle (-2.5, 8.6);
\node[inner sep=0pt,align=center, scale=1] (dots4) at (-1.44,8.3) {$j \in \mathrm{pa}(i)_{\mathcal{G}_{\mathrm{w}}}$};
\draw (-2.35, 6.9) rectangle (-0.55, 8.6);
\node[inner sep=0pt,align=center, scale=1] (dots4) at (-3.25,8.3) {$j \in \mathcal{V}$};

\draw (-3.9, 4.6) rectangle (-2.2, 0);
\node[inner sep=0pt,align=center, scale=1] (dots4) at (-3.05,0.25) {$j \in \mathcal{V} \setminus \{i\}$};

\node[draw,circle, minimum width=0.8cm, scale=1,fill=black!12](legend0) at (3,1) {};
\node[draw,circle, minimum width=0.8cm, scale=1](legend1) at (3,0) {};

\node[inner sep=0pt,align=center, scale=1] (dots4) at (4.1,1) {Latent};
\node[inner sep=0pt,align=center, scale=1] (dots4) at (4.28,0) {Observed};
\end{tikzpicture}
};

\end{tikzpicture}
  }
\caption{Causal (summary) graph of (\hyperref[pos]{M1}) for node $i$ \cite[Ch. 10]{peters_book}; plate notation is used to represent sets of variables \cite{plate_notation}.}
  \label{fig:causal_diagram}
\end{figure}
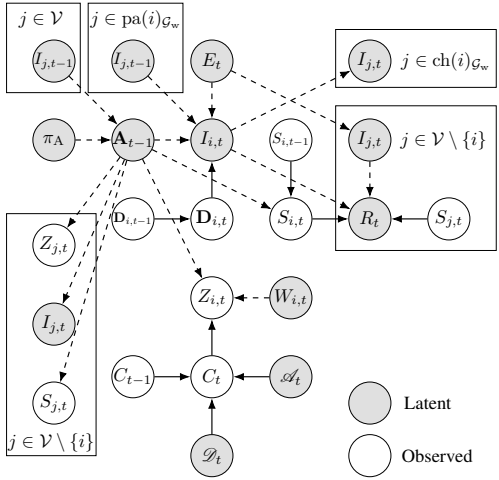
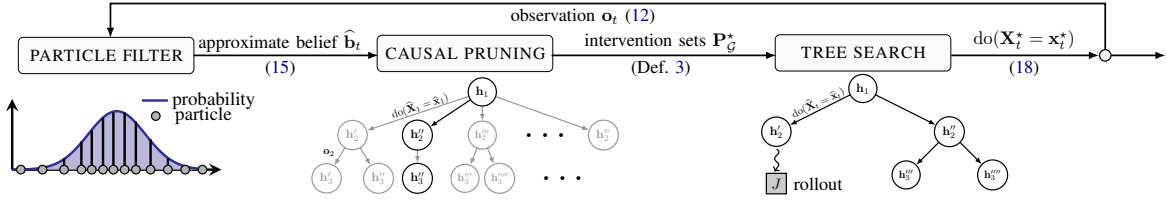
\begin{figure*}
  \centering
  \scalebox{1.15}{
    \begin{tikzpicture}

\node[scale=1] (kth_cr) at (-0.7,0)
{
\begin{tikzpicture}
\node[scale=1] (kth_cr) at (0.3,0)
{
\begin{tikzpicture}

  \def\B{11};
  \def\Bs{3.0};
  \def\xmax{\B+3.2*\Bs};
  \def\ymin{{-0.1*gauss(\B,\B,\Bs)}};
  \def\h{0.08*gauss(\B,\B,\Bs)};
  \def\N{50}

  \begin{axis}[every axis plot post/.append style={
      mark=none,domain=0:20,
      samples=\N,smooth},
               xmin=0, xmax=20,
               ymin=0, ymax={1.1*gauss(\B,\B,\Bs)},
               axis lines=middle,
               axis line style=thick,
               enlargelimits=upper, % extend the axes a bit to the right and top
               ticks=none,
%               ylabel=$\mathbb{P}(\Gamma)$,
               every axis x label/.style={at={(current axis.right of origin)},anchor=north},
               width=4cm,
               height=2.4cm,
               clip=false
              ]

    \addplot[Blue,thick,name path=B] {gauss(x,\B,\Bs)};
    \path[name path=xaxis](0,0) -- (20,0);
    \addplot[Blue!25] fill between[of=xaxis and B];
  \end{axis}
\node[draw,circle, minimum width=0.1cm, scale=0.3, fill=black!30](p1) at (0.1,0) {};
\node[draw,circle, minimum width=0.1cm, scale=0.3, fill=black!30](p2) at (0.35,0) {};
\node[draw,circle, minimum width=0.1cm, scale=0.3, fill=black!30](p3) at (0.6,0) {};
\node[draw,circle, minimum width=0.1cm, scale=0.3, fill=black!30](p4) at (0.8,0) {};
\node[draw,circle, minimum width=0.1cm, scale=0.3, fill=black!30](p5) at (0.92,0) {};
\node[draw,circle, minimum width=0.1cm, scale=0.3, fill=black!30](p6) at (1.05,0) {};
\node[draw,circle, minimum width=0.1cm, scale=0.3, fill=black!30](p7) at (1.17,0) {};
\node[draw,circle, minimum width=0.1cm, scale=0.3, fill=black!30](p8) at (1.3,0) {};
\node[draw,circle, minimum width=0.1cm, scale=0.3, fill=black!30](p9) at (1.43,0) {};
\node[draw,circle, minimum width=0.1cm, scale=0.3, fill=black!30](p10) at (1.6,0) {};
\node[draw,circle, minimum width=0.1cm, scale=0.3, fill=black!30](p11) at (1.8,0) {};
\node[draw,circle, minimum width=0.1cm, scale=0.3, fill=black!30](p12) at (2,0) {};
\node[draw,circle, minimum width=0.1cm, scale=0.3, fill=black!30](p13) at (2.2,0) {};

\draw[-, thick] (p3) to (0.6, 0.13);
\draw[-, thick] (p4) to (0.8, 0.325);
\draw[-, thick] (p5) to (0.92, 0.455);
\draw[-, thick] (p6) to (1.05, 0.6);
\draw[-, thick] (p7) to (1.17, 0.662);
\draw[-, thick] (p8) to (1.3, 0.662);
\draw[-, thick] (p9) to (1.43, 0.555);
\draw[-, thick] (p10) to (1.6, 0.35);
\draw[-, thick] (p11) to (1.8, 0.13);

\draw[-, thick, color=Blue] (1.5, 0.8) to (1.8, 0.8);
\node[draw,circle, minimum width=0.1cm, scale=0.3, fill=black!30](p12) at (1.66,0.62) {};
\node[inner sep=0pt,align=center, scale=0.62, color=black] (hacker) at (2.25,0.62) {
   particle
 };
\node[inner sep=0pt,align=center, scale=0.62, color=black] (hacker) at (2.37,0.8) {
   probability
};

\end{tikzpicture}
};
\node[scale=0.7] (box) at (-0.05,0.9)
{
\begin{tikzpicture}
  \draw[rounded corners=0.5ex, fill=black!2] (0,0) rectangle node (m1){} (2.9,0.6);
  \node[inner sep=0pt,align=center, scale=1, color=black] (hacker) at (1.52,0.3) {
    \textsc{particle filter}
  };
\end{tikzpicture}
};
\end{tikzpicture}
};

\node[scale=1] (kth_cr) at (3.1,-0.1)
{
\begin{tikzpicture}
\node[scale=1] (kth_cr) at (0,0)
{
\begin{tikzpicture}
\node[draw,circle, minimum width=0.17cm, scale=0.4](root) at (0,0) {$\mathbf{h}_1$};
\node[draw,circle, minimum width=0.15cm, scale=0.4, opacity=0.4](h1) at (-1.5,-0.5) {$\mathbf{h}^{\prime}_2$};
\node[draw,circle, minimum width=0.15cm, scale=0.4](h2) at (-0.75,-0.5) {$\mathbf{h}^{\prime\prime}_2$};
\node[draw,circle, minimum width=0.15cm, scale=0.35, opacity=0.4](h3) at (0,-0.5) {$\mathbf{h}^{\prime\prime\prime}_2$};
\node[inner sep=0pt,align=center, scale=1.2, color=black] (hacker) at (0.8,-0.5) {
  $\hdots$
};
\node[draw,circle, minimum width=0.15cm, scale=0.4, opacity=0.4](h4) at (1.4,-0.5) {$\mathbf{h}^{n}_2$};

\node[draw,circle, minimum width=0.15cm, scale=0.4, opacity=0.4](h5) at (-1.8,-1) {$\mathbf{h}^{\prime}_3$};
\node[draw,circle, minimum width=0.15cm, scale=0.4, opacity=0.4](h6) at (-1.2,-1) {$\mathbf{h}^{\prime\prime}_3$};
\node[draw,circle, minimum width=0.15cm, scale=0.4](h7) at (-0.75,-1) {$\mathbf{h}^{\prime\prime}_3$};
\node[draw,circle, minimum width=0.15cm, scale=0.35, opacity=0.4](h8) at (-0.2,-1) {$\mathbf{h}^{\prime\prime\prime}_3$};
\node[draw,circle, minimum width=0.15cm, scale=0.35, opacity=0.4](h9) at (0.2,-1) {$\mathbf{h}^{\prime\prime\prime\prime}_3$};
\node[inner sep=0pt,align=center, scale=1.2, color=black] (hacker) at (1,-1) {
  $\hdots$
};
\draw[-{Latex[length=0.8mm]}, opacity=0.4] (root) to (h1);
\draw[-{Latex[length=0.8mm]}] (root) to (h2);
\draw[-{Latex[length=0.8mm]}, opacity=0.4] (root) to (h3);
\draw[-{Latex[length=0.8mm]}, opacity=0.4] (root) to (h4);
\draw[-{Latex[length=0.8mm]}, opacity=0.4] (h1) to (h5);
\draw[-{Latex[length=0.8mm]}, opacity=0.4] (h1) to (h6);
\draw[-{Latex[length=0.8mm]}, opacity=0.4] (h2) to (h7);
\draw[-{Latex[length=0.8mm]}, opacity=0.4] (h3) to (h8);
\draw[-{Latex[length=0.8mm]}, opacity=0.4] (h3) to (h9);

\node[inner sep=0pt,align=center, scale=0.35, color=black, rotate=18] (hacker) at (-0.7,-0.15) {
  $\mathrm{do}(\widehat{\mathbf{X}}_1=\widehat{\mathbf{x}}_1)$
};
\node[inner sep=0pt,align=center, scale=0.35, color=black, rotate=0] (hacker) at (-1.75,-0.7) {
  $\mathbf{o}_{2}$
};
\end{tikzpicture}
};
\node[scale=0.7] (box) at (0,0.95)
{
\begin{tikzpicture}
  \draw[rounded corners=0.5ex, fill=black!2] (0,0) rectangle node (m1){} (2.9,0.6);
  \node[inner sep=0pt,align=center, scale=1, color=black] (hacker) at (1.52,0.3) {
    \textsc{causal pruning}
  };
\end{tikzpicture}
};
\end{tikzpicture}
};

\node[scale=1] (kth_cr) at (7.95,-0.1)
{
\begin{tikzpicture}
\node[scale=1] (kth_cr) at (0,0.11)
{
\begin{tikzpicture}
\node[draw,circle, minimum width=0.17cm, scale=0.4](root) at (0,0) {$\mathbf{h}_1$};
\node[draw,circle, minimum width=0.15cm, scale=0.4](h1) at (-1,-0.5) {$\mathbf{h}^{\prime}_2$};
\node[draw,circle, minimum width=0.15cm, scale=0.4, opacity=1](h2) at (1,-0.5) {$\mathbf{h}^{\prime\prime}_2$};
\node[draw,circle, minimum width=0.15cm, scale=0.35](h5) at (0.5,-1) {$\mathbf{h}^{\prime\prime\prime}_3$};
\node[draw,circle, minimum width=0.15cm, scale=0.35](h6) at (1.5,-1) {$\mathbf{h}^{\prime\prime\prime\prime}_3$};
\draw[-{Latex[length=0.8mm]}, opacity=1] (root) to (h1);
\draw[-{Latex[length=0.8mm]}, opacity=1] (root) to (h2);
\draw[-{Latex[length=0.8mm]}, opacity=1] (h2) to (h5);
\draw[-{Latex[length=0.8mm]}, opacity=1] (h2) to (h6);
\node[inner sep=0pt,align=center, scale=0.35, color=black, rotate=25] (hacker) at (-0.5,-0.15) {
  $\mathrm{do}(\widehat{\mathbf{X}}_t=\widehat{\mathbf{x}}_t)$
};

  \node[rotate=-90, scale=0.25](test) at (-1,-0.83) {
    \begin{tikzpicture}
\draw[->, x=0.15cm,y=0.1cm, ultra thick, black]
        (3,0) sin (4,1) cos (5,0) sin (6,-1) cos (7,0)
        sin (8,1) to (10,-1);
    \end{tikzpicture}
  };
\node at (-1,-1.1) [draw, fill=black!20, scale=0.5] (v100) {$J$};

\node[inner sep=0pt,align=center, scale=0.6, color=black, rotate=0] (hacker) at (-0.48,-1.1) {
  rollout
};
\end{tikzpicture}
};
\node[scale=0.7] (box) at (-0.25,1.05)
{
\begin{tikzpicture}
  \draw[rounded corners=0.5ex, fill=black!2] (0,0) rectangle node (m1){} (2.9,0.6);
  \node[inner sep=0pt,align=center, scale=1, color=black] (hacker) at (1.52,0.3) {
    \textsc{tree search}
  };
\end{tikzpicture}
};
\end{tikzpicture}
};

\draw[-{Latex[length=1.2mm]}, line width=0.2mm] (-0.04, 0.6) to (2.1, 0.6);
\draw[-{Latex[length=1.2mm]}, line width=0.2mm] (4.1, 0.6) to (6.7, 0.6);
\node[draw,circle, fill=gray2, scale=0.4] (j1) at (10.5,0.6) {};
\draw[-{Latex[length=1.2mm]}, line width=0.2mm] (8.72, 0.6) to (j1);

  \node[inner sep=0pt,align=center, scale=0.6, color=black] (hacker) at (9.6,0.78) {
    $\mathrm{do}(\mathbf{X}^{\star}_t=\mathbf{x}^{\star}_t)$
  };
  \node[inner sep=0pt,align=center, scale=0.6, color=black] (hacker) at (9.6,0.45) {
    (\ref{eq:planning_operator})
  };

  \node[inner sep=0pt,align=center, scale=0.6, color=black] (hacker) at (1,0.77) {
    approximate belief $\widehat{\mathbf{b}}_t$
  };
  \node[inner sep=0pt,align=center, scale=0.6, color=black] (hacker) at (1,0.45) {
    (\ref{eq:particle_filter})
  };
  \node[inner sep=0pt,align=center, scale=0.6, color=black] (hacker) at (5.4,0.75) {
    intervention sets $\mathbf{P}_{\mathcal{G}}^{\star}$
  };
    \node[inner sep=0pt,align=center, scale=0.6, color=black] (hacker) at (5.4,0.45) {
    (Def. \ref{def:pomis})
  };
  \draw[-{Latex[length=1.2mm]}, line width=0.2mm] (j1) to (11.2, 0.6);
  \draw[-{Latex[length=1.2mm]}, line width=0.2mm] (j1) to (10.5, 1.2) to (-1, 1.2) to (-1, 0.8);
  \node[inner sep=0pt,align=center, scale=0.6, color=black] (hacker) at (4.5,1.05) {
    observation $\mathbf{o}_t$ (\ref{eq:feedback})
  };

\end{tikzpicture}
  }
\caption{Causal Partially Observed Monte-Carlo Planning (\textsc{c-pomcp}, Alg. \ref{alg:c_pomcp}); the figure illustrates one time step during which (\textit{i}) a particle filter is used to compute an approximate belief state $\widehat{\mathbf{b}}_t$ (\ref{eq:belief_def}); (\textit{ii}) a causal graph \cite[Def. 2.2.1]{pearl2000causality} (see Fig. \ref{fig:causal_diagram}) is used to prune the search tree of possible histories $\mathbf{h}_k$ (\ref{eq:history_def}) by only considering histories with interventions in \pomis{}s (see Def. \ref{def:pomis}); and (\textit{iii}) tree search is used to find an optimal intervention $\mathrm{do}(\mathbf{X}^{\star}_t=\mathbf{x}^{\star}_t)$ (\ref{eq:planning_operator}).}
  \label{fig:approach}
\end{figure*}
\subsection{The Defender Problem in \textsc{cage-2}}
Given (\hyperref[pos]{M1}) and the defender objective (\ref{eq:defender_objective}), the problem for the defender can be stated as follows.
\begin{problem}[Optimal defender strategy in \textsc{cage-2} \protect{(\hyperref[pos]{M1})}]\label{prob:sec_response}
\begin{subequations}\label{eq:intruson_response_problem}
\begin{align}
    \maximize_{\pi_{\mathrm{D}}} \quad &\mathbb{E}_{\pi_{\mathrm{D}}}\left[J \mid \mathbf{V}_1\right]\\
  \text{\normalfont subject to} \quad & \mathrm{do}(\widehat{\mathbf{X}}_t=\pi_{\mathrm{D}}(\mathbf{h}_t)) && \forall t \label{eq:policy_space} \\
                                       &\pi_{\mathrm{A}} \sim P(\pi_{\mathrm{A}}) && \forall t\label{eq:exo_0} \\
                                       &\mathscr{A}_{t} \sim P(\mathscr{A}_{t}), \mathscr{D}_{t} \sim P(\mathscr{D}_{t}) && \forall t\label{eq:exo_1} \\
                                       &E_{t} \sim P(E_{t}),W_{i,t} \sim P(W_{i,t})  && \forall t,i\label{eq:exo_2} \\
                                       &\mathbf{A}_t = f_{\mathrm{A}}(\pi_{\mathrm{A}},\{I_{i,t}\}_{i \in \mathcal{V}}) && \forall t,i \label{eq:endo_0}\\
                                       &\mathbf{D}_{i,t} = f_{\mathrm{D}}(\mathbf{D}_{i,t-1})  && \forall t,i \label{eq:endo_8}\\
                          &I_{i,t} = f_{\mathrm{I}}(I_{i,t-1}, \mathbf{A}_{t-1}, \mathbf{D}_{i,t}, E_{t}) && \forall t, i\label{eq:endo_1} \\
                          &Z_{i,t} = f_{\mathrm{Z},i}(C_{t}, \mathbf{A}_{t-1}, W_{i,t}) && \forall t,i\label{eq:endo_3} \\
                          &C_{t} = f_{\mathrm{C}}(C_{t-1}, \mathscr{A}_{t}, \mathscr{D}_{t}) && \forall t\label{eq:endo_4} \\
                                       &S_{i,t} = f_{\mathrm{S}}(\mathbf{A}_{t-1},S_{i,t-1}) && \forall t,i\label{eq:endo_6}\\
                                       &R_{t} = f_{\mathrm{R}}(\{I_{i,t}, S_{i,t}\}_{i \in \mathcal{V}}) && \forall t\label{eq:endo_5}\\
                                       &J = f_{\mathrm{J}}(\{R_{t}\}_{t=1,\hdots,\mathcal{T}}),\label{eq:endo_7}
\end{align}
\end{subequations}
\end{problem}
where (\ref{eq:policy_space}) defines the interventions; (\ref{eq:exo_1})--(\ref{eq:exo_2}) capture the distribution of $\U$; and (\ref{eq:endo_0})--(\ref{eq:endo_7}) define $\mathbf{F}$.

We say that a \textit{defender strategy} $\pi^{\star}_{\mathrm{D}}$ is \textit{optimal} if it solves Prob. \ref{prob:sec_response}. This problem is well-defined in the following sense.
\begin{theorem}\label{thm:well_defined_prob}
Assuming $C_t,\mathcal{A},\mathcal{V},q_t,\beta_{I_{i,t},z},\psi_{z}$, are finite, and $\mathcal{T}$ is finite or $\gamma < 1$, then there exists an optimal deterministic defender strategy $\pi^{\star}_{\mathrm{D}}$. If $\mathcal{T}=\infty$, then there exists a $\pi^{\star}_{\mathrm{D}}$ that is stationary.
\end{theorem}
\begin{proof}
For notational convenience, let $\mathbf{S}_t\triangleq \{S_{i,t}\}_{i \in \mathcal{V}}$, $\mathbf{I}_t \triangleq \{I_{i,t}\}_{i \in \mathcal{V}}$, $\mathbf{D}_t \triangleq \{\mathbf{D}_{i,t}\}_{i \in \mathcal{V}}$, $\mathbf{W}_t\triangleq \{W_{i,t}\}_{i \in \mathcal{V}}$, and $\mathbf{Z}_t\triangleq \{Z_{i,t}\}_{i \in \mathcal{V}}$. We break down the proof into the following steps.

\vspace{1.3mm}

\noindent 1) $\mathbf{\Sigma}_t \triangleq (\mathbf{I}_t, \mathbf{D}_t, C_t, \mathbf{A}_{t-1}, \pi_{\mathrm{A}}, \mathbf{S}_{t-1})$ is Markovian.

\textsc{proof}:

\vspace{1.3mm}

{\centering
  $ \displaystyle
    \begin{aligned}
&P(\mathbf{\Sigma}_{t+1} \mid \mathbf{\Sigma}_1, \hdots, \mathbf{\Sigma}_t) = P(\mathbf{A}_{t}\mid \pi_{\mathrm{A}}, \mathbf{I}_t)P(\mathscr{A}_{t+1})P(\mathscr{D}_{t+1})\times \nonumber\\
&P(C_{t+1}\mid \mathscr{A}_{t+1},\mathscr{D}_{t+1}, C_{t})P(E_{t+1})P(\mathbf{S}_{t+1}\mid \mathbf{S}_{t},\mathbf{A}_{t})\times \nonumber\\
&P(\mathbf{I}_{t+1}\mid \mathbf{I}_t, \mathbf{A}_{t}, \mathbf{D}_{t+1}, E_{t+1})=P(\mathbf{\Sigma}_{t+1} \mid \mathbf{\Sigma}_t).\nonumber
    \end{aligned}
  $
\par}

\vspace{1.3mm}

\noindent 2) $(\mathbf{O}_{t},R_t \indep \{\mathbf{V}_t\}_{t=1,\hdots,t} \mid \mathbf{\Sigma}_t)$.

\textsc{proof}:

\vspace{1.3mm}

{\centering
  $ \displaystyle
    \begin{aligned}
&P(\mathbf{O}_{t},R_t \mid \{\mathbf{V}_t\}_{t=1,\hdots,t}) = P(\mathbf{Z}_t \mid C_t, \mathbf{A}_{t-1}, \mathbf{W}_t) \times \nonumber\\
  &P(\mathbf{W}_{t})P(\mathbf{S}_t \mid \mathbf{S}_{t-1},\mathbf{A}_{t-1})P(R_t \mid \mathbf{I}_t, \mathbf{S}_t)=P(\mathbf{O}_{t},R_t \mid \mathbf{\Sigma}_{t}).\nonumber%\\
%  &=P(\mathbf{O}_{t},R_t \mid \mathbf{\Sigma}_{t}).\nonumber
    \end{aligned}
  $
  \par}

\vspace{1.3mm}

\noindent 3) $\mathrm{dom}(\mathbf{O}_t)$, $\mathrm{dom}(\mathbf{\Sigma}_{t})$, and $R_t$ are finite.

\textsc{proof}: Follows from the theorem assumptions.

\vspace{0.6mm}

\noindent 4) Each defender strategy $\pi_{\mathrm{D}}$ induces a well-defined probability measure over the random sequence $(\mathbf{\Sigma}_t,\mathbf{O}_t)_{t \geq 1}$.

\textsc{proof}: 3) implies that the sample spaces of $(\mathbf{\Sigma}_t,\mathbf{O}_t)$ and $(\mathbf{\Sigma}_t,\mathbf{O}_t)_{t \geq 1}$ are measurable and countable, respectively. Further, the fact that $\mathcal{G}$ is acyclic implies that the interventional distributions induced by $(\mathrm{do}(\widehat{\mathbf{X}}_t=\widehat{\mathbf{x}}_t))_{t \geq 1}$ are well-defined \cite[Ch. 3]{pearl2000causality}. Consequently, the statement follows from the Ionescu-Tulcea extension theorem \cite{Tulcea49}.

\vspace{1.3mm}

\noindent 5) $P(\mathbf{\Sigma}_{t+1} | \mathbf{\Sigma}_t), P(\mathbf{O}_t |\mathbf{\Sigma}_t)$, and $P(R_t | \mathbf{\Sigma}_t)$ are stationary.

\textsc{proof}: Follows by stationarity of (\hyperref[pos]{M1}).

\vspace{0.6mm}

\noindent Statements 1--5 imply that Prob. \ref{prob:sec_response} defines a finite, stationary, and partially observed Markov decision process (\pomdp) with bounded rewards. The theorem thus follows from standard results in Markov decision theory, see \cite[Thm. 7.6.1]{krishnamurthy_2016}.
\end{proof}
Theorem \ref{thm:well_defined_prob} states that an optimal defender strategy $\pi^{\star}_{\mathrm{D}}$ exists. Finding such a strategy requires estimating the causal effect $P(J \mid \mathrm{do}(\widehat{\mathbf{X}}_1=\pi_{\mathrm{D}}(\mathbf{H}_1)), \hdots, \mathrm{do}(\widehat{\mathbf{X}}_{\mathcal{T}}=\pi_{\mathrm{D}}(\mathbf{H}_{\mathcal{T}})))$ for different strategies $\pi_{\mathrm{D}}$ \cite[Def. 3.2.1]{pearl2000causality}. A key question is thus whether the effect is identifiable, i.e., whether it can be estimated from the observables. The following theorem states that the answer is negative.
\begin{theorem}\label{thm:iv_id}
Problem \ref{prob:sec_response} is not identifiable (Def. \ref{def:control_id}).
\end{theorem}
\begin{proof}
  To prove non-identifiability, it is sufficient to present two sets of causal functions $\mathbf{F}^{\prime},\mathbf{F}^{\prime\prime}$ that induce identical distributions over the observables but have different causal effects (Def. \ref{def:control_id}) \cite{pearl2000causality} \cite[Lem. 1]{shpitser_conditional_id}. For simplicity, consider $\mathcal{T}=2$ and $|\mathcal{V}|=1$. In this case $J=R$ (\ref{eq:defender_objective}). Let
\begin{align*}
  \mathbf{F}^{\prime} \triangleq \bigl\{&f_{\mathrm{R}}(S,I) \triangleq \mathbbm{1}_{I=\mathsf{R}},f_{\mathrm{I}}(\mathbf{A},E,\mathbf{D})=\mathsf{R},\\
  &f_{\mathrm{S}}(S,\mathbf{A})=1, f_{\mathrm{Z}},f_{\mathrm{C}}, f_{\mathrm{J}},f_{\mathrm{A}},f_{\mathrm{D}}\bigr\},
\end{align*}
where $\{f_{\mathrm{Z}},f_{\mathrm{C}}, f_{\mathrm{J}},f_{\mathrm{A}},f_{\mathrm{D}}\}$ are defined arbitrarily. Now let $\mathbf{F}^{\prime\prime}$ be equivalent to $\mathbf{F}^{\prime}$ except for $f_{\mathrm{R}}(S, I)\triangleq S$. Clearly $P(\mathbf{O})$ is the same with both $\mathbf{F}^{\prime}$ and $\mathbf{F}^{\prime\prime}$. However, $P(J=0\mid \mathrm{do}(I=\mathsf{S}))=1$ with $\mathbf{F}^{\prime}$ but $P(J=S \mid \mathrm{do}(I=\mathsf{S}))=1$ with $\mathbf{F}^{\prime\prime}$.
\end{proof}
Theorem \ref{thm:iv_id} states that causal effects of defender interventions (\ref{eq:defender_interventions}) can \textit{not} be identified from observations. This statement is obvious in hindsight but has important ramifications. It implies that to evaluate a defender strategy $\pi_{\mathrm{D}}$, the defender must either know (\hyperref[pos]{M1}) or perform controlled experiments to measure the effects of the interventions prescribed by $\pi_{\mathrm{D}}$.
%For example, the target system may have \textit{zero-day vulnerabilities} \cite{zeroday}, in which case $f_{\mathrm{I}}$ (\ref{eq:intrusion_state_fun}) is unknown.

While it is likely that the defender is aware of certain components of (\hyperref[pos]{M1}), it is unrealistic that it knows the entire model. A more reasonable assumption is that the defender knows the causal graph $\mathcal{G}$ (Fig. \ref{fig:causal_diagram}), which does not capture all nuances of the causal mechanisms but provides structural information. Leveraging this structure, we next present a method for finding an optimal strategy $\pi^{\star}_{\mathrm{D}}$ which only requires access to the causal graph and a simulator of (\hyperref[pos]{M1}).
\begin{remark}
\normalfont Access to a simulator is assumed by virtually all existing methods for autonomous cyber defense \cite{vyas2023automated,wolk2022cage,alan_turing_1, foley_cage_1, foley2023inroads,sussex_1, TANG2024103871,kiely2023autonomous,Richer2023,cage_cognitive,doi:10.1177/1071181322661504,tabular_Q_andy,wiebe2023learning,10476122,yan2024depending,cheng2024rice}.
\end{remark}
%In the case of \textsc{cage-2}, the simulator is provided as part of the evaluation scenario, but more generally it can be learned from data.
%The simulator is required to perform controlled experiments and measure the effects of interventions (Thm. \ref{thm:iv_id}).
%is a weaker assumption compared to requiring a complete model and is
%The simulator can be learned from data, or, as is the the case of \textsc{cage-2}, the simulator may be part of the evaluation environment.
%for benchmarking purposes. In other cases the simulator can be learned from data.
\section{Causal Partially Observable \\Monte-Carlo Planning (c-pomcp)}\label{sec:c_pomcp}
In this section, we present \textsc{c-pomcp}, an online method for obtaining an optimal defender strategy $\pi^{\star}_{\mathrm{D}}$ for Prob. \ref{prob:sec_response}. The method involves three consecutive actions that are performed at each time step $t$ (see Fig. \ref{fig:approach}).
%It takes as input the causal graph $\mathcal{G}$ (Fig. \ref{fig:causal_diagram}) and a simulator $\mathscr{S}$ of (\hyperref[pos]{M1}).

The first action uses the observation $\mathbf{o}_t$ and a \textit{particle filter} to compute the defender's \textit{belief} $\widehat{\mathbf{b}}_t$ in the form of a probability distribution over the latent variables $\mathbf{L}$ in (\hyperref[pos]{M1}). The second action constructs a \textit{search tree} of possible future histories $\mathbf{h}_k$ (\ref{eq:history_def}), which is initialized with a root node that represents the current history $\mathbf{h}_t$. Each edge extends this history by either an observation or an intervention: if $\mathbf{h}_{k+1}$ is a child node of $\mathbf{h}_k$, then either $\mathbf{h}_{k+1} = (\mathbf{h}_{k}, \mathbf{o}_{k+1})$ or $\mathbf{h}_{k+1} = (\mathbf{h}_{k}, \mathrm{do}(\widehat{\mathbf{X}}_{k}=\widehat{\mathbf{x}}_{k}))$. \textsc{c-pomcp} then \textit{prunes} the tree by \textit{excluding} histories that contain interventions that do not belong to a \pomis (Def. \ref{def:pomis}). The third action uses the belief $\widehat{\mathbf{b}}_t$ and the pruned tree to perform \textit{Monte-Carlo tree search}, which involves estimating the reward $J$ (\ref{eq:defender_objective}) through simulations. Once the search has been completed, the intervention from the root node that leads to the highest value of $J$ (\ref{eq:defender_objective}) is returned. The pseudocode of \textsc{c-pomcp} is listed in Alg. \ref{alg:c_pomcp}, and the main components of the method are described below.
%The method involves three actions that are performed at each time step: particle filtering, causal pruning, and tree search (see Fig. \ref{fig:approach}).
%\vspace{2mm}
\subsection{Particle Filtering to Estimate Latent Variables}
The particle filter is a method for state estimation in partially observed dynamical systems \cite{thrun2005probabilistic}. Since (\hyperref[pos]{M1}) can be formulated as such a system (Thm. \ref{thm:well_defined_prob}), we use the particle filter to estimate the values of the latent variables $\mathbf{L}$ in (\hyperref[pos]{M1}), e.g., the intrusion state $I_{i,t}$ (\ref{eq:intrusion_state_fun}).

We define the defender's \textit{belief state} as
\begin{align}
&\mathbf{b}_t(\bm{\sigma}_t)\triangleq P(\mathbf{\Sigma}_t=\bm{\sigma}_t \mid \mathbf{h}_{t}) \label{eq:belief_def}\\
&\numeq{\text{Bayes}} \eta P(\mathbf{o}_t \mid \bm{\sigma}_t, \mathbf{h}_{t-1})P(\bm{\sigma}_t \mid \mathrm{do}(\widehat{\mathbf{X}}_t=\widehat{\mathbf{x}}_t), \mathbf{h}_{t-1}) \nonumber\\
&\numeq{\text{Markov}} \eta P(\mathbf{o}_t \mid \bm{\sigma}_t)P(\bm{\sigma}_t \mid \mathrm{do}(\widehat{\mathbf{X}}_t=\widehat{\mathbf{x}}_t), \mathbf{h}_{t-1}) \nonumber\\
&\numeq{\text{Markov}} \eta P(\mathbf{o}_t \mid \bm{\sigma}_t)\sum_{\bm{\sigma}_{t-1}}P(\bm{\sigma}_t \mid \bm{\sigma}_{t-1}, \mathrm{do}(\widehat{\mathbf{X}}_t=\widehat{\mathbf{x}}_t))\mathbf{b}_{t-1}(\bm{\sigma}_{t-1}), \nonumber
\end{align}
where $\mathbf{h}_{t}$ is the history (\ref{eq:history_def}), $\bm{\sigma}_t$ is a realization of $\mathbf{\Sigma}_t$ (see Thm. \ref{thm:well_defined_prob}), and $\eta$ is a normalizing constant. The sum in (\ref{eq:belief_def}) is over all possible realizations of $\mathbf{\Sigma}_{t-1}$.

The computational complexity of (\ref{eq:belief_def}) is $\mathcal{O}(|\mathrm{dom}(\mathbf{\Sigma})|^2)$, which grows quadratically with the size of the state space and exponentially with the number of state variables. For this reason, the particle filter approximates (\ref{eq:belief_def}) by representing $\mathbf{b}_t$ by a set of $M$ sample states (particles) $\mathcal{P}_t=\{\widehat{\bm{\sigma}}_t^{(1)},\hdots,\widehat{\bm{\sigma}}_t^{(M)}\}$ \cite{particle_filter_survey}. These particles are sampled recursively as
\begin{subequations}\label{eq:particle_filter}
\begin{align}
\overline{\mathcal{P}}_t &\triangleq \bigcup_{i=1}^M  \left\{\widehat{\bm{\sigma}}_t^{(i)} \sim P\left(\cdot \mid \widehat{\bm{\sigma}}_{t-1}^{(i)}, \mathrm{do}(\widehat{\mathbf{X}}_{t-1}=\widehat{\mathbf{x}}_{t-1})\right)\right\} \label{eq:prediction}\\
\mathcal{P}_t &\triangleq \bigcup_{i=1}^M \left\{\widehat{\bm{\sigma}}_t^{(i)} \stackrel{\propto P(\mathbf{o}_t | \widehat{\bm{\sigma}}_t^{(i)})}{\sim}\overline{\mathcal{P}}_t\right\},\label{eq:resampling}
\end{align}
\end{subequations}
where $\widehat{\bm{\sigma}}_{t-1}^{(i)} \in \mathcal{P}_{t-1}$ and $x \stackrel{\propto \varphi}{\sim}$ means that $x$ is sampled with probability proportional to $\varphi$.

(\ref{eq:prediction})--(\ref{eq:resampling}) focus the particle set to regions of the state space with a high probability of generating the latest observation $\mathbf{o}_t$ \cite{thrun2005probabilistic}. This ensures that the belief state induced by the particles converges to (\ref{eq:belief_def}) when $M \rightarrow \infty$, as stated below.
\begin{theorem}\label{thm:particles}
Let $\widehat{\mathbf{b}}(\bm{\sigma}_t) = \frac{1}{M}\sum_{i=1}^M\mathbbm{1}_{\bm{\sigma}_t=\widehat{\bm{\sigma}}_t^{(i)}}$, then
\begin{align*}
\lim_{M \rightarrow \infty}\widehat{\mathbf{b}}_t \rightarrow \mathbf{b}_t \text{ almost surely} && \forall t.
\end{align*}
\end{theorem}
This is a standard result in particle filtering. The proof is given in Appendix \ref{appendix:proof_proposition_1}.

%We first sample a set of \textit{particles} $\widehat{\bm{\sigma}}_t^{(1)}, \hdots, \widehat{\bm{\sigma}}_t^{(M)}$ from the simulator $\mathscr{S}$ with probability proportional to $P(\mathbf{o}_t | \widehat{\bm{\sigma}}_t^{(i)})$, where $\bm{\sigma}$ is a realization of $\Sigma$ (Thm. \ref{thm:well_defined_prob}) and $\mathbf{o}_t$ is the latest observation. We then use the particles to estimate the latent variables with the \textit{belief state}
%\begin{align}
%\widehat{\mathbf{b}}(\bm{\sigma}_t) \triangleq \frac{1}{M}\sum_{i=1}^M\delta_{\bm{\sigma}_t,\widehat{\bm{\sigma}}_t^{(i)}},  \label{eq:belief_def}
%\end{align}
%where $\lim_{M \rightarrow \infty}\widehat{\mathbf{b}}_t(\bm{\sigma}_t) \rightarrow P(\mathbf{\Sigma}_t=\bm{\sigma}_t | \mathbf{h}_{t})$ (see Appendix \ref{appendix:proof_proposition_1}).

\subsection{Causal Pruning of the Search Tree}\label{sec:pruning}
We use the causal graph $\mathcal{G}$ (Fig. \ref{fig:causal_diagram}) to prune the search tree by excluding histories $\mathbf{h}_k$ (\ref{eq:history_def}) that contain interventions that do not belong to a \pomis $\mathbf{P}_{\mathcal{G}}^{\star}$ (Def. \ref{def:pomis}). For example, when $t=\mathcal{T}-1$, then $Z_{i,t}$ (\ref{eq:obs_function}) and $J$ (\ref{eq:defender_objective}) are \textit{d-separated} in $\mathcal{G}$ \cite[Def. 1.2.3]{pearl2000causality}. This means that the intervention $\mathrm{do}(Z_{i,t}=I_{i,t})$ (\ref{analyze_intervention}) has no causal effect on $J$ and thus $Z_{i,t}\not\in \mathbf{P}_{\mathcal{G}}^{\star}$.

By restricting the possible interventions at time $t$ to the set of \pomis{}s $\mathbf{P}_{\mathcal{G}}^{\star}$, the number of interventions in the search tree is reduced by a factor of
\begin{align}
\prod_{t=1}^{\mathcal{T}}\frac{\sum_{\tilde{\mathbf{X}} \in \mathbf{P}_{\mathcal{G}}^{\star}}|\mathrm{dom}(\tilde{\mathbf{X}})|}{\sum_{\widehat{\mathbf{X}} \in 2^{\mathbf{X}_t}}|\mathrm{dom}(\widehat{\mathbf{X}})|},\label{pruning_effect}
\end{align}
where $\mathbf{X}_t$ is the set of manipulative variables at time $t$ (\hyperref[pos]{M1}).

Hence, even if only a small subset of interventions does not belong to an \pomis, a significant reduction in the search tree size can be expected (see Fig. \ref{fig:tree_prune}). Unfortunately, computing $\mathbf{P}_{\mathcal{G}}^{\star}$ is generally intractable, as stated below.
\begin{proposition}\label{prop:pomis_complexity}
Computing $\mathbf{P}_{\mathcal{G}}^{\star}$ (Def. \ref{def:pomis}) is \textsc{pspace}-hard.
\end{proposition}
\begin{proof}
%=(\Pi, \mathcal{G})
We prove the \textsc{pspace}-hardness by reduction to the problem of solving a \pomdp, which is \textsc{pspace}-hard \cite[Thm. 6]{pspace_complexity}. Let $x$ be an instance of the problem of computing $\mathbf{P}_{\mathcal{G}}^{\star}$ (Def. \ref{def:pomis}). Finding a solution to $x$ involves checking (\ref{eq:pomis_condition}) for each $\tilde{\mathbf{X}}_t \in \mathbf{P}_{\mathcal{G}}^{\star}$. This means that a solution to $x$ allows constructing an optimal solution to Prob. \ref{prob:sec_response}. By Thm \ref{thm:well_defined_prob}, such a solution also provides a solution to a \pomdp.
\end{proof}
Given the impracticality of computing $\mathbf{P}_{\mathcal{G}}^{\star}$ (Prop. \ref{prop:pomis_complexity}), we approximate $\mathbf{P}_{\mathcal{G}}^{\star}$ as follows. First, we reduce the causal graph to a \textit{subgraph} $\mathcal{G}[\mathbf{U}_{t-1} \cup \mathbf{U}_t \cup \mathbf{V}_{t-1} \cup \mathbf{V}_t]$ \cite[Def. 7.1.2]{pearl2000causality}. We then remove all variables in the subgraph whose values are uniquely determined by $\widehat{\mathbf{b}}$ (\ref{eq:particle_filter}). Subsequently, we add a node to the subgraph that represents the target $J$ (\ref{eq:defender_objective}), whose causal parents \cite[Def. 1.2.1]{pearl2000causality} are determined using a \textit{base strategy} $\widehat{\pi}$, which can be chosen freely. It can, for example, be based on heuristics or be designed by a domain expert. Finally, we compute a \pomis for the reduced graph using \cite[Alg. 1]{lee2019structural}.
\begin{remark}
\normalfont Since \cite[Alg. 1]{lee2019structural} is sound and complete \cite[Thm. 9]{lee2019structural}, the approximation described above is exact when the base strategy $\widehat{\pi}$ is optimal.
\end{remark}
%It can also be based on heuristics or be designed by a domain expert. Irrespective of how $\widehat{\pi}$ is chosen, the subgraph truncates the causal graph $\mathcal{G}$ to a single time step $t$, which means that we can efficiently compute the set of \pomis{}s using the algorithm presented in \cite[Alg. 1]{lee2019structural}.

When applying the above procedure to the \textsc{cage-2} scenario, we identify the following types of defender interventions that are never included in a \pomis: (\textit{i}) interventions that start decoys that are already running; (\textit{ii}) defensive interventions on nodes that are not compromised according to $\widehat{\mathbf{b}}$ (\ref{eq:particle_filter}); and (\textit{iii}) forensic and deceptive interventions on nodes that are compromised according to $\widehat{\mathbf{b}}$ (\ref{eq:particle_filter}).

\begin{remark}
\normalfont The pruning of the search tree based on the \pomis{}s (Def. \ref{def:pomis}) occurs \textit{during} the construction of the tree. The complete search tree is generally too large to construct.
\end{remark}

\begin{figure}
  \centering
  \scalebox{0.88}{
    \begin{tikzpicture}[
    % define a style for the dots
    dot/.style={
        draw=black,
        fill=blue!90,
        circle,
        minimum size=3pt,
        inner sep=0pt,
        solid,
    },
    ]

\node[scale=1] (kth_cr) at (0,0)
{
  \begin{tikzpicture}[declare function={sigma(\x)=1/(1+exp(-\x));
sigmap(\x)=sigma(\x)*(1-sigma(\x));}]
\begin{axis}[
        xmin=1,
        xmax=100,
        ymin=0,
        ymax=105,
        width =1.175\columnwidth,
        height = 0.4\columnwidth,
        axis lines=center,
        xmajorgrids=true,
        ymajorgrids=true,
        major grid style = {lightgray},
        minor grid style = {lightgray!25},
        scaled y ticks=false,
        yticklabel style={
        /pgf/number format/fixed,
        /pgf/number format/precision=5
        },
        xlabel style={below right},
        ylabel style={above left},
        axis line style={-{Latex[length=2mm]}},
        smooth,
        legend style={at={(0.92,-0.23)}},
        legend columns=5,
        legend style={
          draw=none,
                    % the /tikz/ prefix is necessary here...
                    % otherwise, it might end-up with `/pgfplots/column 2`
                    % which is not what we want. compare pgfmanual.pdf
            /tikz/column 2/.style={
                column sep=0pt,
              }
              }
              ]
%pa=0.1, pu=0, p_c_1=0.00001
              \addplot[Blue,mark=diamond,smooth, name path=l1, thick, domain=1:100]   (x,{(1-(0.99)^x)*100});

              % pa=0.05, pu=0, p_c_1=0.00001
              \addplot[OliveGreen,mark=pentagon,mark repeat=1, smooth, name path=l1, thick, domain=1:100]   (x,{(1-(0.975)^x)*100});
              \addplot[Red,mark=triangle,mark repeat=1, smooth, name path=l1, thick, domain=1:100]   (x,{(1-(0.95)^x)*100});

              \addplot[Purple,mark=o,mark repeat=1, smooth, name path=l1, thick, domain=1:100]   (x,{(1-(0.925)^x)*100});
              \addplot[Black,mark=square,mark repeat=1, smooth, name path=l1, thick, domain=1:100]   (x,{(1-(0.9)^x)*100});

%pa=0.01, pu=0, p_c_1=0.00001

%pa=0.005, pu=0, p_c_1=0.00001
%              \addplot[Purple,mark=o, mark repeat=1, smooth, name path=l1, thick, domain=1:100]   (x,{1-(1-0.005001)^(x)});
\legend{$0.99$, $0.975$, $0.95$, $0.925$, $0.9$}
\end{axis}
\end{tikzpicture}
};

\node[inner sep=0pt,align=center, scale=1, rotate=0, opacity=1] (obs) at (-1.7,1.57)
{
\% reduction of the search tree size
};
\node[inner sep=0pt,align=center, scale=1, rotate=0, opacity=1] (obs) at (4.89,-0.51)
{
 $\mathcal{T}$
};

  \end{tikzpicture}
  }
  \caption{Reduction of the size of the search tree by pruning the intervention space $|2^{\mathbf{X}_t}|$ to the set of \pomis{}s $\mathbf{P}_{\mathcal{G}}^{\star}$ (Def. \ref{def:pomis}); the x-axis indicates the tree depth $\mathcal{T}$; curves relate to the factor in (\ref{pruning_effect}).}
  \label{fig:tree_prune}
\end{figure}
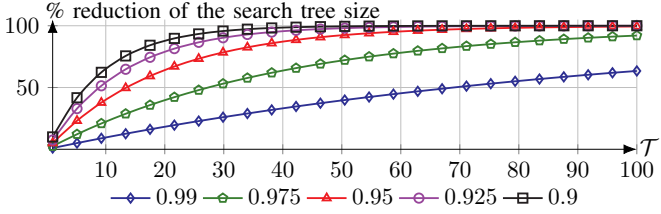

\subsection{Monte-Carlo Tree Search}\label{pomcp_planning}
Given the particle filter (\ref{eq:belief_def}) and the \textsc{pomis} (Def. \ref{def:pomis}), \textsc{c-pomcp} searches for optimal interventions using the tree search algorithm described in \cite[Alg. 1]{pomcp}. This algorithm constructs a search tree iteratively by repeating five steps (see Fig. \ref{fig:mcts}): (\textit{i}) it selects a path from the root to a leaf node using the \textit{tree policy} described below; (\textit{ii}) it expands the tree by adding children to the leaf, each of which corresponds to an intervention (\ref{eq:defender_interventions}) on a \pomis (Def. \ref{def:pomis}); (\textit{iii}) it executes a rollout simulation from the leaf; (\textit{iv}) it adds a child to the leaf that corresponds to the first observation (\ref{eq:feedback}) in the simulation; and (\textit{v}) it records the value of $J$ (\ref{eq:defender_objective}) and backpropagates the value up the tree.

\vspace{2mm}

\noindent\textbf{Tree policy.} A node at depth $k$ of the tree is associated with a history $\mathbf{h}_k$ (\ref{eq:history_def}) and stores two variables: the average cumulative reward $\widehat{J}(\mathbf{h}_k)$ (\ref{eq:defender_objective}) of simulations on the subtree emanating from the node, and the visit count $N(\mathbf{h}_k) \geq 1$, which is incremented whenever the node is visited during the search. Using these variables, we implement the tree policy by selecting nodes that maximize the upper confidence bound
\begin{align}
\widehat{J}(\mathbf{h}_k) + c\sqrt{\frac{\ln N(\mathbf{h}_{k-1})}{N(\mathbf{h}_k)}},\label{eq:ucb} %
\end{align}
where $c > 0$ controls the exploration-exploitation trade-off.

\vspace{2mm}

\noindent\textbf{Rollout.} The initial state of a rollout simulation is sampled from the belief state $\widehat{\mathbf{b}}_k$ (\ref{eq:particle_filter}), and the intervention at each time step is selected using the \textit{base strategy} $\widehat{\pi}$ (\S \ref{sec:pruning}). The simulation executes for a depth of $\delta_{\mathrm{R}}$, after which the reward for the remainder of the simulation is estimated using a \textit{base value function} $J_{\widehat{\pi}}$. Like the base strategy, this function can be chosen freely. It can, for example, be obtained through offline reinforcement learning. After the simulation has completed, the discounted sum of the rewards $R_1,R_t,\hdots R_{\delta_{\mathrm{R}}}$ (\ref{eq:defender_objective}) and $J_{\widehat{\pi}}$ is used to update $\widehat{J}(\mathbf{h}_k)$.
%of the node from which the rollout started.

%Theorem 5.A implies that it is sub-optimal to do a defensive intervention on a node that is not compromised. Theorem 5.B states that it is sub-optimal to do a forensic or  intervention on a compromised node. Theorem 5.C states that it is unnecessary to start a decoy that is already running. Lastly, Thm. 5.D states that it is sub-optimal for the defender to intervene on a node i that is never targeted by the attacker. b t allow to These four rules together with the belief state b prune the search tree by only considering interventions that may belong to a POMIS (Def. 3).

% , which can be chosen freely. It can for example be obtained through offline reinforcement learning or dynamic programming. It can also be based on heuristics or be designed by a domain expert.

\vspace{2mm}

\noindent\textbf{Convergence.}
The process of running simulations and extending the search tree continues for a \textit{search time} $s_{\mathrm{T}}$, after which the intervention that leads to the maximal value of $J$ (\ref{eq:defender_objective}) is returned, i.e.,
\begin{align*}
\mathrm{do}(\tilde{\mathbf{X}}_t=\tilde{\mathbf{x}}_t) \in \argmax_{\mathrm{do}(\widehat{\mathbf{X}}_t=\widehat{\mathbf{x}}_t)}\widehat{J}((\mathbf{h}_t, \mathrm{do}(\widehat{\mathbf{X}}_t=\widehat{\mathbf{x}}_t))).
\end{align*}
We can express this search procedure as
\begin{align}
\mathrm{do}(\tilde{\mathbf{X}}_t=\tilde{\mathbf{x}}_t) \leftarrow \mathscr{T}(\mathbf{h}_t, \widehat{\mathbf{b}}_t, \widehat{\pi}, s_{\mathrm{T}}, \mathscr{S}, \mathcal{G}, \mathbf{P}_{\mathcal{G}}^{\star}), \label{eq:planning_operator}
\end{align}
where $\mathscr{T}$ is a tree search operator.
\begin{theorem}\label{thm:pomcp_convergence}
Under the assumptions made in Thm. \ref{thm:well_defined_prob} and further assuming that the \pomis computation is exact (\S \ref{sec:pruning}), $M \rightarrow \infty$, $s_{\mathrm{T}} \rightarrow \infty$, $\mathcal{T} < \infty$, and $c$ is chosen such that
\begin{align}
P\left(\widehat{J}(\mathbf{h}_k) \leq \mathbb{E}[\widehat{J}(\mathbf{h}_k)] \pm c\sqrt{\frac{\ln N(\mathbf{h}_{k-1})}{N(\mathbf{h}_k)}}\right) \leq k^{-4} \label{eq:thm_3_cond}
\end{align}
for all $k \geq 1$. Then the intervention prescribed by \textsc{c-pomcp} (\ref{eq:planning_operator}) for any $\mathbf{h}_t$ converges in probability to an optimal intervention $\mathrm{do}(\mathbf{X}^{\star}_t=\mathbf{x}^{\star}_t)$.
\end{theorem}
The proof of Thm. \ref{thm:pomcp_convergence} relies on mapping an execution of \textsc{c-pomcp} to an execution of the \textsc{uct} algorithm \cite[Fig. 1]{uct}, which is known to converge as $s_{\mathrm{T}} \rightarrow \infty$ \cite[Thm. 7]{uct}. We provide the proof in Appendix \ref{appendix:proof_theorem_3}. Note that (\ref{eq:thm_3_cond}) can always be satisfied by choosing a large $c$.
\begin{remark}
\normalfont Theorem \ref{thm:pomcp_convergence} is not confined to \textsc{cage-2} (\hyperref[pos]{M1}). Rather, the theorem is general and applies to any control problem based on an \scm with interventions that can be formulated as a finite and stationary \pomdp (Thm. \ref{thm:well_defined_prob}).
\end{remark}
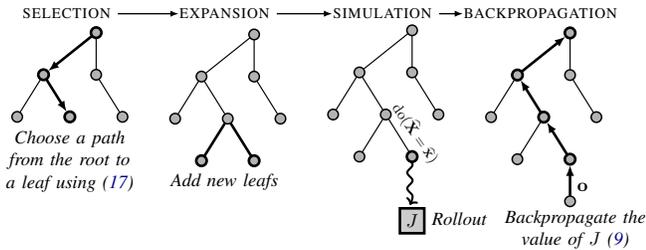
\begin{figure}
  \centering
  \scalebox{1.4}{
    \begin{tikzpicture}

\node[scale=1] (kth_cr) at (0,0.17)
{
\begin{tikzpicture}
\node[draw,circle, minimum width=0.1cm, scale=0.3, fill=black!30, line width=0.25mm](p1) at (0,0) {};
\node[draw,circle, minimum width=0.1cm, scale=0.3, fill=black!30, line width=0.25mm](p2) at (-0.5,-0.4) {};
\node[draw,circle, minimum width=0.1cm, scale=0.3, fill=black!30](p3) at (0,-0.4) {};
%\node[draw,circle, minimum width=0.1cm, scale=0.3, fill=black!30](p4) at (0.5,-0.5) {};
\node[draw,circle, minimum width=0.1cm, scale=0.3, fill=black!30](p5) at (-0.75,-0.8) {};
\node[draw,circle, minimum width=0.1cm, scale=0.3, fill=black!30, line width=0.25mm](p6) at (-0.25,-0.8) {};
\node[draw,circle, minimum width=0.1cm, scale=0.3, fill=black!30](p7) at (0.25,-0.8) {};
%\node[draw,circle, minimum width=0.1cm, scale=0.3, fill=black!30](p8) at (0.75,-1) {};

\draw[-{Latex[length=1.05mm]},line width=0.25mm] (p1) to (p2);
\draw[-,line width=0.11mm] (p1) to (p3);
%\draw[-,line width=0.11mm] (p1) to (p4);
\draw[-,line width=0.11mm] (p2) to (p5);
\draw[-{Latex[length=1.05mm]},line width=0.25mm] (p2) to (p6);
\draw[-,line width=0.11mm] (p3) to (p7);
%\draw[-,line width=0.11mm] (p4) to (p8);
\end{tikzpicture}
};

\node[scale=1] (kth_cr) at (1.5,-0.05)
{
\begin{tikzpicture}
\node[draw,circle, minimum width=0.1cm, scale=0.3, fill=black!30](p1) at (0,0) {};
\node[draw,circle, minimum width=0.1cm, scale=0.3, fill=black!30](p2) at (-0.5,-0.4) {};
\node[draw,circle, minimum width=0.1cm, scale=0.3, fill=black!30](p3) at (0,-0.4) {};
%\node[draw,circle, minimum width=0.1cm, scale=0.3, fill=black!30](p4) at (0.5,-0.5) {};
\node[draw,circle, minimum width=0.1cm, scale=0.3, fill=black!30](p5) at (-0.75,-0.8) {};
\node[draw,circle, minimum width=0.1cm, scale=0.3, fill=black!30](p6) at (-0.25,-0.8) {};
\node[draw,circle, minimum width=0.1cm, scale=0.3, fill=black!30](p7) at (0.25,-0.8) {};
%\node[draw,circle, minimum width=0.1cm, scale=0.3, fill=black!30](p8) at (0.75,-1) {};

\node[draw,circle, minimum width=0.1cm, scale=0.3, fill=black!30, line width=0.25mm](p9) at (-0.5,-1.2) {};
\node[draw,circle, minimum width=0.1cm, scale=0.3, fill=black!30, line width=0.25mm](p10) at (0,-1.2) {};

\draw[-,line width=0.11mm] (p1) to (p2);
\draw[-,line width=0.11mm] (p1) to (p3);
%\draw[-,line width=0.11mm] (p1) to (p4);
\draw[-,line width=0.11mm] (p2) to (p5);
\draw[-,line width=0.11mm] (p2) to (p6);
\draw[-,line width=0.11mm] (p3) to (p7);
%\draw[-,line width=0.11mm] (p4) to (p8);

\draw[-,line width=0.11mm, line width=0.25mm] (p6) to (p9);
\draw[-,line width=0.11mm, line width=0.25mm] (p6) to (p10);
\end{tikzpicture}
};

\node[scale=1] (kth_cr) at (3,-0.35)
{
\begin{tikzpicture}
\node[draw,circle, minimum width=0.1cm, scale=0.3, fill=black!30](p1) at (0,0) {};
\node[draw,circle, minimum width=0.1cm, scale=0.3, fill=black!30](p2) at (-0.5,-0.4) {};
\node[draw,circle, minimum width=0.1cm, scale=0.3, fill=black!30](p3) at (0,-0.4) {};
%\node[draw,circle, minimum width=0.1cm, scale=0.3, fill=black!30](p4) at (0.5,-0.5) {};
\node[draw,circle, minimum width=0.1cm, scale=0.3, fill=black!30](p5) at (-0.75,-0.8) {};
\node[draw,circle, minimum width=0.1cm, scale=0.3, fill=black!30](p6) at (-0.25,-0.8) {};
\node[draw,circle, minimum width=0.1cm, scale=0.3, fill=black!30](p7) at (0.25,-0.8) {};
%\node[draw,circle, minimum width=0.1cm, scale=0.3, fill=black!30](p8) at (0.75,-1) {};

\node[draw,circle, minimum width=0.1cm, scale=0.3, fill=black!30](p9) at (-0.5,-1.2) {};
\node[draw,circle, minimum width=0.1cm, scale=0.3, fill=black!30, line width=0.25mm](p10) at (0,-1.2) {};

\draw[-,line width=0.11mm] (p1) to (p2);
\draw[-,line width=0.11mm] (p1) to (p3);
%\draw[-,line width=0.11mm] (p1) to (p4);
\draw[-,line width=0.11mm] (p2) to (p5);
\draw[-,line width=0.11mm] (p2) to (p6);
\draw[-,line width=0.11mm] (p3) to (p7);
%\draw[-,line width=0.11mm] (p4) to (p8);

\draw[-,line width=0.11mm] (p6) to (p9);
\draw[-,line width=0.11mm] (p6) to (p10);

  \node[rotate=-90, scale=0.25](test) at (0,-1.45) {
    \begin{tikzpicture}
\draw[->, x=0.15cm,y=0.1cm, line width=0.8mm, black]
        (3,0) sin (4,1) cos (5,0) sin (6,-1) cos (7,0)
        sin (8,1) cos (9,0) sin (10,-1) cos (11,0)
        sin (12,1) cos (14.2,0);
    \end{tikzpicture}
  };
\node at (0,-1.81) [draw, fill=black!20, scale=0.5, line width=0.25mm] (v100) {$J$};
\end{tikzpicture}
};

\node[scale=1] (kth_cr) at (4.5,-0.23)
{
\begin{tikzpicture}
\node[draw,circle, minimum width=0.1cm, scale=0.3, fill=black!30, line width=0.25mm](p1) at (0,0) {};
\node[draw,circle, minimum width=0.1cm, scale=0.3, fill=black!30, line width=0.25mm](p2) at (-0.5,-0.4) {};
\node[draw,circle, minimum width=0.1cm, scale=0.3, fill=black!30](p3) at (0,-0.4) {};
%\node[draw,circle, minimum width=0.1cm, scale=0.3, fill=black!30](p4) at (0.5,-0.5) {};
\node[draw,circle, minimum width=0.1cm, scale=0.3, fill=black!30](p5) at (-0.75,-0.8) {};
\node[draw,circle, minimum width=0.1cm, scale=0.3, fill=black!30, line width=0.25mm](p6) at (-0.25,-0.8) {};
\node[draw,circle, minimum width=0.1cm, scale=0.3, fill=black!30](p7) at (0.25,-0.8) {};
%\node[draw,circle, minimum width=0.1cm, scale=0.3, fill=black!30](p8) at (0.75,-1) {};

\node[draw,circle, minimum width=0.1cm, scale=0.3, fill=black!30](p9) at (-0.5,-1.2) {};
\node[draw,circle, minimum width=0.1cm, scale=0.3, fill=black!30, line width=0.25mm](p10) at (0,-1.2) {};
\node[draw,circle, minimum width=0.1cm, scale=0.3, fill=black!30](p11) at (0,-1.6) {};

\draw[{Latex[length=1.05mm]}-,line width=0.11mm, line width=0.25mm] (p1) to (p2);
\draw[-,line width=0.11mm] (p1) to (p3);
%\draw[-,line width=0.11mm] (p1) to (p4);
\draw[-,line width=0.11mm] (p2) to (p5);
\draw[{Latex[length=1.05mm]}-,line width=0.11mm, line width=0.25mm] (p2) to (p6);
\draw[-,line width=0.11mm] (p3) to (p7);
%\draw[-,line width=0.11mm] (p4) to (p8);

\draw[-,line width=0.11mm] (p6) to (p9);
\draw[{Latex[length=1.05mm]}-,line width=0.11mm, line width=0.25mm] (p6) to (p10);
\draw[{Latex[length=1.05mm]}-,line width=0.11mm, line width=0.25mm] (p10) to (p11);
\end{tikzpicture}
};

\node[inner sep=0pt,align=center, scale=0.5, color=black] (selection) at (0,0.75) {
\textsc{selection}
};

\node[inner sep=0pt,align=center, scale=0.5, color=black] (expansion) at (1.5,0.75) {
\textsc{expansion}
};

\node[inner sep=0pt,align=center, scale=0.5, color=black] (simulation) at (3,0.75) {
\textsc{simulation}
};

\node[inner sep=0pt,align=center, scale=0.5, color=black] (backprop) at (4.5,0.75) {
\textsc{backpropagation}
};

\draw[-{Latex[length=0.8mm]}, line width=0.13mm] (selection) to (expansion);
\draw[-{Latex[length=0.8mm]}, line width=0.13mm] (expansion) to (simulation);
\draw[-{Latex[length=0.8mm]}, line width=0.13mm] (simulation) to (backprop);
%\draw[-{Latex[length=0.8mm]}, line width=0.13mm] (backprop) to (5.4, 0.75) to (5.4, 1.05) to (-0.6, 1.05) to (-0.6, 0.75) to (selection);

\node[inner sep=0pt,align=center, scale=0.5, color=black] (backprop) at (0,-0.65) {
  \textit{Choose a path} \\\textit{from the root to}\\\textit{a leaf using (\ref{eq:ucb})}
};

\node[inner sep=0pt,align=center, scale=0.5, color=black] (backprop) at (1.5,-0.85) {
  \textit{Add new leafs}
};

\node[inner sep=0pt,align=center, scale=0.4, color=black, rotate=-55] (backprop) at (3.3,-0.4) {
  $\mathrm{do}(\widehat{\mathbf{X}}=\widehat{\mathbf{x}})$
};

\node[inner sep=0pt,align=center, scale=0.5, color=black] (backprop) at (3.73,-1.2) {
  \textit{Rollout}
};

\node[inner sep=0pt,align=center, scale=0.5, color=black] (backprop) at (4.8,-1.3) {
  \textit{Backpropagate the}\\
  \textit{value of $J$ (\ref{eq:defender_objective})}
};

\node[inner sep=0pt,align=center, scale=0.5, color=black] (backprop) at (4.9,-0.9) {
  $\mathbf{o}$
};

\end{tikzpicture}
  }
  \caption{Tree search in \textsc{c-pomcp}; a search tree is constructed iteratively where each iteration consists of the four phases above.}
  \label{fig:mcts}
\end{figure}

\begin{algorithm}
\footnotesize
  \SetNoFillComment
  \SetKwProg{myInput}{Input:}{}{}
  \SetKwProg{myOutput}{Output:}{}{}
  \SetKwProg{myalg}{Algorithm}{}{}
  \SetKwProg{myproc}{Procedure}{}{}
  \SetKw{KwTo}{inp}
  \SetKwFor{Forp}{in parallel for}{\string do}{}%
  \SetKwFor{Loop}{Loop}{}{EndLoop}
  \DontPrintSemicolon
  \SetKwBlock{DoParallel}{do in parallel}{end}
  \myInput{
    \upshape Simulator $\mathscr{S}$ of (\hyperref[pos]{M1}), causal graph $\mathcal{G}$ (Fig. \ref{fig:causal_diagram}),\\
    $\quad\quad\quad$ search time $s_{\mathrm{T}}$, horizon $\mathcal{T}$, number of particles $M$.
  }{}
  \myOutput{
    \upshape Interventions $\mathrm{do}(\tilde{\mathbf{X}}_1=\tilde{\mathbf{x}}_1),\hdots,\mathrm{do}(\tilde{\mathbf{X}}_{\mathcal{T}}=\tilde{\mathbf{x}}_{\mathcal{T}})$.
  }{}
  \caption{\textsc{c-pomcp}.}\label{alg:c_pomcp}
  \myalg{}{
    $\mathbf{h}_1=(\V_1,)$.\;
    \For{$t=1,2,\hdots,\mathcal{T}$}{
      Compute $\widehat{\mathbf{b}}_t$ using (\ref{eq:particle_filter}) with $M$ particles.\;
      Compute $\mathbf{P}_{\mathcal{G}}^{\star}$ (Def. \ref{def:pomis}).\;
      $\mathrm{do}(\tilde{\mathbf{X}}_t=\tilde{\mathbf{x}}_t) \leftarrow \mathscr{T}(\mathbf{h}_t, \widehat{\mathbf{b}}_t, \widehat{\pi}, s_{\mathrm{T}}, \mathscr{S}, \mathcal{G}, \mathbf{P}_{\mathcal{G}}^{\star})$ (\ref{eq:planning_operator}).\;
      Perform intervention $\mathrm{do}(\tilde{\mathbf{X}}_t=\tilde{\mathbf{x}}_t)$ (\ref{eq:defender_interventions}). \;
      Observe $\mathbf{o}_{t+1}$. \;
      Update history $\mathbf{h}_{t+1}=(\mathbf{h}_t, \mathrm{do}(\tilde{\mathbf{X}}_t=\tilde{\mathbf{x}}_t), \mathbf{o}_{t+1})$ (\ref{eq:history_def}).\;
    }
    }
  \normalsize
\end{algorithm}

\subsection{Comparison with Other Methods}
\textsc{c-pomcp} (Alg. \ref{alg:c_pomcp}) distinguishes itself from existing methods evaluated against the \textsc{cage-2} benchmark (\cite{vyas2023automated,wolk2022cage,alan_turing_1, foley_cage_1, foley2023inroads,sussex_1,TANG2024103871,kiely2023autonomous,Richer2023,tabular_Q_andy,wiebe2023learning,10476122,cheng2024rice,yan2024depending}) in four key aspects. First, it incorporates the causal structure of the target system. Second, it guarantees an optimal solution (Thm. \ref{thm:pomcp_convergence}). No such guarantees are available for the existing methods. Third, while the above-referenced methods ignore the latent variables, \textsc{c-pomcp} explicitly models the uncertainty of the latent variables and how this uncertainty changes in light of new observations (\ref{eq:belief_def}). Fourth, in contrast to the existing \textit{offline} methods, \textsc{c-pomcp} is an \textit{online} method that updates the defender strategy at each time step.
%, which means that it does not require a training phase.
%This allows \textsc{c-pomcp} to make informed decisions that account for the uncertainty.
%and can dynamically adapt to changes in the target system.
%This uncertainty quantification allows the defender to make more informed decisions that account for the uncertainty about the attacker. By

\section{Evaluating \textsc{c-pomcp} Against \textsc{cage-2}}\label{sec:evaluation}
We implement \textsc{c-pomcp} (Alg. \ref{alg:c_pomcp}) in Python and run it to learn defender strategies for the \textsc{cage-2} scenario \cite{cage_challenge_2_announcement}. The source code is available at \cite{csle_docs}; the system configuration is listed in Appendix \ref{appendix:infrastructure_configuration}; the hyperparameters are listed in Appendix \ref{appendix:hyperparameters}; and the computing environment is an \textsc{m2}-ultra processor.

\vspace{2mm}

\noindent\textbf{Baselines.} We compare the performance of \textsc{c-pomcp} (Alg. \ref{alg:c_pomcp}) with that of two baselines: \textsc{cardiff-ppo}, a current state-of-the-art method for \textsc{cage-2} \cite{vyas2023automated}, and \textsc{pomcp}  \cite[Alg. 1]{pomcp}, a non-causal version \textsc{c-pomcp}. Note that, while we only compare against \textsc{cardiff-ppo} from the \textsc{cage-2} leaderboard, it represents all methods on the leaderboard since it achieves better performance than the other methods \cite{cage_challenge_2_announcement}.

\vspace{2mm}

\noindent\textbf{Evaluation metrics.} We use two evaluation metrics: the cumulative reward $\mathbb{E}[J]$ (\ref{eq:defender_objective}) and the cumulative regret \cite{lattimore2020bandit}
\begin{align}
\mathfrak{R}_n \triangleq n\mathbb{E}_{\pi^{\star}_{\mathrm{D}}}[J] - \mathbb{E}_{\bm{\pi}_{n,\mathrm{D}}}\left[\sum_{l=1}^{n}J_l\right],\label{eq:regret}
\end{align}
where $n$ is the total computational time in minutes, $J_l$ is the value of (\ref{eq:defender_objective}) achieved after $l$ minutes, and $\bm{\pi}_{n,\mathrm{D}}=(\pi_{1,\mathrm{D}},\pi_{2,\mathrm{D}},\hdots, \pi_{n,\mathrm{D}})$ represents the sequence of defender strategies after $n$ minutes of computations (e.g., $n$ minutes of tree search) \cite{lattimore2020bandit}. Since computing $\mathbb{E}_{\pi^{\star}_{\mathrm{D}}}[J]$ is \textsc{pspace}-hard \cite[Thm. 6]{pspace_complexity}, we estimate $\mathbb{E}_{\pi^{\star}_{\mathrm{D}}}[J]$ using the current state-of-the-art value when computing (\ref{eq:regret}).

\begin{table*}
  \setlength{\tabcolsep}{2pt}
  \centering
\resizebox{1\textwidth}{!}{%
  \begin{tabular}{l|l|lll|lll|lll}
    \toprule
    \multirow{2}{*}{Method} & Training / search &
      \multicolumn{3}{c|}{$\mathcal{T}=30$} &
     \multicolumn{3}{c|}{$\mathcal{T}=50$} &
     \multicolumn{3}{c}{$\mathcal{T} = 100$} \\
     &  (minutes) (seconds)& {\textsc{scenario} \ref{scenario:1}} & {\textsc{scenario} \ref{scenario:2}} & {\textsc{scenario} \ref{scenario:3}} & {\textsc{scenario} \ref{scenario:1}} & {\textsc{scenario} \ref{scenario:2}} & {\textsc{scenario} \ref{scenario:3}} & {\textsc{scenario} \ref{scenario:1}} & {\textsc{scenario} \ref{scenario:2}} & {\textsc{scenario} \ref{scenario:3}}\\
      \midrule
    \textsc{cardiff}  & $2000$ / $10^{-4}$ & $-3.57 \pm 0.06$ & $-5.69 \pm 1.68$ & $-4.76 \pm 1.90$ & $-6.44 \pm 0.16$ & $-9.23 \pm 2.87$ & $-7.64 \pm 2.78$ & $-13.69 \pm 0.533$ & $-17.16 \pm 4.41$ & $-15.28 \pm 4.18$ \\
    \midrule
    \textsc{c-pomcp} & $0$ / $0.05$ & $-4.64 \pm 0.5$ & $-5.73 \pm 0.08$ & $-5.18 \pm 0.13$ & $-9.20 \pm 0.38$ & $-9.35 \pm 0.16$ & $-9.27 \pm 0.67$ & $-25.05 \pm 3.02$ & $-18.29 \pm 0.13$ & $-21.67 \pm 3.19$\\
    \textsc{c-pomcp} & $0$ / $0.1$ & $-3.89 \pm 0.25$ & $-5.62 \pm 0.14$ & $-4.75 \pm 0.34$ & $-8.46 \pm 0.27$ & $-8.92 \pm 0.23$ & $-8.69 \pm 0.47$ & $-21.28 \pm 0.72$ & $-17.38 \pm 0.20$ & $-19.33 \pm 1.03$\\
    \textsc{c-pomcp} & $0$ / $0.5$ & $-4.00 \pm 0.14$ & $-5.61 \pm 0.02$ & $-4.81 \pm 0.24$ & $-7.38 \pm 0.19$ & $-8.62 \pm 0.18$ & $-8.00 \pm 0.44$ & $-18.08 \pm 1.32$ & $-16.81 \pm 0.14$ & $-17.45 \pm 1.14$\\
    \textsc{c-pomcp} & $0$ / $1$ & $-3.64 \pm 0.13$ & $\bm{-5.52 \pm 0.16}$ & $-4.58 \pm 0.27$ & $-6.60 \pm 0.32$ & $-8.55 \pm 0.08$ & $-7.58 \pm 0.29$ & $-17.42 \pm 1.08$ & $-16.34 \pm 0.44$ & $-16.88 \pm 1.29$\\
    \textsc{c-pomcp} & $0$ / $5$ & $-3.50 \pm 0.11$ & $-5.65 \pm 0.11$ & $-4.57 \pm 0.23$ & $-6.52 \pm 0.34$ & $-8.46 \pm 0.11$ & $-7.49 \pm 0.46$ & $-13.23 \pm 0.43$ & $-16.46 \pm 0.30$ & $-14.85 \pm 0.79$\\
    \textsc{c-pomcp} & $0$ / $15$ & $\bm{-3.37 \pm 0.08}$ & $-5.66 \pm 0.11$ & $\bm{-4.52 \pm 0.19}$ & $-6.57 \pm 0.38$ & $-8.57 \pm 0.13$ & $-7.57 \pm 0.52$ & $\bm{-12.98 \pm 1.55}$ & $\bm{-15.87 \pm 0.67}$ & $\bm{-14.43 \pm 1.99}$\\
    \textsc{c-pomcp} & $0$ / $30$ & $-3.42 \pm 0.09$ & $-5.70 \pm 0.09$ & $-4.56 \pm 0.14$ & $\bm{-6.34 \pm 0.28}$ & $\bm{-8.52 \pm 0.18}$ & $\bm{-7.43 \pm 0.52}$ & $-13.32 \pm 0.18$ & $-16.05 \pm 0.96$ & $-14.68 \pm 1.02$\\
    \midrule
    \textsc{pomcp} & $0$ / $0.05$ & $-6.87 \pm 0.21$ & $-9.50 \pm 0.19$ & $-8.19 \pm 0.37$ & $-13.90 \pm 0.24$ & $-22.26 \pm 0.44$ & $-18.08 \pm 0.72$ & $-38.71 \pm 1.99$ & $-50.24 \pm 2.67$ & $-44.48 \pm 3.11$\\
    \textsc{pomcp}  & $0$ / $0.1$ & $-6.31 \pm 0.12$ & $-8.70 \pm 0.07$ & $-7.51 \pm 0.19$ & $-13.71 \pm 0.22$ & $-20.20 \pm 0.47$ & $-16.96 \pm 0.76$ & $-38.02 \pm 0.53$ & $-46.40 \pm 0.64$ & $-42.21 \pm 0.79$\\
    \textsc{pomcp}  & $0$ / $0.5$ & $-5.32 \pm 0.24$ & $-8.28 \pm 0.13$ & $-6.80 \pm 0.33$ & $-12.89 \pm 0.20$ & $-19.16 \pm 0.09$ & $-16.03 \pm 0.33$ & $-34.92 \pm 0.96$ & $-47.29 \pm 0.39$ & $-41.11 \pm 1.23$\\
    \textsc{pomcp}  & $0$ / $1$ & $-5.27 \pm 0.65$ & $-7.68 \pm 0.10$ & $-6.48 \pm 0.75$ & $-12.57 \pm 0.41$ & $-18.38 \pm 0.50$ & $-15.48 \pm 0.94$ & $-34.50 \pm 0.65$ & $-47.02 \pm 1.75$ & $-40.76 \pm 2.34$\\
    \textsc{pomcp}  & $0$ / $5$ & $-5.11 \pm 0.32$ & $-7.58 \pm 0.05$ & $-6.35 \pm 0.36$ & $-12.03 \pm 0.93$ & $-18.22 \pm 0.19$ & $-15.13 \pm 1.18$ & $-33.06 \pm 0.21$ & $-45.15 \pm 0.54$ & $-39.13 \pm 0.66$\\
    \textsc{pomcp} & $0$ / $15$ & $-5.18 \pm 0.66$ & $-7.30 \pm 0.38$ & $-6.24 \pm 1.22$ & $-11.32 \pm 0.84$ & $-17.68 \pm 0.58$ & $-14.50 \pm 1.37$ & $-30.88 \pm 1.41$ & $-45.19 \pm 0.35$ & $-38.04 \pm 1.57$\\
    \textsc{pomcp}  & $0$ / $30$ & $-4.50 \pm 0.17$ & $-6.92 \pm 0.59$ & $-5.71 \pm 0.77$ & $-9.88 \pm 1.63$ & $-17.55 \pm 0.44$ & $-13.72 \pm 2.09$ & $-29.51 \pm 2.00$ & $-44.27 \pm 1.13$ & $-36.89 \pm 2.49$\\
    \bottomrule
  \end{tabular}
}
  \caption{Comparing \textsc{c-pomcp} with baselines: \textsc{cardiff-ppo} \cite{vyas2023automated} and \textsc{pomcp} \cite{pomcp}; columns indicate the time horizon $\mathcal{T}$ (\ref{eq:defender_objective}); subcolumns indicate the evaluation scenario (\S \ref{sec:evaluation_scenarios}); numbers indicate the mean and the standard deviation of the reward $J$ (\ref{eq:defender_objective}) from evaluations with $3$ random seeds.}\label{tab:results}
\end{table*}

\begin{figure*}
  \centering
  \scalebox{0.83}{
    \input{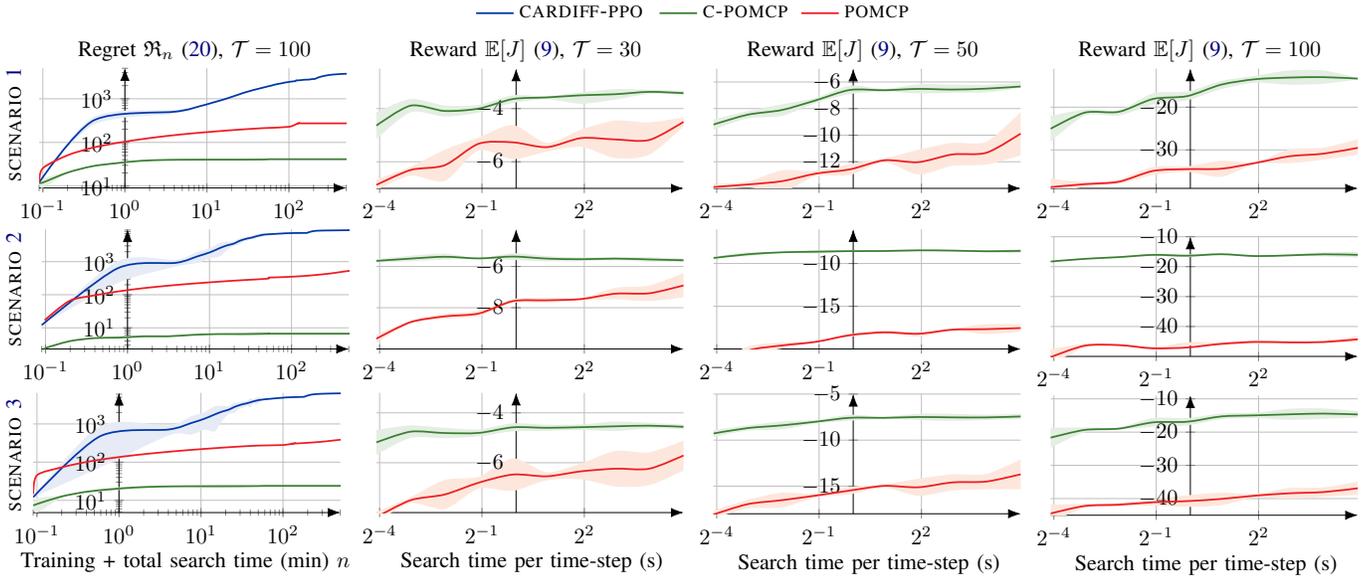}
  }
  \caption{Comparing \textsc{c-pomcp} (green curves) with baselines: \textsc{cardiff-ppo} \cite{vyas2023automated} (blue curves) and \textsc{pomcp} \cite{pomcp} (red curves); rows indicate the evaluation scenario (\S \ref{sec:evaluation_scenarios}); the curves show the mean value from evaluations with $3$ random seeds; shaded areas indicate the standard deviation; the number of data points on the x-axis is $2000$ for the left-most column and $10$ for the other columns.}
  \label{fig:b_line_results}
\end{figure*}

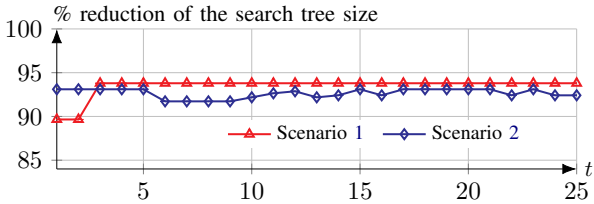
\begin{figure}
  \centering
  \scalebox{0.95}{
    \pgfplotstableread{
1 89.6551724137931 0.0
2 89.6551724137931 0.0
3 93.79310344827586 0.0
4 93.79310344827586 0.0
5 93.79310344827586 0.0
6 93.79310344827586 0.0
7 93.79310344827586 0.0
8 93.79310344827586 0.0
9 93.79310344827586 0.0
10 93.79310344827586 0.0
11 93.79310344827586 0.0
12 93.79310344827586 0.0
13 93.79310344827586 0.0
14 93.79310344827586 0.0
15 93.79310344827586 0.0
16 93.79310344827586 0.0
17 93.79310344827586 0.0
18 93.79310344827586 0.0
19 93.79310344827586 0.0
20 93.79310344827586 0.0
21 93.79310344827586 0.0
22 93.79310344827586 0.0
23 93.79310344827586 0.0
24 93.79310344827586 0.0
25 93.79310344827586 0.0
26 93.79310344827586 0.0
27 93.79310344827586 0.0
28 93.79310344827586 0.0
29 93.79310344827586 0.0
30 93.79310344827586 0.0
31 93.79310344827586 0.0
32 93.79310344827586 0.0
33 93.79310344827586 0.0
34 93.79310344827586 0.0
35 93.79310344827586 0.0
36 93.79310344827586 0.0
37 93.79310344827586 0.0
38 93.79310344827586 0.0
39 93.79310344827586 0.0
40 93.79310344827586 0.0
41 93.79310344827586 0.0
42 93.79310344827586 0.0
43 93.79310344827586 0.0
44 93.79310344827586 0.0
45 93.79310344827586 0.0
46 93.79310344827586 0.0
47 93.79310344827586 0.0
48 93.79310344827586 0.0
49 93.79310344827586 0.0
50 93.79310344827586 0.0
}\bline

\pgfplotstableread{
1 93.10344827586207 0.0
2 93.10344827586207 0.0
3 93.10344827586207 0.0
4 93.10344827586207 0.0
5 93.10344827586207 0.0
6 91.72413793103448 0.0
7 91.72413793103448 0.0
8 91.72413793103448 0.0
9 91.72413793103448 0.0
10 92.18390804597701 0.006502131321255622
11 92.64367816091954 0.006502131321255622
12 92.8735632183908 0.008601511233963102
13 92.18390804597701 0.006502131321255622
14 92.41379310344827 0.009753196981883434
15 93.10344827586207 0.009753196981883434
16 92.41379310344827 0.009753196981883434
17 93.10344827586207 0.009753196981883434
18 93.10344827586207 0.009753196981883434
19 93.10344827586207 0.009753196981883434
20 93.10344827586207 0.009753196981883434
21 93.10344827586207 0.009753196981883434
22 92.41379310344827 0.009753196981883434
23 93.10344827586207 0.009753196981883434
24 92.41379310344827 0.009753196981883434
25 92.41379310344827 0.009753196981883434
26 93.10344827586207 0.009753196981883434
27 93.10344827586207 0.009753196981883434
28 93.10344827586207 0.009753196981883434
29 93.10344827586207 0.009753196981883434
30 92.41379310344827 0.009753196981883434
31 93.79310344827586 0.0
32 93.10344827586207 0.009753196981883434
33 92.41379310344827 0.009753196981883434
34 93.79310344827586 0.0
35 92.18390804597701 0.006502131321255622
36 92.41379310344827 0.009753196981883434
37 93.10344827586207 0.009753196981883434
38 93.79310344827586 0.0
39 93.10344827586207 0.009753196981883434
40 93.10344827586207 0.009753196981883434
41 93.79310344827586 0.0
42 93.10344827586207 0.009753196981883434
43 93.10344827586207 0.009753196981883434
44 92.41379310344827 0.009753196981883434
45 93.10344827586207 0.009753196981883434
46 93.79310344827586 0.0
47 93.10344827586207 0.009753196981883434
48 93.10344827586207 0.009753196981883434
49 93.10344827586207 0.009753196981883434
50 92.41379310344827 0.009753196981883434
}\meander

\begin{tikzpicture}]
\node[scale=1] (kth_cr) at (0,2.15)
{
  \begin{tikzpicture}
    \begin{axis}[
        xmax=25,
        ymin=84,
        ymax=100,
        width = \columnwidth,
        height = 0.4\columnwidth,
        axis lines=center,
        grid = both,
        major grid style = {lightgray},
        minor grid style = {lightgray!25},
        scaled y ticks=false,
        xlabel style={below right},
        ylabel style={above left},
        axis line style={-{Latex[length=2mm]}},
        %smooth,
        legend style={at={(0.9,0.35)}},
        legend columns=2,
        legend style={
          draw=none,
            inner sep=0pt,
            font=\footnotesize,
            /tikz/column 2/.style={
                column sep=5pt,
              }
              }
              ]
\addplot[Red,mark=triangle, mark repeat=1, name path=l1, thick, domain=1:50] plot[error bars/.cd, y dir=both, y explicit] table [x index=0, y index=1, y error plus index=2, y error minus index=2] {\bline};
\addplot[Blue,mark=diamond, mark repeat=1, name path=l1, thick, domain=1:50] plot[error bars/.cd, y dir=both, y explicit] table [x index=0, y index=1, y error plus index=2, y error minus index=2] {\meander};

%\addplot[Blue,mark=diamond, smooth, thick, domain=0:100] (x,{(25+12*5)*x + 2});
%\addplot[Red,mark=x, smooth, thick, domain=0:100] (x,{(25+12*10)*x + 2});
%\addplot[Black,mark=triangle, smooth, thick, domain=0:100] (x,{(25+12*20)*x + 2});
\legend{Scenario \ref{scenario:1},Scenario \ref{scenario:2}}
\end{axis}
\node[inner sep=0pt,align=center, scale=0.9, rotate=0, opacity=1] (obs) at (7.5,0)
{
  $t$
};
\node[inner sep=0pt,align=center, scale=0.9, rotate=0, opacity=1] (obs) at (2.28,2.18)
{
  \% reduction of the search tree size
};

\end{tikzpicture}
};

  \end{tikzpicture}
  }
  \caption{Effect of the pruning of the search tree in \textsc{c-pomcp}.}
  \label{fig:pruning}
\end{figure}

\begin{table}
  \centering
  \scalebox{0.9}{
  \begin{tabular}{llll} \toprule
    {\textit{Method}} & {\textit{Training (min)}} & {\textit{Search (s)}} & {Reward \textit{$\mathbb{E}[J]$ (\ref{eq:defender_objective})}} \\ \midrule
    \textsc{cardiff-ppo}\cite{vyas2023automated} & $2000$ & $10^{-4}$ & $-429 \pm 167$\\
    \textsc{c-pomcp} & $0$ & $30$ & $\bm{-13.32 \pm 0.18}$\\
    \textsc{pomcp}  \cite{pomcp} & $0$ & $30$ & $-29.51 \pm 2.00$\\
    \bottomrule\\
  \end{tabular}}
  \caption{Evaluation results for Scenario \ref{scenario:4} with $\mathcal{T}=100$.}\label{tab:scenario_4}
\end{table}
\subsection{\textsc{cage-2} Scenarios}\label{sec:evaluation_scenarios}
The \textsc{cage-2} scenario can be instantiated with different attacker strategies $\pi_{\mathrm{A}}$ (\ref{eq:attacker_function}) as well as different topologies within each zone (see Fig. \ref{fig:use_case}) \cite{cage_challenge_2_announcement}. Based on these parameters, we define the following evaluation scenarios.
%We define four scenarios to evaluate the performance properties of \textsc{c-pomcp}. All scenarios are based on \textsc{cage-2}  (Fig. \ref{fig:use_case}) to enable a fair comparison with the baselines \cite{cage_challenge_2_announcement}.
\begin{scenario}[\textsc{b-line} attacker]\label{scenario:1}
\normalfont In this scenario, the attacker strategy $\pi_{\mathrm{A}}$ represents the \textsc{b-line} attacker from \textsc{cage-2} \cite{cage_challenge_2_announcement}, which attempts to move directly to the operational zone. The topology is shown in Fig. \ref{fig:use_case}.
\end{scenario}

\begin{scenario}[\textsc{meander} attacker]\label{scenario:2}
\normalfont In this scenario, $\pi_{\mathrm{A}}$ represents the \textsc{meander} attacker from \textsc{cage-2} \cite{cage_challenge_2_announcement}. \textsc{meander} explores the network one zone at a time, seeking to gain privileged access to all hosts in a zone before moving on to the next one, eventually arriving at the operational zone. The topology is shown in Fig. \ref{fig:use_case}.
\end{scenario}

\begin{scenario}[\textsc{random} attacker]\label{scenario:3}
\normalfont In this scenario, $\pi_{\mathrm{A}}$ is \textsc{b-line} with probability $0.5$ and \textsc{meander} with probability $0.5$. The topology is shown in Fig. \ref{fig:use_case}.
\end{scenario}

\begin{scenario}[\textsc{random} topology]\label{scenario:4}
\normalfont This scenario is the same as Scenario \ref{scenario:1} except that the topologies of the enterprise and user zones are randomized at the start of each evaluation episode.
\end{scenario}
\subsection{\textsc{cage-2} Benchmark Results}
The evaluation results are summarized in Figs. \ref{fig:b_line_results}--\ref{fig:pruning} and Tables \ref{tab:results}--\ref{tab:scenario_4}. The results show that \textsc{c-pomcp} (Alg. \ref{alg:c_pomcp}) achieves the highest reward (\ref{eq:defender_objective}) and the lowest regret (\ref{eq:regret}) across all evaluation scenarios and time horizons $\mathcal{T}$. (The results are not statistically significant in all cases, though.)

The green curves in Fig. \ref{fig:b_line_results} relate to \textsc{c-pomcp}. The blue and red curves relate to the baselines. The leftmost column in Fig. \ref{fig:b_line_results} shows the regret (\ref{eq:regret}). Notably, the regret of \textsc{c-pomcp} is two orders of magnitude lower than the regret of \textsc{cardiff-ppo} and one order of magnitude lower than the regret of \textsc{pomcp}.

The three rightmost columns in Fig. \ref{fig:b_line_results} show the cumulative reward (\ref{eq:defender_objective}) obtained by \textsc{c-pomcp} and \textsc{pomcp} in function of the search time $s_{\mathrm{T}}$. We observe that \textsc{c-pomcp} achieves a significantly higher reward than \textsc{pomcp}, although the difference diminishes with increasing $s_{\mathrm{T}}$, which is expected (Thm. \ref{thm:pomcp_convergence}). We explain the improvement of \textsc{c-pomcp} compared to \textsc{pomcp} by the pruned search tree, which is obtained by leveraging the causal structure (Def. \ref{def:pomis}). The reduction in search tree size achieved by the pruning is shown in Fig. \ref{fig:pruning}. We see that the pruning reduces the size of the search tree by around 90--95\%.

Lastly, Table \ref{tab:scenario_4} contains the results for Scenario \ref{scenario:4}. We find that \textsc{c-pomcp} and \textsc{pomcp} are agnostic to changes in the topology within each zone. By contrast, the performance of \textsc{cardiff-ppo} reduces drastically when the topology changes, indicating that its strategy is overfitted to the training environment \cite{vyas2023automated}. \textsc{cardiff-ppo} has shown similar behavior in \cite{beyond_cage}.

\subsection{Discussion of the \textsc{cage-2} Benchmark Results}
The key findings from the \textsc{cage-2} benchmark results are:
\begin{enumerate}
\item Leveraging the causal structure of the target system, \textsc{c-pomcp} achieves state-of-the-art performance (Figs. \ref{fig:b_line_results}--\ref{fig:pruning}, Table \ref{tab:results}).
\item The interventions prescribed by \textsc{c-pomcp} are guaranteed to converge to optimal interventions as $s_{\mathrm{T}}\rightarrow \infty$ (Thm. \ref{thm:pomcp_convergence}), which is consistent with the evaluation results.
\item \textsc{c-pomcp} is two orders of magnitude more efficient in computing time than the state-of-the-art method \textsc{cardiff-ppo} (Fig. \ref{fig:b_line_results}).
  %Further, it is one order of magnitude more efficient than its non-causal version \textsc{pomcp} \cite[Alg. 1]{pomcp}.
\item \textsc{c-pomcp} is an online method and can adapt to changes in the topology of the target system (Table \ref{tab:scenario_4}).
\end{enumerate}
Surprisingly, the results demonstrate that \textsc{c-pomcp} requires only $5-15$ seconds of search to achieve competitive performance on the \textsc{cage-2} benchmark. The fact that \textsc{c-pomcp} performs significantly better than its non-causal version \textsc{pomcp} \cite[Alg. 1]{pomcp} indicates that the main enabler of the efficiency is the causal structure, which we exploit for pruning the search space. This observation suggests limitations of existing methods that narrowly focus on \textit{model-free} reinforcement learning and do not consider the causal structure of the underlying system.
%important areas of improvement of methods proposed in prior work, which is focused narrowly on model-free reinforcement learning.
%The most surprising part of the results is that \textsc{c-pomcp} requires relatively little search time to achieve competitive performance on \textsc{cage-2}.

While the results demonstrate clear benefits of \textsc{c-pomcp} compared to the existing methods, \textsc{c-pomcp} has two drawbacks. First, execution of \textsc{c-pomcp} is slower than that of pre-trained methods (see Table \ref{tab:results}, typically $10^{-4}$s vs $10$s). Second, the performance of \textsc{c-pomcp} depends on the causal structure of the target system \cite[Def. 2.2.1]{pearl2000causality}. If no causal structure is known, the performance of \textsc{c-pomcp} drops (cf. the performance of \textsc{c-pomcp} and \textsc{pomcp} in Table \ref{tab:results}).
\section{Conclusion}
This paper presents a formal (causal) model of \textsc{cage-2} (\hyperref[pos]{M1}), which is considered a standard benchmark to evaluate methods for autonomous cyber defense \cite{cage_challenge_2_announcement}. Based on this model, we prove the existence of optimal defender strategies (Thm. \ref{thm:well_defined_prob}) and design an iterative method that converges to such a strategy (Thm. \ref{thm:pomcp_convergence}). The method, which we call Causal Partially Observable Monte-Carlo Planning (\textsc{c-pomcp}), leverages causal structure to prune, construct and traverse a search tree (Alg. \ref{alg:c_pomcp}). \textsc{c-pomcp} has four advantages over the state-of-the-art methods that have been proposed in the context of \textsc{cage-2}: (\textit{i}) it is two orders of magnitude more computationally efficient (Fig. \ref{fig:b_line_results}); (\textit{ii}) it achieves better performance (Table. \ref{tab:results}); (\textit{iii}) it is an online method which adapts to topology changes in the target system (Table \ref{tab:scenario_4}); and (\textit{iv}), it produces provably optimal defender strategies (Thm. \ref{thm:pomcp_convergence}). Future work will focus on leveraging causal discovery techniques to automate the process of creating a causal graph of the underlying system.
%We plan to continue this work three directions. First, we will evaluate \textsc{c-pomcp} on additional benchmarks and use cases. Second, we plan to leverage causal discovery techniques to automate the process of creating the causal graph. Third, we intend to extend our control-theoretic version of \textsc{c-pomcp} to a game-theoretic version, which allows us to study defender strategies against dynamic attackers.
\section{Acknowledgments}
% This research is supported by the Defense Advanced Research Project Agency (\textsc{darpa}) through the \textsc{castle} program under Contract No. \textsc{w912cg23c0029}.
The authors are grateful to \textsc{darpa} and program manager Tejas Patel for supporting this research. The authors would also like to thank the Siemens research team for discussions and feedback: Enrico Lovat, Jagannadh Vempati, Anton Kocheturov, Arun Ramamurthy, Arif Haque, and Abhishek Ramchandran. Finally, the authors thank \textsc{kth} researchers Forough Shahab Samani, Xiaoxuan Wang, and Duc Huy Le for constructive comments.

\appendices

\section{Proof of Theorem \ref{thm:particles}}\label{appendix:proof_proposition_1}
  \begin{enumerate}
  \item \textsc{case:} $t=1$.\\
    $\text{ }\text{ }$\textsc{proof:} $\mathbf{b}_1$ is given, hence $\widehat{\mathbf{b}}_1=\mathbf{b}_1$.
  \item \textsc{case:} $t > 1$.\\
    $\text{ }\text{ }$\textsc{proof:} Assume $\widehat{\mathbf{b}}_{t-1}=\mathbf{b}_{t-1}$ and let
    $$g(\overline{\bm{\sigma}}_t)\triangleq \sum_{\bm{\sigma}_{t-1}}\widehat{\mathbf{b}}_{t-1}(\bm{\sigma}_{t-1})P(\overline{\bm{\sigma}}|\bm{\sigma}_{t-1}, \mathrm{do}(\widehat{\mathbf{X}}_{t-1} = \widehat{\mathbf{x}}_{t-1})).$$
    We then have that
\begin{align}
&\mathbf{b}_t(\bm{\sigma}) = \mathbb{E}_{\overline{\bm{\sigma}}\sim \mathbf{b}_t}[\mathbbm{1}_{\bm{\sigma}=\overline{\bm{\sigma}}}]=\sum_{\overline{\bm{\sigma}}}\mathbf{b}_t(\overline{\bm{\sigma}})\mathbbm{1}_{\bm{\sigma}=\overline{\bm{\sigma}}}\nonumber\\
  &\numeq{a}\sum_{\overline{\bm{\sigma}}}\frac{g(\overline{\bm{\sigma}})}{g(\overline{\bm{\sigma}})}\mathbf{b}_t(\overline{\bm{\sigma}})\mathbbm{1}_{\bm{\sigma}=\overline{\bm{\sigma}}}=\mathbb{E}_{\overline{\mathbf{s}}\sim g}\left[\frac{\mathbf{b}_t(\overline{\bm{\sigma}})}{g(\overline{\bm{\sigma}})}\mathbbm{1}_{\bm{\sigma}=\overline{\bm{\sigma}}}\right]\nonumber\\
  &\numeq{b}\mathbb{E}_{\overline{\bm{\sigma}}\sim g}\left[\frac{\eta P(\mathbf{o}_t \mid \bm{\sigma}_t)g(\overline{\bm{\sigma}})}{g(\overline{\bm{\sigma}})}\mathbbm{1}_{\bm{\sigma}=\overline{\bm{\sigma}}}\right]\nonumber\\
  &=\mathbb{E}_{\overline{\bm{\sigma}}\sim g}\left[\eta P(\mathbf{o}_t \mid \bm{\sigma}_t)\mathbbm{1}_{\bm{\sigma}=\overline{\bm{\sigma}}}\right],\label{slln_particle}
\end{align}
where (b) follows from the definition of $\mathbf{b}$.

As $\widehat{\mathbf{b}}_{t-1}=\mathbf{b}_{t-1}$, the denominator in (a) is non-zero for all $\overline{\bm{\sigma}}$ where $\mathbf{b}_t(\overline{\bm{\sigma}}) > 0$. Since the particles are distributed according to $\eta^{\prime}\eta P(\mathbf{o}_t | \bm{\sigma}_t)g(\widehat{\bm{\sigma}}_t)$ ($\eta^{\prime}$ is a normalizing factor), it follows from the strong law of large numbers \cite[Thm. 6.2]{cinlar_stochastics} that $\widehat{\mathbf{b}}(\bm{\sigma}_t) = \frac{1}{M}\sum_{i=1}^M\mathbbm{1}_{\bm{\sigma}_t=\widehat{\bm{\sigma}}_t^{(i)}}$ converges $P$-almost surely to (\ref{slln_particle}) as $M \rightarrow \infty$. (Remark: the probability measure $P$ in (\ref{slln_particle}) exists since $|\mathrm{dom}(\mathbf{\Sigma}_t) \cup \mathrm{dom}(\mathbf{O}_t)| < \infty$ (Thm. \ref{thm:well_defined_prob}, \cite[Thm. 2.2.1]{Rosenthal2006}).) \qed
\end{enumerate}
%\end{proof}
\section{Proof of Theorem \ref{thm:pomcp_convergence}}\label{appendix:proof_theorem_3}
It follows from Thm. \ref{thm:particles} and \cite[Lem. 1--2]{pomcp} that \textsc{c-pomcp} corresponds to the \textsc{uct} algorithm \cite[Fig. 1]{uct} when $M \rightarrow \infty$. Hence, we can base the proof of convergence of \textsc{c-pomcp} on the proof of convergence for \textsc{uct}, which was originally published in \cite{uct}. The key insight behind the proof is that the decision problem at each node in the search tree corresponds to a non-stationary multi-armed bandit (\textsc{mab}), which becomes stationary if the prescribed actions at the child nodes converge. Further, the tree policy corresponds to the \textsc{ucb1} algorithm \cite[Fig. 1]{ucb_1_auer}. As a consequence, it suffices to prove that \textsc{ucb1} converges at each node in the search tree. Towards this proof, we state and prove the following six lemmas.

\vspace{2mm}

\noindent\textbf{Notation.} $K=\mathcal{O}(|2^{\mathbf{X}_t}|)$ is the number of arms in the \textsc{mab} at each node in the search tree; $t$ indexes the \textsc{mab} rounds; $R_{i,t}$ is the reward of arm $i$ at round $t$; $\overline{R}_{i,n}=\frac{1}{n}\sum_{k=1}^nR_{i,n}$ is the mean reward of arm $i$ based on $n$ samples; $\mu_{i,n}$ is the mean of $\overline{R}_{i,n}$; $\mu_{i}\triangleq \lim_{n \rightarrow \infty} \mu_{i,n}$; $\mu_{i} \triangleq \mu_{i,n}-\delta_{i,n}$; $T_i(t)$ is the number of times arm $i$ has been pulled at round $t$; $\Delta_i\triangleq \mu^{\star}-\mu_i$; $I_t$ is the arm picked at round $t$; and $c_{t,n}$ is the exploration term in \textsc{ucb1} for an arm that has been pulled $n$ times at round $t$. Quantities related to the optimal arm are superscripted by $\star$, i.e., $\mu^{\star}, T^{\star}(t)$, etc.
\begin{assumption}[Bounded rewards and asymptotic stationarity]\label{assumption:eventual_station}
$\quad$
\begin{enumerate}
\item $R_{i,n} \in [0,1]$ for all $i$ and $n$.
\item The limit $\mu_{i}=\lim_{t \rightarrow \infty} \mu_{i,t}$ exists for each arm $i$.
\item There exists a constant $C_p$ and an integer $N_p$ such that for $n \geq N_p$ and any $\delta \geq 0$, the following bounds hold.
\begin{align*}
P(n \overline{R}_{i,n} \geq n \mu_{i,n} + C_p\sqrt[]{n \ln(1/\delta)}) \leq \delta\\
P(n \overline{R}_{i,n} \geq n \mu_{i,n} - C_p\sqrt[]{n \ln(1/\delta)}) \leq \delta.
\end{align*}
\end{enumerate}
\end{assumption}
\begin{lemma}\label{lemma:exp_concentration}
Given Assumption \ref{assumption:eventual_station}, if $c_{t,n}=2C_p\sqrt[]{\frac{\ln t}{n}}$, then
\begin{align*}
P(\overline{R}_{i,n} \geq \mu_{i,n} + c_{t,n}) \leq t^{-4} && n \geq N_p\\
P(\overline{R}_{i,n} \geq \mu_{i,n} - c_{t,n}) \leq t^{-4} && n \geq N_p.
\end{align*}
\end{lemma}
\begin{proof}
\begin{align*}
P(n \overline{R}_{i,n} \geq n \mu_{i,n} + C_p\sqrt[]{n \ln(1/\delta)}) &\numleq{\text{Assumption \ref{assumption:eventual_station}}} \delta\\
\implies P(\overline{R}_{i,n} \geq \mu_{i,n} + C_p\sqrt[]{\frac{ \ln(1/\delta)}{n}}) &\geq \delta\\
\numimp{\delta=t^{-4}} P(\overline{R}_{i,n} \geq \mu_{i,n} + 2C_p\sqrt[]{\frac{\ln(t)}{n}}) &\geq t^{-4}.
\end{align*}
\end{proof}
\begin{lemma}\label{lemma:delta_convergence}
If Assumption \ref{assumption:eventual_station} holds, then there exists an integer $N_0(\epsilon)$ such that $t \geq N_0(\epsilon)\implies |\delta_{i,t}| \leq \frac{\epsilon \Delta_i}{2}$ and $|\delta_{j^{\star}, t}| \leq \min_{i}\frac{\epsilon\Delta_i}{2}$ for all $\epsilon > 0$.
\end{lemma}
\begin{proof}
Follows by the definition of Assumption \ref{assumption:eventual_station}.
\end{proof}
\begin{lemma}\label{lemma:t_1_233}
Given Assumption \ref{assumption:eventual_station}, if the exploration term used by \textsc{ucb1} \cite[Fig. 1]{ucb_1_auer} is $c_{t,s}=2C_p\sqrt[]{\frac{\ln t}{s}}$, then
\begin{align}
\mathbb{E}[T_i(t)] \leq \frac{16C_p^2 \ln t}{(1-\epsilon)^2\Delta_i^2} + N_0(\epsilon) + N_p + 1 + \frac{\pi^2}{3}
\end{align}
for all $\epsilon > 0$ and each sub-optimal arm $i$.
\end{lemma}
\begin{proof}
Let $A_0(t, \epsilon) \triangleq \min\{s | c_{t,s} \leq (1-\epsilon)\Delta_i/2\}$ and $A(t, \epsilon) \triangleq \max[A_0(t, \epsilon), N_0(\epsilon), N_p]$. Next note that
\begin{align*}
&c_{t,s} \leq (1-\epsilon)\Delta_i/2 \implies 2C_p\sqrt[]{\frac{\ln t}{s}} \leq (1-\epsilon)\Delta_i/2\\
&\implies 16C_p\frac{\ln t}{s} \leq (1-\epsilon)^2\Delta^2_i \implies s  \geq \frac{16C_p\ln t}{(1-\epsilon)^2\Delta^2_i}\\
&\implies A_0(t, \epsilon) = \left \lceil \frac{16C_p\ln t}{(1-\epsilon)^2\Delta^2_i} \right\rceil.
\end{align*}
Now consider $T_i(n)$. By definition:
\begin{align}
&T_i(n) = 1 + \sum_{t=K+1}^n\mathbbm{1}\left\{I_t = i\right\}\label{eq:lemma_3_0}\\
        &= 1 + \sum_{t=K+1}^n\mathbbm{1}\left\{I_t = i, T_i(t-1) \geq A(n, \epsilon)\right\} + \nonumber\\
  &\quad\quad\quad \sum_{t=K+1}^n\mathbbm{1}\left\{I_t = i, T_i(t-1) < A(n, \epsilon)\right\}\nonumber\\
         &\numleq{a} A(n, \epsilon) + \sum_{t=K+1}^n\mathbbm{1}\left\{I_t = i, T_i(t-1) \geq A(n, \epsilon)\right\}\nonumber\\
  &\numleq{b} A(n, \epsilon) + \sum_{t=K+1}^n\mathbbm{1}\Bigl\{\overline{R}^{\star}_{T^{\star}(t-1)} + c_{t-1,T^{\star}(t-1)} \leq\nonumber\\
  &\quad\quad\quad\quad\overline{R}_{i,T_i(t-1)} + c_{t-1,T_i(t-1)}, T_i(t-1) \geq A(n, \epsilon)\Bigr\}\nonumber\\
  &\leq A(n, \epsilon) + \sum_{t=K+1}^n\mathbbm{1}\Bigl\{\min_{0 < s < t}\overline{R}^{\star}_{s} + c_{t-1,s} \leq \nonumber\\
  &\quad\quad\quad\quad\quad\quad\quad\quad\quad\quad\quad\max_{A(n,\epsilon) < s_i < t} \overline{R}_{i,s_i} + c_{t-1,s_i} \Bigr\}\nonumber\\
  & \leq A(n, \epsilon) + \sum_{t=K+1}^n\sum_{s=1}^{t-1}\sum_{s_i=A(n,\epsilon)}^{t-1}\mathbbm{1}\Bigl\{\overline{R}^{\star}_{s} + c_{t-1,s} \nonumber\\
  &\quad\quad\quad\quad\quad\quad\quad\quad\quad\quad\quad\quad\quad\quad\leq \overline{R}_{i,s_i} + c_{t-1,s_i}\Bigr\},\nonumber
\end{align}
where (a) follows because the second sum is upper bounded by $n-K$ and $n < A(n,\epsilon) \implies T_i(n) < A(n,\epsilon)$. (b) follows from the arm-selection rule in \textsc{ucb1} \cite[Fig. 1]{ucb_1_auer}.

Next note that for $t \geq A(n, \epsilon) \geq N_0(\epsilon)$, we have $\mu_t^{\star} \geq \mu_{i,t} + 2c_{t,s_i}$. This inequality holds because
\begin{align}
&\mu_t^{\star} \geq \mu_{i,t} + 2c_{t,s_i} \iff \mu^{\star}_{t}-\mu_{i,t} - 4C_p\sqrt[]{\frac{\ln t}{s_i}} \geq 0 \nonumber\\
&\numiff{t \geq A_0(n,\epsilon)}  \mu^{\star}_{t}-\mu_{i,t} - (1-\epsilon)\Delta_i \geq 0 \nonumber\\
&\iff  \mu^{\star}_{t} + \delta^{\star} -\mu_{i,t} - \delta_{i,t} - (1-\epsilon)\Delta_i \geq 0  \nonumber \\
&\numiff{\text{Lemma \ref{lemma:delta_convergence}}}  \mu^{\star}_{t} -\epsilon \Delta_i -\mu_{i,t} - (1-\epsilon)\Delta_i \geq 0  \nonumber \\
&\iff  \mu^{\star}_{t}  -\mu_{i,t} -\Delta_i \geq 0  \iff \Delta_i-\Delta_i \geq 0. \label{eq:lemma_3_1}
\end{align}
Using the above inequality, we deduce that
\begin{align}
&P(\overline{R}^{\star}_s + c_{t-1,s} \leq \overline{R}_{i,s_i} + c_{t-1,s_i}) \leq \label{eq:lemma_3_2}\\
&P(\overline{R}^{\star}_s \leq \mu_t^{\star} + c_{t,s}) + P(\overline{X}_{i,s_i} \geq \mu_{i,t} + c_{t,s_i}).\nonumber
\end{align}
This follows because if the left inequality above holds and the right inequalities do not hold, we obtain
\begin{align*}
\overline{R}^{\star}_s + c_{t-1,s} &\leq  \overline{R}_{i,s_i} + c_{t-1,s_i}\\
\implies \mu^{\star}_t -c_{t,s} + c_{t-1,s} &< \mu_{i,t} + c_{t,s_i} + c_{t-1,s_i}\\
\implies \mu^{\star}_t &< \mu_{i,t} + 2c_{t,s_i},
\end{align*}
which by (\ref{eq:lemma_3_1}) is false for $t \geq A(n, \epsilon) \geq N_0(\epsilon)$. Now we take expectations of both sides of the inequality in (\ref{eq:lemma_3_0}) and plug in (\ref{eq:lemma_3_2}), which gives
\begin{align*}
  &\mathbb{E}[T_i(n)] \leq A(n,\epsilon) + \sum_{t=K+1}^n\sum_{s=1}^{t-1}\sum_{s_i=A(n,\epsilon)}^{t-1}P(\overline{R}^{\star}_s \leq \mu_t^{\star} + c_{t,s})\\
  &\quad\quad\quad\quad\quad\quad\quad\quad\quad\quad\quad\quad\quad + P(\overline{X}_{i,s_i} \geq \mu_{i,t} + c_{t,s_i})\\
&\numleq{\text{Lemma \ref{lemma:exp_concentration}}}A(n,\epsilon) + \sum_{t=K+1}^n\sum_{s=1}^{t-1}\sum_{s_i=A(n,\epsilon)}^{t-1}2t^{-4}\\
&\leq A(n,\epsilon) + \sum_{t=1}^{\infty}\sum_{s=1}^{t}\sum_{s_i=1}^{t}2t^{-4}\\
&= A(n,\epsilon) + \sum_{t=1}^{\infty}t^{-2} + t^{-3} \leq A(n,\epsilon) + \sum_{t=1}^{\infty}t^{-2} + \sum_{t=1}^{\infty}t^{-2}\\
&\numeq{a} A(n,\epsilon) + \frac{\pi^2}{3} \leq \frac{16C_p\ln t}{(1-\epsilon)^2\Delta^2_i} + 1 + N_0(\epsilon) + N_p + \frac{\pi^2}{3},
\end{align*}
where (a) follows from the Riemann zeta function $\zeta(2) = \sum_{t=1}^{\infty}t^{-2}=\frac{\pi^2}{6}$.
\end{proof}
\begin{lemma}[Lower bound]\label{lemma:lower_bound}
Given Assumption \ref{assumption:eventual_station}, there exists a positive constant $\rho$ such that for all $i$ and $t$, $T_i(t) \geq \lceil \rho \log t \rceil$.
\end{lemma}
\begin{proof}
Since $R_{i,t}\in [0,1]$ and $T_i(t-1) \geq 1$ for all $t \geq K$, there exists a constant $M$ such that
\begin{align*}
  \mu_{i,t} + 2C\sqrt[]{\frac{\ln t}{T_i(t-1)}} &\leq M\\
\implies T_i(t-1) &\geq \frac{4C^2\ln t}{(M-\mu_i-\delta_{i,t})^2}
\end{align*}
for all $i$ and $K \leq t < \infty$. Next note that Assumption \ref{assumption:eventual_station} implies that $\lim_{t \rightarrow \infty}\delta_{i,t}=0$, which means that there exists a constant $\rho \geq \frac{4C^2}{(M-\mu_i-\delta_{i,t})^2}$. Hence $T_i(t) \geq \lceil \rho \log t \rceil$.
\end{proof}
\begin{lemma}\label{lemma:thm_3}
Let $\overline{R}_n = \sum_{i=1}^{K}\frac{T_i(n)}{n}\overline{X}_{i,T_i(n)}$ and $N_0=N_0(\epsilon=\frac{1}{2})$. Then, the following holds under Assumption \ref{assumption:eventual_station}.
\begin{align*}
|\mathbb{E}[\overline{R}_n] - \mu^{\star}| \leq |\delta_n^{\star}| + \mathcal{O}\left(\frac{K(C_p^2 \ln n + N_0)}{n}\right).
\end{align*}
\end{lemma}
\begin{proof}
By the triangle inequality,
\begin{align*}
|\mathbb{E}[\overline{R}_n] - \mu^{\star}| &\leq |\mu^{\star} - \mu^{\star}_n| + |\mu^{\star}_n - \mathbb{E}[\overline{R}_n]| \\
                                           &= |\delta_n^{\star}| + |\mu^{\star} - \mathbb{E}[\overline{X}_n]|.
\end{align*}
Hence it only remains to bound $|\mu^{\star} - \mathbb{E}[\overline{X}_n]|$. By definition:
\begin{align}
&|\mu^{\star} - \mathbb{E}[\overline{X}_n]| = \left\lvert \mu^{\star} - \mathbb{E}\left[\sum_{i=1}^K\frac{T_i(n)\overline{R}_{i,T_i(n)}}{n}\right]\right\rvert \label{eq:lemma_4_12}\\
&\implies  n|\mu^{\star} - \mathbb{E}[\overline{X}_n]| = \left\lvert \sum_{t-1}^n\mathbb{E}[R_t^{\star}] - \mathbb{E}\left[\sum_{i=1}^KT_i(n)\overline{R}_{i,T_i(n)}\right]\right\rvert \nonumber\\
&\numeq{q}\left\lvert \sum_{t-1}^n\mathbb{E}[R_t^{\star}] - \mathbb{E}\left[T^{\star}(n)\overline{R}^{\star}_{T^{\star}(n)}\right]\right\rvert - \mathbb{E}\left[\sum_{i\neq i^{\star}}T_i(n)\overline{R}_{i,T_i(n)}\right],\nonumber
\end{align}
where (a) follows because $\overline{R}_{i,t} \in [0,1]$ for all $i$ and $t$ (Assumption \ref{assumption:eventual_station}). We start by bounding the second term in (\ref{eq:lemma_4_12}):
\begin{align*}
&\mathbb{E}\left[\sum_{i\neq i^{\star}}T_i(n)\overline{R}_{i,T_i(n)}\right] \leq \mathbb{E}\left[\sum_{i\neq i^{\star}}T_i(n)\right]\\
  &\numleq{\text{Lemma \ref{lemma:t_1_233}}} K\left(\frac{16C_p^2 \ln t}{(1-\epsilon)^2\Delta_i^2} + N_0(\epsilon) + N_p + 1 + \frac{\pi^2}{3}\right)\\
  &=\mathcal{O}\left(K\left(C_p^2 \ln n + N_0(\epsilon)\right)\right).
\end{align*}
Now we consider the first term in (\ref{eq:lemma_4_12}). Note that $T^{\star}(n)\overline{R}^{\star}_{T^{(\star)}(n)} = \frac{T^{\star}(n)}{T^{\star}(n)}\sum_{t=1}^{T^{\star}(n)}\overline{R}^{\star}_t=\sum_{t=1}^{T^{\star}(n)}\overline{R}^{\star}_t$. Using this expression we obtain:
\begin{align*}
  &\left\lvert \sum_{t-1}^n\mathbb{E}[R_t^{\star}] - \mathbb{E}\left[T^{\star}(n)\overline{R}^{\star}_{T^{\star}(n)}\right]\right\rvert = \left\lvert\mathbb{E}\left[\sum_{t-1}^nR_t^{\star}- \sum_{t-1}^{T^{\star}(n)}R_t^{\star}\right]\right\rvert\\
  &\numeq{a}\sum_{t=T^{\star}(n)+1}^n\mathbb{E}\left[R_t^{\star}\right] \leq \mathbb{E}[n - T^{\star}(n)] = \sum_{i\neq i^{\star}}\mathbb{E}[T_i(n)]\\
  &\numleq{\text{Lemma \ref{lemma:t_1_233}}} K\left(\frac{16C_p^2 \ln t}{(1-\epsilon)^2\Delta_i^2} + N_0(\epsilon) + N_p + 1 + \frac{\pi^2}{3}\right)\\
  &=\mathcal{O}\left(K\left(C_p^2 \ln n + N_0(\epsilon)\right)\right),
\end{align*}
where (a) follows from the fact that $\overline{R}_{i,t} \in [0,1]$ for all $i$ and $t$ (Assumption \ref{assumption:eventual_station}).
\end{proof}
\begin{lemma}\label{lemma:kocsis}
Let $n_0$ be such that $\sqrt[]{n_0} \geq \mathcal{O}(K(C_p^2\ln n_0 + N_0(\frac{1}{2})))$. Given Assumption \ref{assumption:eventual_station}, the following holds for any $n \geq n_0$ and $\delta>0$:
\begin{align*}
P\left(n \overline{X}_n \geq n\mathbb{E}[\overline{X}_n] + 9\sqrt[]{2\ln(2/\delta)}\right) \leq \delta\\
P\left(n \overline{X}_n \geq n\mathbb{E}[\overline{X}_n] - 9\sqrt[]{2\ln(2/\delta)}\right) \leq \delta.
\end{align*}
\end{lemma}
\begin{proof}
This lemma was originally proved by Kocsis \& Szepesvári \cite[Thm. 5]{uct}. A more accessible version of the proof can be found in \cite[Thm. 5]{power_uct}. We omit the full proof for brevity. For the sake of completeness, we briefly outline the main ideas behind the proof here. The proof involves defining a counting process that represents the number of times a sub-optimal arm is pulled and then bounding the deviation of this process. Key to this argument is the Hoeffding-Azuma inequality \cite[Lem. 8--10]{uct}. By leveraging this inequality and martingale theory, it is possible to conclude that the desired inequalities must hold.
\end{proof}
\begin{lemma}\label{lemma:convergence}
Given Assumption \ref{assumption:eventual_station}, $\lim_{t \rightarrow \infty}P(I_t \neq i^{\star})=0$.
\end{lemma}
\begin{proof}
Let $p_{i,t} = P(\overline{R}_{i, T_i(t)} \geq R^{\star}_{T^{\star}(t)})$. Clearly, $P(I_t \neq i^{\star}) \leq \sum_{i \neq i^{\star}}p_{it}$. Hence it suffices to show that $p_{i,t} \leq \frac{\epsilon}{K}$ for all $i$ and any $\epsilon > 0$. Towards this proof, note that if $\overline{R}_{i,T_i(t)} < \mu_i + \frac{\Delta_i}{2}$ and $\overline{R}^{\star}_{T^{\star}(t)} > \mu^{\star}-\frac{\Delta}{2}$, then
\begin{align*}
\overline{R}_{i,T_i(t)} < \mu_i + \frac{\Delta_i}{2} = \mu^{\star} - \frac{\Delta_i}{2} < \overline{R}^{\star}_{T^{\star}(t)}.
\end{align*}
As a consequence,
\begin{align*}
p_{i,t} \leq P\left(\overline{R}_{i,T_i(t)} < \mu_i + \frac{\Delta_i}{2}\right)+ P\left(\overline{R}^{\star}_{T^{\star}(t)} > \mu^{\star}-\frac{\Delta}{2}\right).
\end{align*}
Since $T_i(t)$ grows slower than $T^{\star}(t)$, it suffices to bound the first of the two terms above. By definition:
\begin{align*}
&P\left(\overline{R}_{i,T_i(t)} < \mu_i + \frac{\Delta_i}{2}\right) \\
&= P\left(\overline{R}_{i,T_i(t)} < \mu_{i,T_i(t)}-|\delta_{i,T_i(t)}| + \frac{\Delta_i}{2}\right).
\end{align*}
Next note that $|\delta_{i,T_i(t)}|$ converges to $0$ by Assumption \ref{assumption:eventual_station}. Hence, we can assume that $|\delta_{i,T_i(t)}|$  is decreasing in $t$ without loss of generality. It then follows from Lemma \ref{lemma:lower_bound} that $|\delta_{i, T_i(t)}| \leq |\delta_{i, \lceil \rho \log t \rceil}|$. Now consider $t \geq \lceil \rho \log t\rceil \geq N_0(\frac{\Delta}{4})$, then $|\delta_{i,T_i(t)}| \leq |\delta_{i,\lceil \rho \log t \rceil}| \leq \frac{\Delta_i}{4}$ by Lemma \ref{lemma:delta_convergence}. As a consequence,
\begin{align}
  &P\left(\overline{R}_{i,T_i(t)} < \mu_{i,T_i(t)}-|\delta_{i,T_i(t)}| + \frac{\Delta_i}{2}\right) \label{eq:leamm6_23}\\
  &\leq P\left(\overline{R}_{i,T_i(t)} < \mu_{i,T_i(t)}+\frac{\Delta_i}{4}\right)\nonumber \\
  &\leq P\left(\overline{R}_{i,T_i(t)} < \mu_{i,T_i(t)}+\frac{\Delta_i}{4},T_i(t) \geq a\right) + P\left(T_i(t) \leq a\right).\nonumber
\end{align}
Since $\lim_{t \rightarrow \infty}T_i(t)=\infty$ (Lemma \ref{lemma:lower_bound}), we have that $\lim_{t \rightarrow \infty}P(T_i(t) < a) = 0$ for any $a$. Hence it suffices to bound the first term in (\ref{eq:leamm6_23}), which can be done as follows.
\begin{align*}
&P\left(\overline{R}_{i,T_i(t)} < \mu_{i,T_i(t)}+\frac{\Delta_i}{4},T_i(t) \geq a\right) \leq \\
  &P\left(\overline{R}_{i,T_i(t)} < \mu_{i,T_i(t)}+\frac{\Delta_i}{4}\mid T_i(t) \in [a,b]\right) P\left(T_i(t) \in [a,b]\right)\\
&  + P\left(T_i(t) \not\in [a,b]\right).
\end{align*}
Next note that
\begin{align*}
  &P\left(\overline{R}_{i,T_i(t)} < \mu_{i,T_i(t)}+\frac{\Delta_i}{4} \mid T_i(t) \in [a,b]\right) \\
  &\leq \sum_{k=a}^bP\left(\overline{R}_{i,k} < \mu_{i,k}+\frac{\Delta_i}{4}\right)\\
  &\leq (b-a+1)\max_{a \leq k \leq b}P\left(\overline{R}_{i,k} < \mu_{i,k}+\frac{\Delta_i}{4}\right).
\end{align*}
Now we use Assumption \ref{assumption:eventual_station}, which implies that the value of $a$ in the expression above can be chosen such that for all $t\geq a$, $(t+1)P(\overline{X}_{i,t} \geq \mu_{i,t} + \frac{\Delta_i}{4}) < \frac{\epsilon}{2K}$, which means that
\begin{align*}
  &(b-a+1)\max_{a \leq k \leq b}P\left(\overline{R}_{i,k} < \mu_{i,k}+\frac{\Delta_i}{4}\right) \\
  &=^{b = 2a} (a+1)\max_{a \leq k \leq b}P\left(\overline{R}_{i,k} < \mu_{i,k}+\frac{\Delta_i}{4}\right)\\
  &\leq (a+1)P\left(\overline{R}_{i,a} < \mu_{i,a}+\frac{\Delta_i}{4}\right) \leq \frac{\epsilon}{2K}.
\end{align*}
Putting the bounds above together, we get:
\begin{align*}
P(I_t \neq i^{\star}) \leq \sum_{i \neq i^{\star}}p_{it}\leq K \frac{\epsilon}{K} = \epsilon.
\end{align*}
\end{proof}
Now we use Lemmas \ref{lemma:exp_concentration}--\ref{lemma:convergence} to prove Thm. 3.
\subsection{Proof of Theorem \ref{thm:pomcp_convergence}}
The proof uses mathematical induction on the time horizon $\mathcal{T}$, corresponding to the search tree's depth. Without loss of generality, we assume that the target variables are normalized to the interval $[0,1]$.

For the inductive base case where $\mathcal{T}=2$, \textsc{c-pomcp} corresponds to a stationary \mab problem, which means that Assumption \ref{assumption:eventual_station}.1--2 hold. Further, Assumption \ref{assumption:eventual_station}.3 follows from Hoeffding's inequality:
\begin{align*}
P\left(\overline{R}_{i,n} \leq \mu_{i,n} \pm \frac{1}{n}\sqrt[]{\frac{\ln t}{n}}\right) &\leq e^{-2\frac{\sqrt[]{\frac{2\ln t}{n}}}{n}} = e^{-4 \ln t} = t^{-4}.
\end{align*}
As a consequence, we can invoke Lemma \ref{lemma:thm_3} and Lemma \ref{lemma:convergence}, which asserts that the theorem statement holds when $\mathcal{T}=2$.

Assume by induction that the theorem statement holds for time horizons $3,4,\hdots,\mathcal{T}-1$ and consider time horizon $\mathcal{T}$. By Thm. \ref{thm:particles}, the problem of finding an optimal intervention from the root node when the horizon is $\mathcal{T}$ corresponds to a non-stationary \textsc{mab} with correlated rewards. This \textsc{mab} satisfies Assumption \ref{assumption:eventual_station}.1 if the rewards are divided by $\mathcal{T}$. Further, it follows from the induction hypothesis that the reward distributions in each subtree converge, which means that Assumption \ref{assumption:eventual_station}.2 holds. Moreover, the induction hypothesis together with Lemma \ref{lemma:kocsis} implies that Assumption \ref{assumption:eventual_station}.3 is satisfied. As a consequence, we can apply Lemma \ref{lemma:convergence}, which ensures that the intervention prescribed at the root converges to an optimal intervention. Similarly, we can apply Lemma \ref{lemma:thm_3}, which states that
\begin{align*}
\left\lvert \mathbb{E}\left[\overline{R}_n - \mu^{\star}_n\right]\right\rvert \leq |\delta^{\star}| + \mathcal{O}\left(\frac{KC_p^2\ln n + N_0}{n}\right),
\end{align*}
where $\overline{R}_n$ is the average reward at the root.
By the induction hypothesis,
\begin{align*}
|\delta^{\star}| &= \mathcal{O}\left(\frac{K(\mathcal{T}-1)\log n + K^{\mathcal{T}-1}}{n}\right).
\end{align*}
Hence,
\begin{align*}
& \left\lvert \mathbb{E}[\overline{R}_n - \mu^{\star}_n]\right\rvert \leq \mathcal{O}\left(\frac{K(\mathcal{T}-1)\log n + K^{\mathcal{T}-1}}{n}\right) \\
& \quad\quad\quad\quad\quad\quad\quad\quad + \mathcal{O}\left(\frac{KC_p^2\ln n + N_0}{n}\right).
\end{align*}
Thus, it only remains to bound $N_0$. It follows from Lemma \ref{lemma:delta_convergence} that $N_0$ is upper bounded by the smallest value of $n$ for which the following inequality holds
\begin{align*}
\frac{K(\mathcal{T}-1) \log n + K^{\mathcal{T}-1}}{n} &\geq \frac{\Delta_i}{2}\\
  \implies \frac{2\left(K^{\mathcal{T}-1} + K(\mathcal{T}-1)\log n\right)}{\Delta_i} &\geq n \\
  \implies N_0 &= \mathcal{O}(K^{\mathcal{T}-1}).
\end{align*}
As a consequence,
\begin{align*}
&\left\lvert \mathbb{E}[\overline{R}_n - \mu^{\star}_n]\right\rvert \leq \mathcal{O}\left(\frac{K(\mathcal{T}-1)\log n + K^{\mathcal{T}-1}}{n}\right)\\
&+ \mathcal{O}\left(\frac{KC_p^2\ln n + K^{\mathcal{T}}}{n}\right) =\mathcal{O}\left(\frac{(K\mathcal{T} \log n + K^{\mathcal{T}})}{n}\right).\nonumber
\end{align*}
Hence, by induction, the theorem holds for all $\mathcal{T}$. \qed

\section{Hyperparameters}\label{appendix:hyperparameters}
The hyperparameters used for the evaluation are listed in Table \ref{tab:hyperparams} and were obtained through random search.

\begin{table}
\centering
\resizebox{1\columnwidth}{!}{%
  \begin{tabular}{ll} \toprule
    \textbf{Parameters} & {\textbf{Values}} \\
    \midrule
    Model parameters &   \\
    \hline
    $\beta_{\mathsf{R},1}$, $\beta_{\mathsf{R},2}$, $\beta_{\mathsf{R},3}$  & $0.1$, $1$, $1$\\
    $\beta_{I_{i,t}, z}$ & $0$ if $I_{i,t} \neq \mathsf{R}$\\
    $q_t$ & $1$ for restore interventions, $0$ otherwise\\
    $\psi_{z_i}$ & $10$ if $z_i=3$, $0$ otherwise\\
    Cyborg version \cite{cyborg} & commit 9421c8e\\
    \midrule
    \textsc{c-pomcp} and \textsc{pomcp} \cite[Alg. 1]{pomcp} &   \\
    \midrule
    search time, default node value & $0.05$s--$30$s, $0$\\
    $M$, $\gamma$, $c$ & $1000$, $0.99$, $0.5$\\
    rollout depth, maximum search depth & $4$, $50$\\
    intervention space & intervention space described in \cite{vyas2023automated}\\
    base strategy, base value & $\widehat{\pi}_{\mathrm{D}}(\mathrm{do}(\emptyset))=1$, $J_{\widehat{\pi}}(\cdot)=0$\\
    \midrule
    \textsc{cardiff-ppo} \cite[Alg. 1]{ppo} \cite{vyas2023automated} &   \\
    \midrule
    learning rate, \# hidden layers,  & $5148 \cdot 10^{-5}$, $1$, \\
    \# neurons per layer, \# steps between updates, & $64$, $2048$,\\
    batch size, discount factor $\gamma$ & $16$, $0.99$\\
    \textsc{gae} $\lambda$, clip range, entropy coefficient & $0.95$, $0.2$, $2\cdot 10^{-4}$\\
    value coefficient, max gradient norm & $0.102$, $0.5$\\
    feature representation & the original cyborg features \cite{cyborg} \& \\
                      & one-hot encoded scan-state \& \\
                      & decoy-state for each node\\
  \bottomrule\\
\end{tabular}
}
\caption{Hyperparameters.}\label{tab:hyperparams}
\end{table}

\section{Configuration of the Target System (Fig. \ref{fig:use_case})}\label{appendix:infrastructure_configuration}
The configuration of the target system in \textsc{cage-2} (Fig. \ref{fig:use_case}) is available in Table \ref{tab:processes}. The attacker actions are listed in Table \ref{tab:exploits} and the defender interventions are listed in Table \ref{tab:defender_interventions}. The decoy services are listed in Table \ref{tab:decoys}, and the workflow graph is shown in Fig. \ref{fig:workflow_graph}.

\begin{table}
  \centering
  \scalebox{0.76}{
  \begin{tabular}{lllll} \toprule
    {\textit{Node \textsc{id}, hostname}} & {\textit{Processes}} & {\textit{Ports}} & {\textit{Users}} & {\textit{Vulnerabilities}}\\ \midrule
    1, \textsc{client}-1 & \textsc{sshd.exe} & 22 & \textsc{sshd\_server} & \cwe-251\\
     & \textsc{femitter.exe} & 21 & \textsc{system} & \cve-2020-26299\\
    \midrule
    2, \textsc{client}-2 & \textsc{smss.exe} & 445,139 & \textsc{system} & -\\
    & \textsc{svchost.exe} & 135 & \textsc{system} & -\\
    & \textsc{svchost.exe} & 3389 & \textsc{network} & \cve-2019-0708\\
    \midrule
    3, \textsc{client}-3 & \textsc{mysql} & 3389 & \textsc{root} & \cve-2019-0708\\
     & \textsc{apache2} & 80,443 & \textsc{www-data} & \cwe-89, \textsc{http}-\textsc{(s)rfi}\\
     & \textsc{smtp} & 25 & \textsc{root} & \cve-2016-1000282\\
    \midrule
    4, \textsc{client}-4 & \textsc{sshd} & 22 & \textsc{root} & \cwe-251\\
     & \textsc{mysql} & 3390 & \textsc{root} & \cwe-89\\
     & \textsc{apache2} & 80, 443 & \textsc{www-data} & \cwe-89, \textsc{http}-\textsc{(s)rfi}\\
     & \textsc{smtp} & 25 & \textsc{root} & \cve-2016-1000282\\
    \midrule
    5, \textsc{enterprise}-1 & \textsc{sshd.exe} & 22 & \textsc{root} & \cwe-251\\
    \midrule
    6, \textsc{enterprise}-2 & \textsc{sshd.exe} & 22 & \textsc{sshd\_server} & \cwe-251\\
     & \textsc{svchost.exe} & 135 & \textsc{system} & -\\
     & \textsc{svchost.exe} & 3389 & \textsc{system} & \cve-2019-0708\\
     & \textsc{smss.exe} & 445,139 & \textsc{system} & \cve-2017-0144\\
     & \textsc{tomcat8.exe} & 80,443 & \textsc{network} & \cwe-89,\textsc{http}-\textsc{(s)rfi}\\
    \midrule
    7, \textsc{enterprise}-3 & \textsc{sshd.exe} & 22 & \textsc{sshd\_server} & \cwe-251\\
     & \textsc{svchost.exe} & 135 & \textsc{system} & -\\
     & \textsc{svchost.exe} & 3389 & \textsc{system} & \cve-2019-0708\\
     & \textsc{smss.exe} & 445,139 & \textsc{system} & \cve-2017-0144\\
     & \textsc{tomcat8.exe} & 80,443 & \textsc{network} & \cwe-89,\textsc{http}-\textsc{(s)rfi}\\
    \midrule
    8, \textsc{operational}-1 & \textsc{sshd} & 22 & \textsc{root} & \cwe-251\\
    \midrule
    9, \textsc{operational}-2 & \textsc{sshd} & 22 & \textsc{root} & \cwe-251\\
    \midrule
    10, \textsc{operational}-3 & \textsc{sshd} & 22 & \textsc{root} & \cwe-251\\
    \midrule
    11, \textsc{operational}-4 & \textsc{sshd} & 22 & \textsc{root} & \cwe-251\\
    \midrule
    12, \textsc{defender} & \textsc{sshd} & 22 & \textsc{root} & \cwe-251\\
                    & \textsc{systemd} & 53,78 & \textsc{systemd+} & \\
    \midrule
    13, \textsc{attacker} & \textsc{sshd.exe} & 22 & \textsc{sshd\_server} & \cwe-251\\
     & \textsc{femitter.exe} & 21 & \textsc{system} & \cve-2020-26299\\
    \bottomrule\\
  \end{tabular}}
  \caption{Configuration of the target system in \textsc{cage-2} (Fig. \ref{fig:use_case}); vulnerabilities are identified by their identifiers in the Common Vulnerabilities and Exposures (\cve) database \cite{cve} and the Common Weakness Enumeration (\cwe) list \cite{cwe}.}\label{tab:processes}
\end{table}

\begin{table}
  \centering
  \scalebox{0.85}{
  \begin{tabular}{lll} \toprule
    {\textit{Type}} & {\textit{Actions}} & {\textsc{mitre att\&ck} technique} \\ \midrule
    Reconnaissance & Subnet scan for nodes & \textsc{t1018} system discovery \\
                   & Port scan on a specific node & \textsc{t1046} service scanning\\\\
    Exploits & \cve-2017-0144, \http-\textsc{srfi} & \textsc{t1210} service exploitation \\
             & \textsc{sql} \textsc{injection} (\cwe-89) & \textsc{t1210} service exploitation \\
             & \cve-2016-1000282 & \textsc{t1210} service exploitation \\
             &  \cve-2020-26299 \http-\textsc{rfi} & \textsc{t1210} service exploitation \\\\
    Brute-force & \ssh & \textsc{t1110} brute force\\\\
    Escalate & Escalate privileges of user to root & \textsc{t1068} privilege escalation\\\\
    Impact & Stop services running on node & \textsc{t1489} service stop\\
    \bottomrule\\
  \end{tabular}}
  \caption{Attacker actions in \textsc{cage}-2 \cite{cage_challenge_2_announcement}.}\label{tab:exploits}
\end{table}

\begin{table}
  \centering
  \scalebox{0.85}{
  \begin{tabular}{lll} \toprule
    {\textit{Type}} & {\textit{Interventions}} & {\textsc{mitre d3fend technique}} \\ \midrule
    Monitor & Network monitoring & \textsc{d3-nta} network analysis\\
            & Forensic analysis & \textsc{d3-fa} file analysis \\\\
    Start decoys & \textsc{apache}, \textsc{femitter} & \textsc{d3-de} decoy environment\\
                 & \textsc{haraka}, \textsc{smss} & \textsc{d3-de} decoy environment \\
                 & \textsc{sshd}, \textsc{svchost} \textsc{tomcat} & \textsc{d3-de} decoy environment \\\\
    Restore & Restore node to a checkpoint & \textsc{d3-ra} restore access \\
           & Attempt to remove attacker & \textsc{d3-fev} file eviction\\
    \bottomrule\\
  \end{tabular}}
  \caption{Defender interventions in \textsc{cage}-2 \cite{cage_challenge_2_announcement}.}\label{tab:defender_interventions}
\end{table}

\begin{table}
  \centering
  \scalebox{0.9}{
  \begin{tabular}{lll} \toprule
    {\textit{\textsc{id}}} & {\textit{Name}} & {\textit{Description}} \\ \midrule
    $1$ & \textsc{decoy-apache} & Starts a vulnerable \textsc{apache} \textsc{http} server decoy \\
    $2$ & \textsc{decoy-femitter} & Starts a vulnerable \textsc{femitter} \textsc{ftp} server decoy \\
    $3$ & \textsc{decoy-smtp} & Starts a vulnerable \textsc{haraka} \textsc{smtp} server decoy \\
    $4$ & \textsc{decoy-smss} & Starts a vulnerable \textsc{smss} server decoy \\
    $5$ & \textsc{decoy-sshd} & Starts an \textsc{ssh} server decoy with a weak password \\
    $6$ & \textsc{decoy-svchost} & Starts a vulnerable \textsc{svchost.exe} process decoy \\
    $7$ & \textsc{decoy-tomcat} & Starts a vulnerable \textsc{tomcat} \textsc{http} server decoy \\
    $8$ & \textsc{decoy-vsftpd} & Starts a vulnerable \textsc{vsftpd} \textsc{ftp} server decoy \\
    \bottomrule\\
  \end{tabular}}
  \caption{Decoy services in \textsc{cage}-2 \cite{cage_challenge_2_announcement}.}\label{tab:decoys}
\end{table}

\begin{figure}
  \centering
  \scalebox{0.95}{
    \begin{tikzpicture}[fill=white, >=stealth,
    node distance=3cm,
    database/.style={
      cylinder,
      cylinder uses custom fill,
      shape border rotate=90,
      aspect=0.25,
      draw}]

    \tikzset{
node distance = 9em and 4em,
sloped,
   box/.style = {%
    shape=rectangle,
    rounded corners,
    draw=blue!40,
    fill=blue!15,
    align=center,
    font=\fontsize{12}{12}\selectfont},
 arrow/.style = {%
    line width=0.1mm,
    -{Triangle[length=5mm,width=2mm]},
    shorten >=1mm, shorten <=1mm,
    font=\fontsize{8}{8}\selectfont},
}

\node[scale=0.7] (gi) at (8.8,0)
{
\begin{tikzpicture}
\node[draw,circle, minimum width=0.8cm, scale=1](n1) at (0,0) {};
\node[draw,circle, minimum width=0.8cm, scale=1](n2) at (1.5,0) {};
\node[draw,circle, minimum width=0.8cm, scale=1](n3) at (3,0) {};
\node[draw,circle, minimum width=0.8cm, scale=1](n4) at (4.5,0) {};

\node[draw,circle, minimum width=0.8cm, scale=1](n5) at (1.5,-1.25) {};
\node[draw,circle, minimum width=0.8cm, scale=1](n6) at (3,-1.25) {};
\node[draw,circle, minimum width=0.8cm, scale=1](n7) at (2.25,-2.5) {};
\node[draw,circle, minimum width=0.8cm, scale=1](n8) at (2.25,-3.75) {};

\node[inner sep=0pt,align=center, scale=1] (dots4) at (0,0)
{$1$};
\node[inner sep=0pt,align=center, scale=1] (dots4) at (1.5,0)
{$2$};
\node[inner sep=0pt,align=center, scale=1] (dots4) at (3,0)
{$3$};
\node[inner sep=0pt,align=center, scale=1] (dots4) at (4.5,0)
{$4$};
\node[inner sep=0pt,align=center, scale=1] (dots4) at (1.5,-1.25)
{$5$};
\node[inner sep=0pt,align=center, scale=1] (dots4) at (3,-1.25)
{$6$};
\node[inner sep=0pt,align=center, scale=1] (dots4) at (2.25,-2.5)
{$7$};
\node[inner sep=0pt,align=center, scale=1] (dots4) at (2.25,-3.75)
{$8$};

\draw[-{Latex[length=2mm]}, line width=0.22mm, color=black] (n1) to (n5);
\draw[-{Latex[length=2mm]}, line width=0.22mm, color=black] (n2) to (n5);
\draw[-{Latex[length=2mm]}, line width=0.22mm, color=black] (n3) to (n6);
\draw[-{Latex[length=2mm]}, line width=0.22mm, color=black] (n4) to (n6);
\draw[-{Latex[length=2mm]}, line width=0.22mm, color=black] (n5) to (n7);
\draw[-{Latex[length=2mm]}, line width=0.22mm, color=black] (n6) to (n7);
\draw[-{Latex[length=2mm]}, line width=0.22mm, color=black] (n7) to (n8);
\end{tikzpicture}
};

\end{tikzpicture}
  }
  \caption{Workflow graph $\mathcal{G}_{\mathrm{W}}$ in \textsc{cage}-2 \cite{cage_challenge_2_announcement}; circles represent nodes of the target system (Fig. \ref{fig:use_case}) and edges represent service dependencies.}
  \label{fig:workflow_graph}
\end{figure}
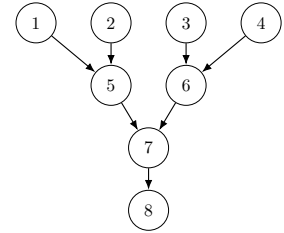

\hypersetup{
  colorlinks,
  linkcolor={black},
  citecolor={black},
  urlcolor={black}
}
\bibliographystyle{IEEEtran}
\bibliography{references}

\end{document}